\documentclass[twoside]{article}

\usepackage{aistats2022}

\usepackage[round]{natbib}

\newcommand{\xv}{{\boldsymbol x}}

\newcommand{\zv}{{\boldsymbol z}}

\usepackage{soul}

\newcommand{\cov}{\text{cov}}
\newcommand{\EE}{\mathbb{E}}

\newcommand{\BR}{\mathbb{R}}

\newcommand{\ud}{\,\text{d}}

\newcommand{\CE}{\mathcal{E}}

\newcommand{\KL}{\text{KL}}

\newcommand{\CM}{\mathcal{M}}

\DeclareMathOperator{\argmin}{\arg\min}

\newcommand{\CU}{\mathcal{U}}

\newcommand{\CN}{\mathcal{N}}

\usepackage{tikz}
\newcommand{\beq}{\begin{equation}}
\newcommand{\eeq}{\end{equation}}
\newcommand{\beqs}{\begin{eqnarray}}
\newcommand{\eeqs}{\end{eqnarray}}
\newcommand{\barr}{\begin{array}}
\newcommand{\earr}{\end{array}}
\usepackage{multirow}
\usepackage{hyperref}
\usepackage{url}
\usepackage{booktabs}
\usepackage{xcolor}
\usepackage{algorithm,algorithmic}

\usepackage{amsmath,amssymb,amsthm,amscd,amstext}
\usepackage{wrapfig}
\usepackage{capt-of}

\newcommand{\ELBO}{\text{ELBO}}
\newcommand{\EUBO}{\text{EUBO}}

\newcommand{\TVO}{\text{TVO}}

\newcommand{\IWELBO}{\text{IW-ELBO}}
\newcommand{\RVI}{\text{RVI}}
\newcommand{\AVI}{\text{VI}_{\text{Aux}}}
\newcommand{\MCMCVI}{\text{VI}_{\text{MCMC}}}

\newcommand{\WUBO}{\text{WUBO}}
\newcommand{\WLBO}{\text{WLBO}}

\newcommand{\HBO}{\text{HBO}}
\newcommand{\ESS}{\text{ESS}}

\newcommand{\IW}{\text{IW}}

\newcommand{\BD}{\mathbb{D}}
\newcommand{\Zv}{\mathbb{Z}}

\theoremstyle{plain}
\newtheorem{thm}{Theorem}[section] % reset theorem numbering for each chapter
\newtheorem{coro}{Corollary}[section] % reset theorem numbering for each chapter
\theoremstyle{definition}
\newtheorem{defn}[thm]{Definition} % definition numbers are dependent on theorem numbers
\newtheorem{prop}[thm]{Proposition} % same for example numbers

\newtheorem{lem}[thm]{Lemma}

\usepackage[toc,page,header]{appendix}
\usepackage{minitoc}

\doparttoc % Tell to minitoc to generate a toc for the parts
\faketableofcontents % Run a fake tableofcontents command for the partocs

\begin{document}

\twocolumn[

\aistatstitle{Variational Inference with H\"older Bounds}

\aistatsauthor{ Junya Chen${}^{*}$ \And Danni Lu \And  Zidi Xiu \And Ke Bai \And Lawrence Carin \And Chenyang Tao${}^{*}$}

\aistatsaddress{ Duke University \And  Virginia Tech \And Duke University \And Duke University \And Duke University \And Duke University} ]

\begin{abstract}
The recent introduction of thermodynamic integration techniques has provided a new framework for understanding and improving variational inference (VI). In this work, we present a careful analysis of the thermodynamic variational objective (TVO), bridging the gap between existing variational objectives and shedding new insights to advance the field. In particular, we elucidate how the TVO naturally connects the three key variational schemes, namely the importance-weighted VI, R\'enyi-VI and MCMC-VI, which subsumes most VI objectives employed in practice. To explain the performance gap between theory and practice, we reveal how the pathological geometry of thermodynamic curves negatively affect TVO. By generalizing the integration path from the geometric mean to the weighted H\"older mean, we extend the theory of TVO and identify new opportunities for improving VI. This motivates our new VI objectives, named the {\it H\"older bounds}, which flatten the thermodynamic curves and promise to achieve one-step approximation of the exact marginal $\log$-likelihood. A comprehensive discussion on the choices of numerical estimators is provided. We present strong empirical evidence on both synthetic and real-world datasets to support our claims. 
\end{abstract}

\vspace{-1em}
\section{Introduction}
\vspace{-5pt}
One of the key challenges in modern machine learning is to approximate complex distributions.
Due to recent advances on learning scalability \citep{hoffman2013stochastic} and flexibility \citep{kingma2016improving}, and the development of automated inference procedures \citep{ranganath2014black}, {\it variational inference} (VI) has become a popular approach for general latent variable models \citep{blei2017var}.
Variational inference leverages a posterior approximation to derive a lower bound on the log-evidence of the observed data, and it can be efficiently optimized.
This variational bound, more commonly known as the {\it evidence lower bound} (ELBO), serves as a surrogate objective for {\it maximum likelihood estimation} (MLE).
Successful applications of VI have been reported in document analysis \citep{blei2003latent}, neuroscience \citep{friston2007variational}, generative modeling \citep{kingma2013auto}, among many others.

A widely recognized heuristic is that, tightening the variational bound, in general, improves model performance \citep{burda2015importance}.
Consequently, considerable research has been directed toward this goal.
The most direct approach seeks to boost the expressive power of the approximate posterior, such that it can match the true posterior better.
For instance, normalizing flows \citep{rezende2015variational, kingma2016improving} exploited invertible transformations with tractable Jacobian on the latent codes, \citet{salimans2015markov, gregor2015draw, ranganath2016hierarchical} explored the hierarchical structure of the latent code generation, \citet{miller2016variational} modeled the posterior as a mixture of Gaussians, and \citet{mescheder2017adversarial} adversarially trained a neural sampler to enable flexibility. Alternatively, employing data-adaptive priors similarly closes the variational gap \citep{tomczak2018vae}. While showing varying degree of successes, these approaches often involve specific design choices ({\it e.g.}, network architectures), which complicates implementations. 

\begin{figure}
\begin{center}
\includegraphics[width=1.\columnwidth]{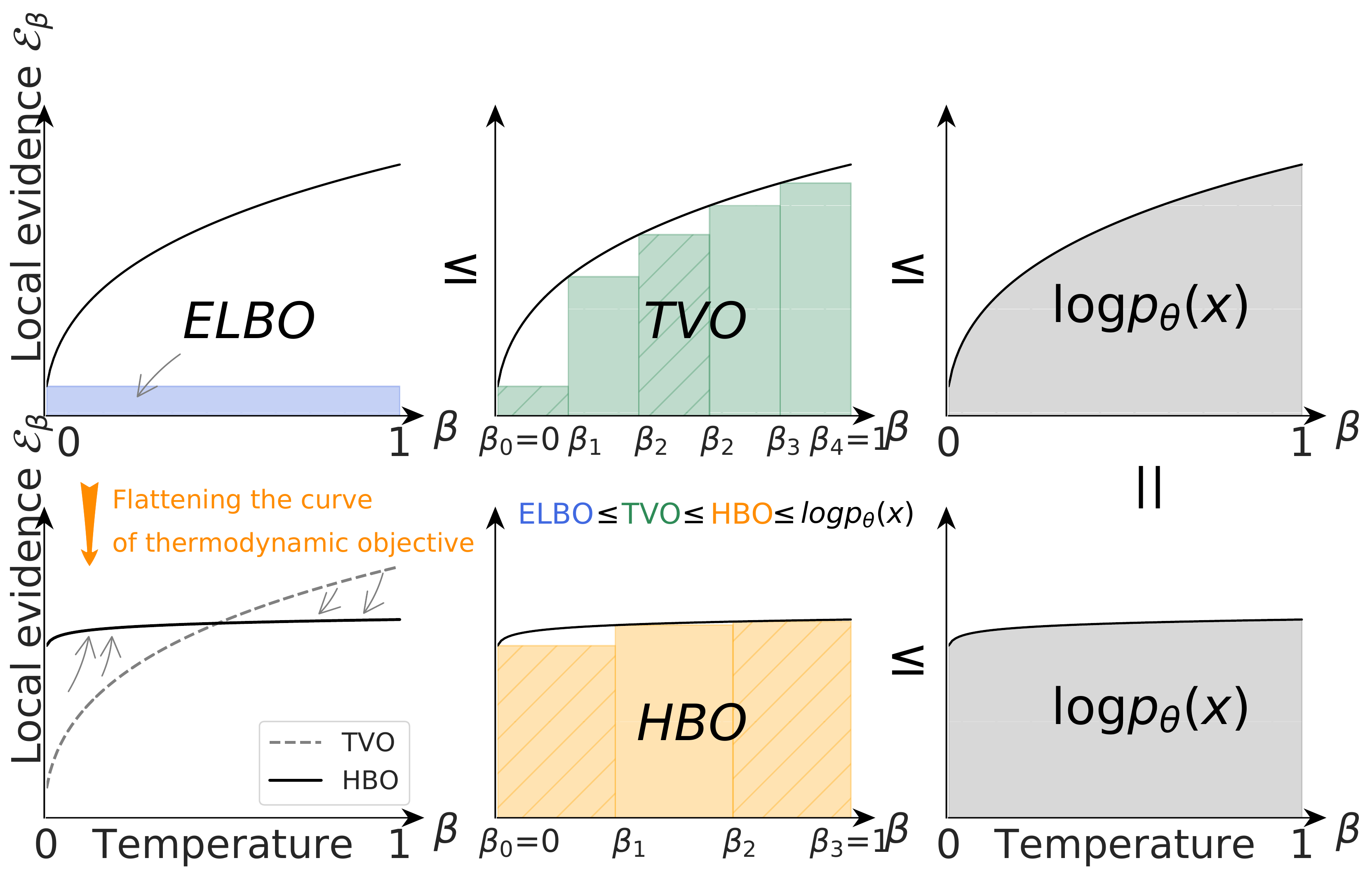}
\end{center}
\vspace{-1em}
\caption{Thermodynamic integration and variational bounds. Marginal likelihood $\log p_{\theta}(x)$ equals the area under the thermodynamic curve. We show by carefully choosing an integral path to flatten the thermodynamic curve, the $\log p_{\theta}(x)$ can be approximated more accurately and with less computations.  \label{fig:tv}}
\label{fig:tvo_grad}
\vspace{-2em}
\end{figure}
%\end{minipage}
%\end{wrapfigure}

Independent of design choices of posteriors and priors, an orthogonal direction instead advocates direct modifications to the variational objectives, making them provably tighter or otherwise more favorable ({\it e.g.}, efficiency, stability, {\it etc.}).
As a prominent example, the $\ELBO$ can be sharpened by leveraging multiple posterior samples weighted in accordance with their importance \citep{burda2015importance}.  
Further, R\'enyi-VI \citep{li2016renyi} and $\chi$-VI \citep{dieng2017chi} derived a family of variational targets that interpolates between lower and upper bounds to the $\log$-likelihood \citep{tao2018variational}, and they belong to the more general family of $f$-divergence VI \citep{wan2020f}.
\citet{bamler2017perturbative} developed a more general view on and presented a low-variance estimator based on a perturbative argument.
However, an important note made by \citet{rainforth2017tighter} is that sharpening the variational bound may inadvertently hurt model inference ({\it i.e.}, the approximate posterior).

More recently, \citet{masrani2019the} derived a novel variational objective using {\it Thermodynamic Integration} techniques \citep{lartillot2006computing}. This new scheme, known as the {\it thermodynamic variational objective} (TVO), is defined by a path integral of local evidence curve, which in theory recovers the exact model marginal $\log$-likelihood (see Figure \ref{fig:tv} and Sec \ref{sec:tvo}). The discrete Riemann left sum of the curve yields lower bounds the $\log$ marginal likelihood and is provably sharper than the ELBO. TVO applies widely to continuous, discrete and non-reparametrizable distributions, with competitive performance reported. Perhaps surprisingly, TVO manages to break the tension between the bound tightness and estimator variance \citep{tao2018variational}, at the cost of sampling transition distributions along the integral path. 
%  \footnote{Upper bounds are similar derived using right sums.}

While TVO points to new directions for advancing variational inference, there are important open questions to be answered: ($a$) How is this TVO connected to the existing vast literature on VI? ($b$) Is the convenient choice of geometric path optimal? If not, what is a feasible, more favorable alternative? ($c$) Implementation-wise, what are the trade-offs between different numerical strategies? ($d$) Why TVO failed to deliver the promised exact inference even on toy problem? 

In this study, we seek answers to the above questions, in the hope that our insights and observations can serve to better understand and guide the practice of TVO as well as other advanced VI schemes. In particular, our key contributions include: ($i$) Clarification of how TVO recovers major VI schemes, including importance-weighted (IW) VAE, R\'enyi-VI and MCMC-VI; ($ii$) careful analysis of TVO wrt the more general H\"older path integrals, extending theoretical results and motivating novel VI objectives named $\HBO$; ($iii$) practical discussions on the choice of numerical schemes for thermodynamic variational schemes, covering important topics such as trade-offs and automated parameter tuning.

\vspace{-8pt}
\section{Thermodynamic Objectives for Variational Inference}
\vspace{-3pt}

We first review the basics of $\TVO$ and its role in inference, and then elaborate on the important connections to other well-established VI schemes ({\it e.g.}, IW-VAE, R\'enyi-VI and MCMC-VI). 

%\vspace{-8pt}
\vspace{-5pt}
\subsection{Thermodynamic integration and VI}
\vspace{-2pt}
\label{sec:tvo}
%\vspace{-3pt}
%\vspace{-5pt}

{\bf Thermodynamic integration.} Assume we have two unnormalized distributions $\tilde{\pi}_0(z)$ and $\tilde{\pi}_1(z)$, with $Z_0$ and $Z_1$ as their respective normalizing constants ({\it i.e.}, $\int \tilde{\pi}_i(z) \ud z / Z_i = 1, i \in \{ 0, 1\}$). Thermodynamic integration allows us to compute the $\log$ ratio $\log \frac{Z_1}{Z_0}$ between the normalization constants via integrating over a path $\tilde{\pi}_{\beta}$ that interpolates between $\tilde{\pi}_0$ and $\tilde{\pi}_1$. Let $\pi_{\beta} = \tilde{\pi}_{\beta}/Z_{\beta}$ be the normalized form of some intermediate unnormalized density $\tilde{\pi}_{\beta}$ for $\beta\in[0,1]$, with $Z_{\beta}=\int \tilde{\pi}_{\beta}(z) \ud z$ as its normalizing constant. We further denote its potential function as $U_{\beta}(z)\triangleq \log \tilde{\pi}_{\beta}(z)$. Then the following identity immediately follows from \citep{masrani2019the}
\vspace{-5pt}
\beq
\begin{array}{rcl}
\log Z_1 - \log Z_0 & = & \int_0^1 \partial_{\beta} \{ \log Z_{\beta} \} \ud \beta \\
[5pt]
& = & \int_0^1 \EE_{Z\sim\pi_{\beta}}[\partial_{\beta} U_{\beta}(Z)] \ud \beta. 
\end{array}
\label{eq:ti}
\vspace{-.5em}
\eeq

{\bf Variational inference.} For a latent variable model $p_{\theta}(x,z)$, we consider $x$ as an observation ({\it i.e.}, data) and $z$ as the latent variable we want to infer. 
The marginal likelihood $p_{\theta}(x) = \int p_{\theta}(x,z) \ud z$ typically does not have a closed-form expression, 
and to avoid direct numerical estimation of $p_{\theta}(x)$, VI instead optimizes a variational bound to the marginal $\log$-likelihood. The most popular choice is known as the {\it Evidence Lower Bound} (ELBO), given by
\beq
\ELBO \triangleq \EE_{Z\sim q_{\phi}(z|x)}\left[ \log \frac{p_{\theta}(x,Z)}{q_{\phi}(Z|x)} \right] \leq \log p_{\theta}(x), 
\eeq
where $q_{\phi}(z|x)$ is an approximation to the true posterior $p_{\theta}(z|x)$ and the inequality is a direct result of Jensen's inequality. This bound tightens as $q_{\phi}(z|x)$ approaches the true posterior $p_{\theta}(z|x)$. For estimation, we seek parameters $\theta$ that maximize the ELBO, and the commensurately learned parameters $\phi$ are often used in a subsequent inference task with new data.

% \begin{figure}[!t]
%     \centering
%     % \begin{minipage}{0.32\textwidth}
%     % \vspace{-1em}
%     % \begin{center}
%     % \includegraphics[width=.85\columnwidth]{figures/tvo/relation}
%     % \end{center}
%     % \vspace{-1.5em}
%     % \caption{Relation between bounds.\label{fig:relation}}
%     % \end{minipage}
%     \hspace{.3em}
%     \begin{minipage}{.22\textwidth}
%         \begin{center}
% %        \includegraphics[width=1.\columnwidth, trim=3in 1.4in 3in 1.5in, clip]{figures/tvo/decompose}
% 	\includegraphics[width=.9\columnwidth]{figures/tvo/grad}
% 	\end{center}
% 	\vspace{-1em}
% 	\caption{TVO gradient can be factorized into an interpolation between the ELBO and IW-ELBO gradients plus a REINFORCE gradient correction. See Eq. (\ref{eq:grad}).}
% 	\label{fig:tvo_grad}
%     \end{minipage}%
%     \hspace{.3em}
%     \begin{minipage}{0.22\textwidth}
% 	\begin{center}
% 	\includegraphics[width=1.\columnwidth, trim=.5in 0in 0 .5in, clip]{figures/tvo/increase-n.pdf}
% 	\end{center}
% 	\vspace{-2em}
% 	\caption{The $\alpha$-R\'enyi variational bound (yellow) is equivalent to partially integrated $\TVO$ bound (red) rescaled with $\frac{1}{\alpha}$. \label{fig:tvo_renyi}}
% %	\caption{Partial $\TVO$ (red), R\'enyi bound (yellow), and ELBO (blue). \label{fig:tvo_renyi}}
%     \end{minipage}
%  \vspace{-1em}
% \end{figure}

\begin{figure}[!t]
    \centering
	\includegraphics[width=.45\columnwidth]{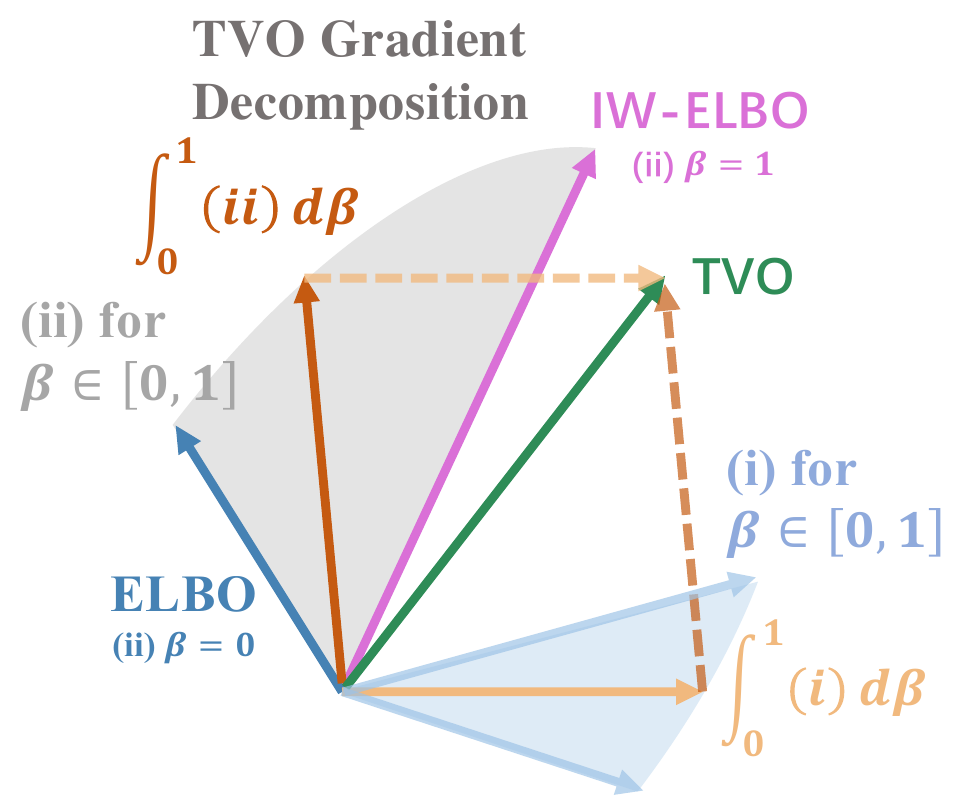}
	\includegraphics[width=.45\columnwidth, trim=.5in 0in 0 .5in, clip]{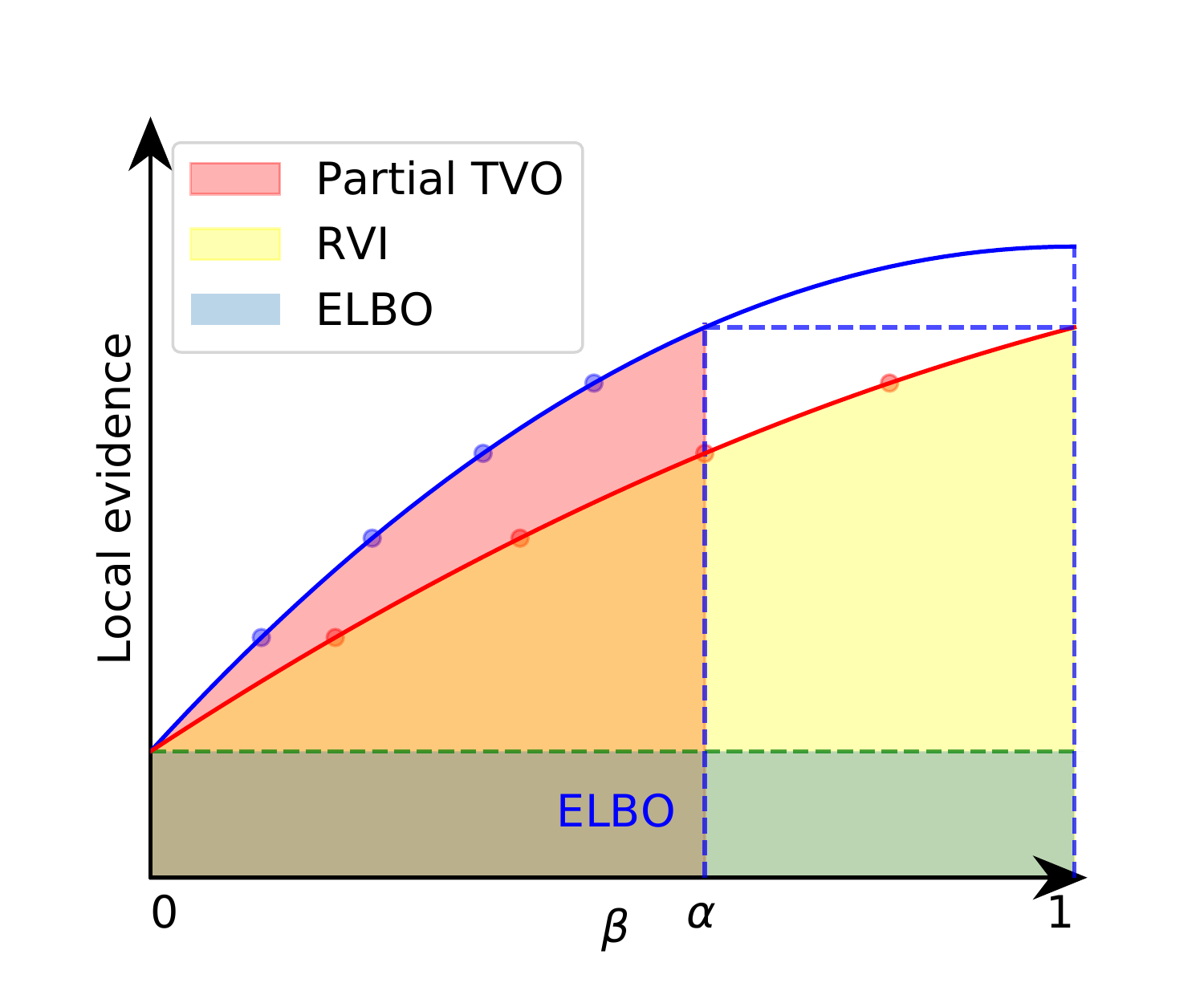}
	\vspace{-1em}
	\caption{(Left) TVO gradient can be factorized into an interpolation between the ELBO and IW-ELBO gradients (Eq.(\ref{eq:grad}) term ($ii$)) plus a REINFORCE gradient correction (Eq.(\ref{eq:grad}) term ($i$)). (Right)
	The $\alpha$-R\'enyi variational bound (yellow) is equivalent to partially integrated $\TVO$ bound (red) rescaled with $\frac{1}{\alpha}$.}
	\label{fig:tvo_grad}
 \vspace{-1em}
\end{figure}

\vspace{-5pt}
{\bf Thermodynamic variational objective.} To connect VI to TVO, let us set $\tilde{\pi}_0 = q_{\phi}(z|x), \tilde{\pi}_1 = p_{\theta}(x,z)$, and conveniently choose the {\it geometric mean} path 
\beq
\tilde{\pi}_{\beta}(z) = \tilde{\pi}_1^{\beta} \tilde{\pi}_0^{1-\beta} = p_{\theta}(x,z)^{\beta} q_{\phi}(z|x)^{1-\beta}
\eeq
between the posterior approximation $q_{\phi}$ and model distribution $p_{\theta}$. Simple computation reveals 
% \vspace{-.5em}
\beq
\resizebox{.9\hsize}{!}{$\log Z_1 = \log p_{\theta}(x), \, \log Z_0 = 0, \, \partial_{\beta}U_{\beta}(z) = \log \frac{p_{\theta}(x,z)}{q_{\phi}(z|x)}.$}
% \begin{array}{c}
% \log Z_1 = \log p_{\theta}(x), \, \log Z_0 = 0, \\
% [3pt]
% \partial_{\beta}U_{\beta}(z) = \log \frac{p_{\theta}(x,z)}{q_{\phi}(z|x)}.
% \end{array}
\eeq
Plugging these into (\ref{eq:ti}), it gives us
%\vspace{-10pt}{
\beq
%\log p_{\theta}(x) = \int_{0}^1 \CE_{\beta} \ud \beta, \text{where } \CE_{\beta} \triangleq \EE_{Z_{\beta}\sim \pi_{\beta}}\left[\log \frac{p_{\theta}(x,Z_{\beta})}{q_{\phi}(Z_{\beta}|x)}\right]. 
\resizebox{.9\hsize}{!}{$
\log p_{\theta}(x) = \int_{0}^1 \CE_{\beta} \ud \beta, \,\, \CE_{\beta} \triangleq \EE_{Z_{\beta}\sim \pi_{\beta}}\left[\log \frac{p_{\theta}(x,Z_{\beta})}{q_{\phi}(Z_{\beta}|x)}\right]. 
$}
\label{eq:tvi}
\eeq
%{\it local evidence}
To simplify our discussion, hereafter we refer to $\CE_{\beta}$ as the {\it local evidence}, which recovers the classical ELBO when $\beta=0$ and becomes an upper bound to the likelihood at $\beta=1$ (denoted as $\EUBO$). 
That is to say, when following a geometric mean path, the $\log$ evidence $\log p_{\theta}(x)$ can be expressed as the integration of local evidence along the path. 
To construct practical estimators from (\ref{eq:tvi}), one can partition the unit interval $[0,1]$ into $K$ discrete bins $\{ \beta_k \}_{k=0}^K$ and integrate the local evidence with left Riemann sum as
\beq
\TVO = \sum\nolimits_k (\beta_{k+1} - \beta_k) \CE_{\beta_k}. 
\eeq
The main result from the original TVO paper \citep{masrani2019the} is that $\CE_{\beta}$ is a non-decreasing function of $\beta$, {\it i.e.}, $\CE_{\beta} \leq \CE_{\beta'}$ for $\beta < \beta'$, which implies
% \citep{masrani2019the}
\beq
\ELBO \leq \TVO \leq \log p_{\theta}(x), 
\eeq
And the gap $\BD(q(z|x) \parallel p(z|x)) \triangleq \log p(x) - \TVO $ defines a specific divergence measure between $q(z|x)$ and $p(z|x)$ \citep{brekelmans2020all}, {\it i.e.}, $\BD(q(z|x)\parallel p(z|x))\geq 0$ with the equality holds iff $p(z|x) = q(z|x)$.

\vspace{-6pt}
\subsection{Bridging the gap}
\vspace{-4pt}

A missing piece in the original work of \citet{masrani2019the} is the TVO's connection to the more recent developments in the VI literature. We contribute this section to the discussion of the inherent connections between TVO and other prominent examples of advanced VI schemes, which all seek to improve the bound. Technical derivations are deferred to the {\it supplementary material} (SM) Sec. C.

%\begin{wrapfigure}[8]{R}{0.5\columnwidth}
%\vspace{-2em}
%\begin{minipage}{0.5\columnwidth}
%\begin{center}
%\includegraphics[width=1.\columnwidth, trim=3in 1.4in 3in 1.5in, clip]{figures/tvo/decompose}
%\end{center}
%\vspace{-1em}
%\caption{Gradient factorization for $\IWELBO$ and $\TVO$. }
%\label{fig:tvo_grad}
%\vskip -.2in
%\end{minipage}
%\end{wrapfigure}

%\begin{wrapfigure}[8]{R}{0.3\columnwidth}
%\begin{minipage}{0.3\columnwidth}
%\vspace{-2em}
%\begin{figure}
%\begin{center}
%\includegraphics[width=1.\columnwidth, trim=3in 1.4in 3in 1.5in, clip]{figures/tvo/decompose}
%\end{center}
%\vspace{-1em}
%\caption{Gradient decomposition for $\IWELBO$ and $\TVO$. }
%\label{fig:tvo_grad}
%\end{figure}
%\vskip -.2in
%\end{minipage}
%\end{wrapfigure}

\vspace{-2pt}
{\bf Importance-weighted VAE \citep{burda2015importance}} uses multiple latent samples to tighten the variational bound. In particular, the importance-weighted ELBO is given by 
\beq
\resizebox{.9\hsize}{!}{
$\IWELBO = \EE_{Z_{1:S}\sim q_{\phi}}\left[ \log \left\{ \frac{1}{S} \sum_s \frac{p_{\theta}(x,Z_s)}{q_{\phi}(Z_s|x)} \right\} \right],$
}
\eeq
where $\{Z_s\}$ are $S$ independent samples from $q_{\phi}$ and $\IWELBO$ is provably tighter than the vanilla ELBO. The empirical estimator for $\TVO$ also employed self-normalized importance weights, using $q_{\phi}$ as the proposal distribution. The importance weights for $\pi_{\beta}$ are then given by $w_s^{\beta} = \tilde{w}_s^{\beta} / (\sum_s \tilde{w}_s^{\beta})$, where $\tilde{w}_s = p_{\theta}(x,Z_s)/q_{\phi}(Z_s)$ denotes the unnormalized importance weights and $Z_s \sim q_{\phi}(z|x)$. 
$\TVO$ uses the following estimator to compute the gradients of $\lambda \triangleq (\theta, \phi)$
\setlength\arraycolsep{2pt}
\begin{equation}
\resizebox{.9\hsize}{!}{
$
\begin{aligned}\label{eq:grad}
&\nabla_{\lambda} \CE_{\beta}  =  \EE_{\pi_{\beta}}\left[ \nabla_{\lambda} f_{\lambda}\right] + \cov_{\pi_{\beta}}[\nabla_{\lambda} \log \tilde{\pi}_{\beta}(\lambda), f_{\lambda}] \\
&\nabla_{\lambda} \hat{\CE}_{\beta} = \sum_{s} \underbrace{ w_s^{\beta} \nabla_{\lambda} \log \tilde{\pi}_{\beta}(Z_s) (f_{\lambda}(Z_s) -\bar{f}_\lambda)}_{(i)} + \sum_s \underbrace{w_s^{\beta} \nabla_{\lambda} f_{\lambda}(Z_s)}_{(ii)}, 
\end{aligned}$
}
\vspace{-.5em}
\end{equation}
where $f_{\lambda}(z) \triangleq \log \frac{p_{\theta}(x,z)}{q_{\phi}(z|x)}, \bar{f}_\lambda \triangleq \sum_s w_s^{\beta} f_{\lambda}(Z_s)$. Note the ($i$) term is the REINFORCE gradient of the ELBO \citep{williams1992simple} \footnote{We note that while the original $\TVO$ paper claim to be ``REINFORCE-free'', the derived gradient estimator is actually an instantiation of the REINFORCE gradient.}. When $\beta\in\{0,1\}$, the ($ii$) term coincides with the $\ELBO$'s and $\IWELBO$'s gradient respectively. This helps enforce the view that  $\TVO$ is a generalization of standard bounds with annealed importance weights, enhanced with additional REINFORCE gradient corrections (see Figure \ref{fig:tvo_grad}).

\vspace{-2pt}
{\bf R\'enyi variational inference \citep{li2016renyi, dieng2017chi, tao2018variational}} While vanilla VI minimizes the $\KL$-divergence between the approximate and true posterior, R\'enyi-VI instead (implicitly) optimizes the more general R\'enyi-divergence \citep{renyi1961measures}, also known as the $\alpha$-divergence or $\chi$-divergence. Specifically, R\'enyi-VI targets the following objective
\beq
\RVI_{\alpha} \triangleq \frac{1}{\alpha} \log \EE_{Z\sim q_{\phi}}\left[\left(\frac{p_{\theta}(x,Z)}{q_{\phi}(Z|x)}\right)^{\alpha}\right],
\label{eq:renyi}
\eeq
where $\alpha>0$ and $\RVI_{0} \triangleq \lim_{\alpha\rightarrow0} \RVI_{\alpha}$. Note that $\RVI_{\alpha}$ is a non-decreasing function of $\alpha$, and the following relation holds
\beq
%\ELBO = \RVI_{0} \leq \RVI_{\alpha\in(0,1)} \leq \RVI_{1} = \log p_{\theta}(x) \leq \RVI_{\alpha>1}. 
\begin{array}{c}
\RVI_{0} = \ELBO, \quad \RVI_{1} = \log p_{\theta}(x),  \\
[5pt]
\ELBO \leq \RVI_{\alpha\in(0,1)} \leq \log p_{\theta}(x) \leq \RVI_{\alpha>1}. 
\end{array}
\eeq
In simple words, for $\alpha\in[0,1]$, the bound is tighter for a larger $\alpha$. The following result explicitly connects $\RVI$ to $\TVO$, showing $\RVI$ is essentially a re-scaled version of partially integrated $\TVO$ (see Figure \ref{fig:tvo_grad}).
%, thus establishing a direct link between the two schemes 
\vspace{5pt}
\begin{prop}
\small{If $\alpha \in [0,1]$, then $\RVI_{\alpha} = \frac{1}{\alpha} \int_0^\alpha \CE_{\beta} \ud \beta$}.
\end{prop}
To clarify the implications from the above Proposition, we recall two major issues with the direct implementation of $\RVI$ \citep{tao2018variational}: ($i$) practical estimators are not guaranteed to be a lower bound; ($ii$) for large $\alpha$, it suffers from the variance-tightness trade-off. These are due to the disconnection between theoretical definition and practical estimation, and the integrand in (\ref{eq:renyi}) is very unstable for larger values of $\alpha$. Our insight implies $\RVI$ can be equivalently implemented with thermodynamic integrations, which is guaranteed to be a lower bound and it partly solves the variance issue as the integrand $\CE_{\beta}$ is now expressed in the more stable $\log$-scale \footnote{This is because the use of importance re-weighting removes another major source of variance.}. Further, the R\'enyi interpretation lends insight for understanding the mode-covering behavior of the inference distribution $q_{\phi}(z|x)$. As discussed in \citet{li2016renyi}, the approximate posterior $q_{\phi}(z|x)$ transitions from mode-seeking to mode-covering as we gradually increase $\alpha$. 

{\bf MCMC variational inference \citep{salimans2015markov}} belongs to a more general family of VI schemes known as the {\it Auxiliary VI} (Aux-VI) \citep{maaloe2016auxiliary}, which seeks to improve bound sharpness via introducing additional auxiliary variables $\tilde{z}$. Due to space limit, we show details on how MCMC-VI recover $\TVO$ in the SM Sec. A. 

\vspace{-2pt}

\begin{figure*}[!t]
    \centering
    \begin{minipage}{.3\textwidth}
        \centering
        \vspace{1em}
	\includegraphics[width=.9\columnwidth]{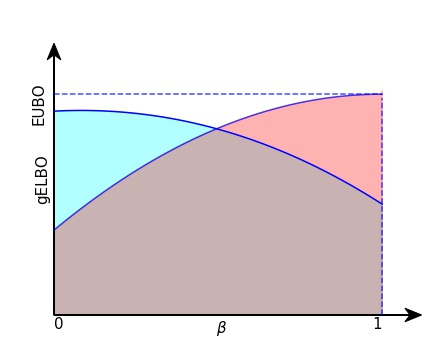}
        \vspace{-.6em}
        \caption{Comparing geometric (green) \& Wasserstein (red) thermodynamic curves, the monotonicity is flipped, but area under curve remains the same.  
        \label{fig:wtv}}
    \end{minipage}%
    \hspace{.4em}
    \begin{minipage}{0.3\textwidth}
	\begin{center}
	\includegraphics[width=.92\columnwidth,  trim=.8in 0in 0 0in, clip]{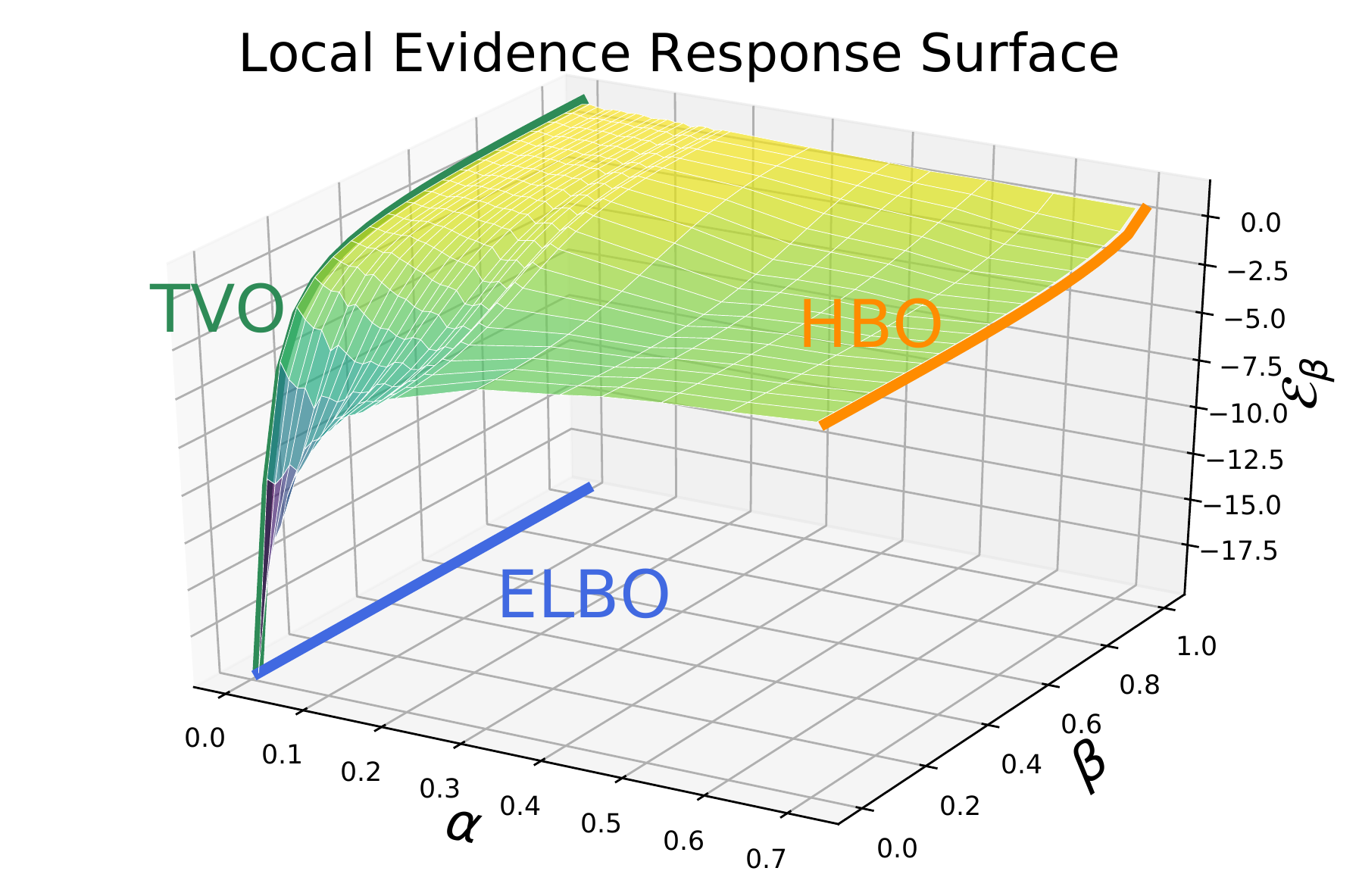}
	\end{center}
	\vspace{-1.5em}
	\caption{Local evidence surface for the H\"older paths. HBO curve flattens as $\alpha$ increases, until the monotonicity flips. }
	\label{fig:surface}
    \end{minipage}
    \hspace{.4em}
    \begin{minipage}{0.3\textwidth}
	\begin{center}
    	\includegraphics[width=.9\columnwidth]{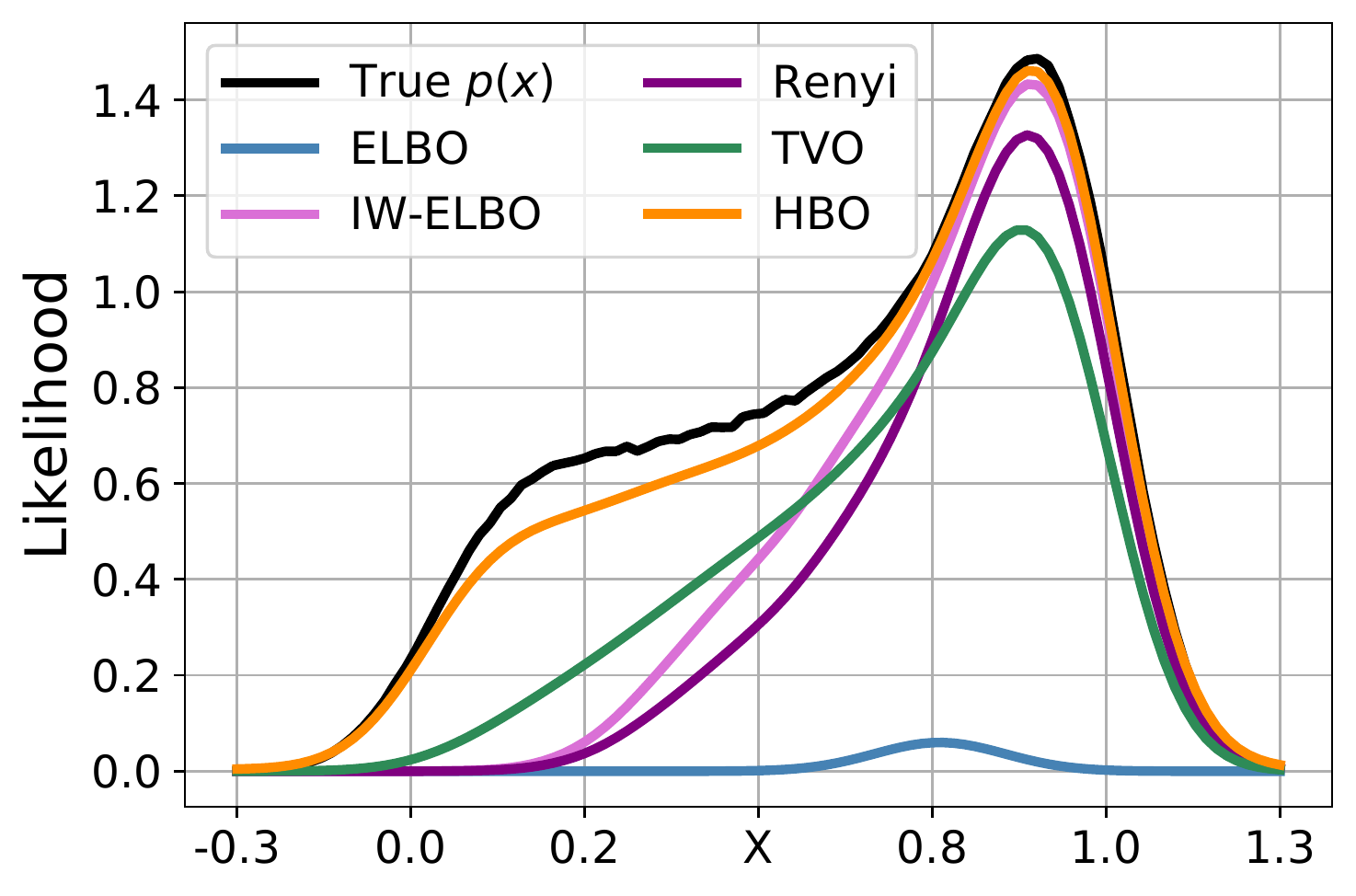}
	\end{center}
	\vspace{-1.5em}
	\caption{Comparison of approximation accuracy for different variational bounds. Bounds closer to the true $p(x)$ (black line) is considered better.}
	\label{fig:bounds}
	\vspace{-1.5em}
     \end{minipage}
\vspace{-.8em}
\end{figure*}

\vspace{-5pt}
\section{H\"older Bound Objective: A H\"older Path Analysis of Generalized TVO}
\vspace{-4pt}

While the TVO framework puts no restriction on the path $\{\tilde{\pi}_{\beta}\}$, the convenient choice of a geometric mean path yields an elegant solution. A natural question is, can we improve the results by taking an alternative path? This seems an illegitimate question at first glance, since the $\TVO$ is already a sharp bound to the $\log$-evidence. We argue that, as the original $\TVO$ paper observed, in practice TVO performs characteristically different from its theoretical predictions, and fails to outperform state-of-the-art VI counterparts. In this section, we provide a careful analysis of its failure, and generalize TVO beyond its original scope to seek remedies, inspiring a new family of thermodynamic variational bounds to close this gap. 

\vspace{-5pt}
\subsection{Limitations of TVO and a heuristic fix}
\vspace{-5pt}

We identify the main culprit for the degenerated performance of TVO as the pathological curvature of the thermodynamic curve $\CE_{\beta}$, which apparently offsets in practical terms the theoretical advantage enjoyed by $\TVO$. By taking an alternative integration path, one hopes to improve $\TVO$ via attenutating the harmful geometry: a flatter $\TVO$ curve allows a sharper approximation to the marginal $\log$ evidence given the same or smaller partition budget $K$ (Figure \ref{fig:tvo_grad}). 

To motivate, we first look at the simplest case, with the geometric mean path replaced by the {\it arithmetic mean} path\footnote{ This corresponds to follow the Wasserstein geodesics in the space of probability distributions}. We will refer to the corresponding bound as Wasserstein bound (WBO), and with a bit of computation, we have the following assertion. 

\begin{prop}
The thermodynamic curve $\CE_{\beta}^W$ for arithmetic mean path is non-increasing wrt $\beta$. Denoting the lower ($\CE_1^W$) and upper bound ($\CE_0^W$) respectively as $\WLBO$ and $\WUBO$, the following inequalities hold
\beq
%\frac{1}{p_{\theta}(x)} \ELBO(x) \leq \WLBO(x) \leq \log p_{\theta}(x) \leq \WUBO(x) \leq \EUBO(x). 
\begin{array}{c}
\frac{1}{p_{\theta}(x)} \ELBO(x) \leq \WLBO(x) \leq \log p_{\theta}(x),  \\
[5pt]
\log p_{\theta}(x) \leq \WUBO(x) \leq \EUBO(x). 
\end{array}
\label{eq:wa_ineq}
\eeq
\end{prop}
\vspace{-.5em}

%\begin{wrapfigure}[12]{R}{0.3\columnwidth}
%%\hspace{-4pt}
%\begin{minipage}{0.3\columnwidth}
%\begin{figure}
%%\vspace{-2em}
%\begin{center}
%\includegraphics[width=1.\columnwidth, trim=.4in .2in 0 .5in, clip]{figures/tvo/decrease}
%\end{center}
%\vspace{-1.5em}
%\caption{Comparing the geometric (red) and Wasserstein (cyan) thermodynamic curves, the monotonicity is flipped. \label{fig:wtv}}
%\end{figure}
%\vskip -.2in
%\end{minipage}
%\end{wrapfigure}

Unfortunately, the above bounds are not very useful for practical considerations. This is because for real-world applications, it is often expected that $p(x)\ll 1$, that is to say $\ELBO<0$. As such, the Wasserstein lower bound $\WLBO$ is much worse than the $\ELBO$, aggregating the pathological geometry that we intend to fix (see SM Sec. C for more analysis). 
However, there is a silver lining in the statement: an interesting observation is that, perhaps surprisingly, the monotonicity of the $\TVO$ curve has been flipped wrt the Wasserstein geodesic (see Figure \ref{fig:wtv}).
{ This suggests that, by looking for a continuous path $\CE_{\alpha,\beta}$ parameterized by $\alpha$ which interpolates between the geometric ($\alpha=0$) and arithmetic ($\alpha=1$) means, we might be able to find a tipping point $\alpha^*$ where the monotonicity flips}. In that case, the $\TVO$ curve is flat and we are able to evaluate the exact $\log$-likelihood at any point along the thermodynamic curve.

\vspace{-5pt}
\subsection{Flattening the curve with H\"older bounds}
\vspace{-5pt}

%\begin{wrapfigure}[13]{R}{0.5\columnwidth}
%%\hspace{-4pt}
%\begin{minipage}{0.5\columnwidth}
%\vspace{-1.5em}
%\begin{figure}
%\begin{center}
%%\includegraphics[width=1.\columnwidth, trim=.4in .2in 0 .5in, clip]{figures/tvo/decrease}
%\includegraphics[width=.92\columnwidth,  trim=.8in 0in 0 0in, clip]{figures/hbo/surface}
%\end{center}
%\vspace{-1.5em}
%\caption{Local evidence response surface. }
%\label{fig:surface}
%\end{figure}
%\vskip -.2in
%\end{minipage}
%\end{wrapfigure}

In order to generalize the integration path, we consider the more general H\"older averaging operation. 
\begin{defn}[Weighted H\"older mean path] For $a,b \in \BR_+$ and $\beta \in [0,1]$, the weighted H\"older mean $ \CM_{\alpha}(a,b;\beta) \triangleq \left[\beta a^\alpha+(1-\beta)b^{\alpha}\right]^{\frac{1}{\alpha}}$, and for $\alpha=0$ we use $\CM_{0} = \lim_{\alpha \rightarrow 0} \CM_{\alpha}$.
\end{defn}
%\hl{In Figure S1 in the SM, we visualize such H\"older paths, with special examples at $\alpha=0, 1$ respectively recovering the geometric and arithmetic means.} 
{See Figure S6 in the SM for examples of H\"older paths.}
We can analogously define the thermodynamic curves wrt H\"older paths, with the following statement directly generalizing the monotonicity result.

\begin{prop}
%[H\"older thermodynamic curve and its mononticity]
Under the weighted H\"older mean path $\tilde{\pi}_{\alpha,\beta} = \left[\beta \tilde{\pi}_1^\alpha+(1-\beta)\tilde{\pi}_0^{\alpha}\right]^{\frac{1}{\alpha}}$, the H\"older thermodynamic curve given by
\beq
\CE_{\alpha,\beta} = \mathbb{E}_{\pi_{\alpha,\beta}}\left[\frac{1}{\alpha}\frac{\tilde{\pi}_1^\alpha - \tilde{\pi}_0^\alpha}{\tilde{\pi}_{\alpha, \beta}^\alpha}\right] 
\eeq
is non-decreasing with respect to $\beta$ for $\alpha\leq 0$ and non-increasing for $\alpha\geq 1$.
\label{thm:g_mono}
\end{prop}

While Proposition \ref{thm:g_mono} does not confirm the existence of such a conjectured tipping point when the curve becomes flat, it predicts where we might be able to find one if it exists. More specifically, such a phase transition point might exist in $\alpha \in (0,1)$ for H\"older paths. 
So motivated, we propose the following new family of H\"older variational objectives that promise to rectify the pathological geometry of $\TVO$. 

%\vspace{-5pt}
\begin{defn}[{\bf H}\"older {\bf BO}unds ($\HBO$)] 
\beq
\HBO_{\alpha} \triangleq \int_0^1 \CE_{\alpha,\beta} \ud \beta, \text{ for } \alpha \in (0, 1). 
\eeq  
%{\small (Note $\HBO_{\alpha} = \TVO = \log p_{\theta}(\xv)$.)}
\end{defn}

\begin{coro}
\label{thm:hbo}
$\HBO_{\alpha} = \log p_{\theta}(\xv)$. 
\end{coro}
%\vspace{-5pt}
Corollary \ref{thm:hbo} is a direct consequence of the thermodynamic integration equality (\ref{eq:ti}), since the H\"older mean path $\tilde{\pi}_{\alpha,\beta}$ connects $\pi_0 = q_{\phi}(z|x)$ and $\pi_1 = p_{\theta}(x,z)$. With slight abuse of notation, we denote 
\beq
\HBO_{\alpha} \triangleq \sum_k (\beta_{k+1} - \beta_k) \CE_{\alpha, \beta_k}
\eeq
as the empirical $\HBO$ estimator, which can be understood as approximations to the $\log$-likelihood. 

Figure \ref{fig:surface} provides a more intuitive picture via visualizing the H\"older thermodynamic curves wrt different $\alpha$ (using toy model from our Experiment section). In this example, a non-trivial choice of $\alpha$ significantly improves the pathological curvature observed in vanilla $\TVO$, with a wide range of $\alpha$ yielding relatively flat geometry. In Figure \ref{fig:bounds}, we further show comparisons of different popular variational bounds, presented in the original likelihood scale for better visualization. While our $\HBO$ consistently outperforms all its counterparts, $\TVO$ underperforms strong baselines ({\it e.g.}, $\IWELBO$, $\RVI$) in many of the regions.

\vspace{-6pt}
\subsection{Practical estimation of $\HBO$}
\vspace{-4pt}
\label{sec:imp}

{\bf Finding an appropriate $\alpha$.} The choice of $\alpha$ is crucial for the performance of $\HBO$. For a good $\alpha$, one can optimally approximate $\log p_{\theta}(\xv)$ with a minimal number of partitions, thus greatly reducing computational overhead. We consider two simple strategies to pick the best $\alpha$ from a candidate set. 
\vspace{-8pt}
\begin{itemize}
\item {\it Trial \text{$\&$} error.} Sample a few test points $\{\beta_k \in [0,1]\}$, evaluate $\CE_{\alpha, \beta}$ at all $\beta_k$ on candidate $\alpha$, choose the one with minimal gap $\hat{\alpha} = \argmin_{\alpha} \{ \max_{\beta} \{ \CE_{\alpha, \beta} \} - \min_{\beta} \{ \CE_{\alpha, \beta} \}  \}$. 
\vspace{-.5em}
\item {\it Binary search.} Assuming the monotonicity of the $\HBO$ curve holds, one can use root finding binary search to efficiently locate the optimal $\alpha$. Specifically, initialize with $0\leq\alpha_L<\alpha_R\leq1$, such that $\CE_{\alpha_L, \beta}$ and $\CE_{\alpha_R, \beta}$ are respectively monotonically increasing and decreasing wrt $\beta$. Pick $\alpha_M$ in between $\alpha_L$ and $\alpha_R$ ({\it e.g.}, mid point).  If $\CE_{\alpha_M, \beta}$ is increasing wrt $\beta$, set $\alpha_L\leftarrow \alpha_M$, otherwise $\alpha_R \leftarrow \alpha_M$. Repeat this until some stopping criteria is met ({\it e.g.}, $|\alpha_R-\alpha_L|$ or slope of $\CE_{\alpha_M, \beta}$ fall below some tolerance threshold). 
%Specifically, let $\alpha_l$ be ordered, and denote $\alpha_L$ as the left bound ($\uparrow$) and $\alpha_R$ be the right bound ($\downarrow$). Pick $\alpha_M$ in between $\alpha_L$ and $\alpha_R$, if $\CE_{\alpha_M, \beta}$ is $\uparrow$, set $\alpha_L\leftarrow \alpha_M$, otherwise $\alpha_R \leftarrow \alpha_M$. Repeat until a stopping criteria is met (under estimation error). 
\end{itemize}
\vspace{-8pt}
The binary search approach is more efficient, but less reliable if the underlying assumption is violated. Note since all $\CE_{\alpha, \beta}$ based on finite-sample empirical estimate using ({\it i.e.}, a mini-batch of $z$ sampled from the proposal distribution), one should also properly account for the estimation variance involved. 

\vspace{-2pt}
{\bf Importance-weighted sampling of H\"older path.} To estimate $\CE_{\alpha, \beta}$ one needs to draw samples from the intermediate distributions $\tilde{\pi}_{\alpha, \beta}$ along the H\"older path. To avoid the excessive computation entailed by Markov chain Mote-Carlo (MCMC) schemes, we adopt a similar importance weighting strategy employed by the original $\TVO$. In particular, one draws $B$ samples $\{ z_i \}_{i=1}^B$ from the approximate posterior $q_{\phi}(z|x)$, and then adjusts according to the importance weights $\tilde{w}_i^{\beta} \triangleq \tilde{\pi}_{\alpha, \beta}(x, z_i) / q_{\phi}(z_i)$. After some algebraic manipulations, the importance weighted local evidence can be expressed as
\beq
\hat{\CE}_{\alpha,\beta}^{\IW} = \frac{1}{\alpha \sum_{i'} (\beta s_{i'} +1)^{1/\alpha}} \sum_i \frac{s_i}{(\beta s_i +1)^{1-1/\alpha}},
\eeq
where $s_i \triangleq \left(\frac{p_{\theta}(x,z_i)}{q_{\phi}(z_i)}\right)^\alpha - 1$. The corresponding importance-weighted $\HBO$ thus writes
\beq
\widehat{\HBO}_{\alpha}^{\IW} = \sum_k (\beta_{k+1} - \beta_k) \hat{\CE}_{\alpha,\beta_k}^{\IW}.
\vspace{-1em}
\eeq
%where $s_i \triangleq \left(\frac{p_{\theta}(x,z_i)}{q_{\phi}(z_i)}\right)^\alpha - 1$. 

%For inference of complex distributions, drawing from the exact intermediate distributions $\tilde{\pi}_{\alpha, \beta}$ using Markov chain Mote-Carlo (MCMC) techniques such as Langevin simulation or Metropolis-Hasting are usually considered infeasible, due to excessive computational overhead. 

%We consider the following two alternative estimators
%\begin{itemize}
%\item {\bf Importance weighted estimator (IW).} This strategy is similar to the one adopted in the original $\TVO$ paper. One draws samples from the approximate posterior $Z_s \sim q_{\phi}(z|x)$, and then adjusts according to the importance weights $\tilde{w}_s^{\beta} \triangleq \tilde{\pi}_{\alpha, \beta}(x, Z_s) / q_{\phi}(Z_s)$. The importance weighted local evidence is thus giving by 
%\beq
%\hat{\CE}_{\alpha,\beta}^{\IW} = \frac{1}{\alpha \sum_{i'} (\beta s_{i'} +1)^{1/\alpha}} \sum_i \frac{s_i}{(\beta s_i +1)^{1-1/\alpha}}, 
%\eeq
%where $s_i \triangleq \left(\frac{p_{\theta}(x,z_i)}{q_{\phi}(z_i)}\right)^\alpha - 1$. 
%%\item {\bf Markov chain Mote-Carlo estimator (MCMC).} We draw $Z_{\alpha, \beta}$ directly from $\tilde{\pi}_{\alpha, \beta}$ using Langevin dynamics. 
%\end{itemize}

{\bf The perturbed HBO.} While in theory we can directly simulate any H\"older path based on its definition, we might not want to do so for numerical considerations with a larger $\alpha$. This is because $\tilde{\pi}_0$ and $\tilde{\pi}_1$ typically live on very different scales, and hence a brute-force treatment can potentially lead to catastrophic numerical overflow. Instead, a more interesting regime is where $\alpha$ is close to zero, in which case both distributions are evaluated near a more comparable $\log$-scale. Inspired by the perturbation argument originally presented in \citet{bamler2017perturbative}, we consider a linear expansion near $\alpha=0$ wrt the H\"older path. As summarized by the following statement, the perturbed $\HBO$ admits a simple expression. 

\begin{prop}[Perturbed HBO] 
\label{prop:phbo} 
For a sufficiently small H\"older parameter $\delta \ll 1$, we have the following approximation for integrand $\partial_{\beta} U_{\delta,\beta}$ and sampling distribution $\tilde{\pi}_{\delta, \beta}$
%\vspace{-3pt}
\vspace{-1em}
\setlength\arraycolsep{2pt}
\beq
\resizebox{.85\hsize}{!}{$
\begin{array}{rcl}
\partial_{\beta} U_{\delta, \beta} & \approx & \log \frac{p_{\theta}(x,z)}{q_{\phi}(z|x)} + \left(\frac{1}{2} - \beta\right)\left[ \log \frac{p_{\theta}(x,z)}{q_{\phi}(z|x)} \right]^2 \delta , \\
[8pt]
\log\tilde{\pi}_{\delta, \beta} & \approx & [\beta\log\tilde{\pi}_1 +(1-\beta)\log\tilde{\pi}_0]+ \\
[5pt]
& & \frac{1}{2}\left[\beta(\log\tilde{\pi}_1)^2 +(1-\beta)(\log\tilde{\pi}_0)^2\right] \delta . 
\end{array}
$}
\eeq
\setlength\arraycolsep{5pt}
\end{prop}
\vspace{-8pt}

Proposition \ref{prop:phbo} allows us to calibrate vanilla $\TVO$ with first-order corrective terms to approximate $\HBO$, which hopefully helps to close the gap between the lower and upper bounds of a $\TVO$ curve, thereby ``flattening'' curvature for improved performance. 

\vspace{-6pt}
%\subsection{Related Work}
\section{Related Work}
\vspace{-4pt}

% \begin{figure}[h]
% \begin{center}
% \includegraphics[width=.48\textwidth]{figures/hbo/hbo_k_err}
% \includegraphics[width=.48\textwidth]{figures/hbo/hbo_tvo_k}
% \end{center}
% \vspace{-1em}
% \caption{Comparison of $\TVO$ and $\HBO$ bounds with different partition budget $K$.  \label{fig:hbo_tvo_k}}
% \vspace{-1.em}
% \end{figure}

\begin{figure*}[!ht]
    \centering
        \begin{minipage}{.25\textwidth}
        \centering
        \includegraphics[width=1.\textwidth]{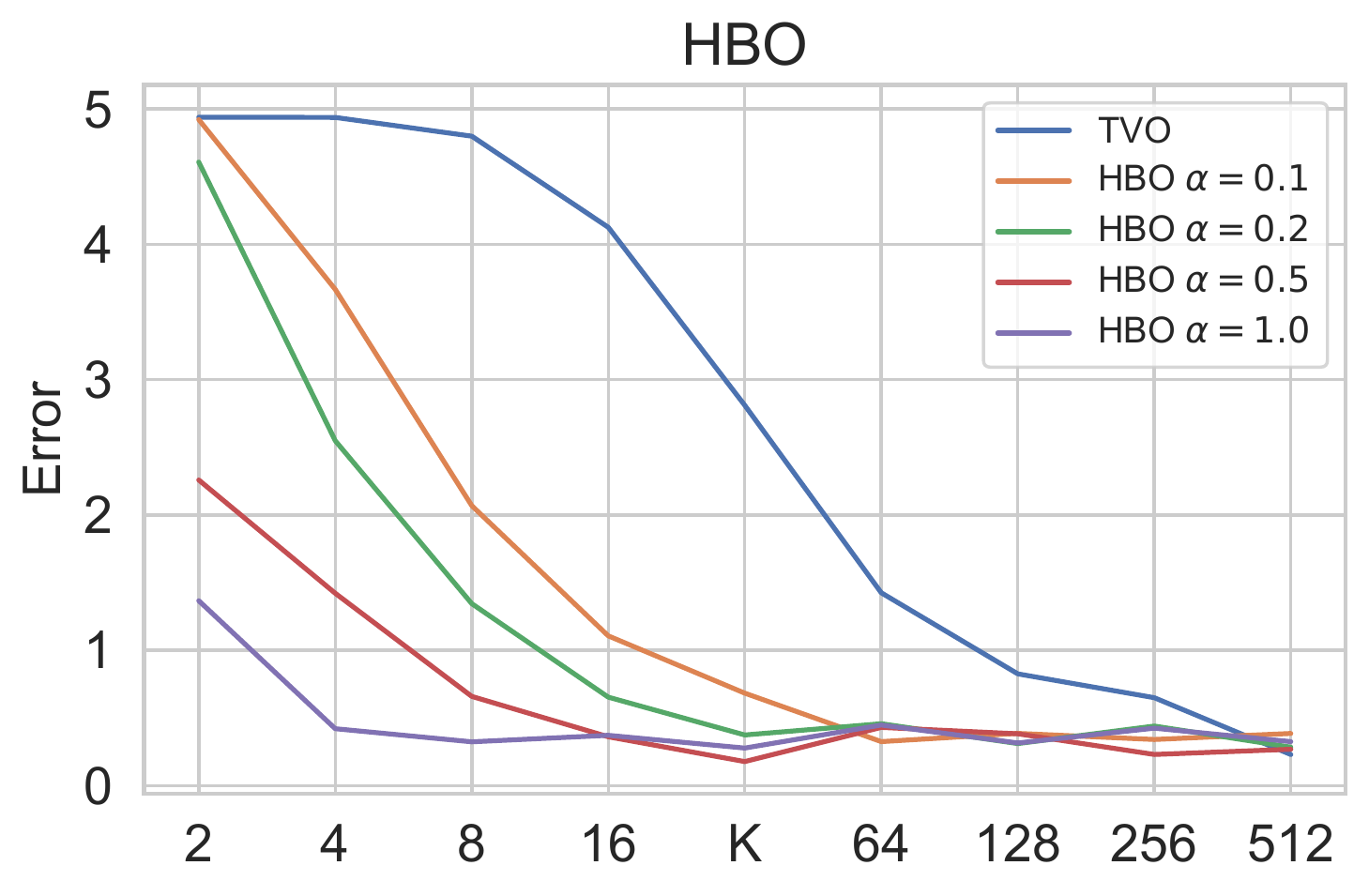}
        \vspace{-2em}
        \caption{Approximation error for $\TVO$ \& $\HBO$ with different partition size $K$.\label{fig:hbo_k_err}}
    \end{minipage}%
    \hspace{8pt}
    \begin{minipage}{.45\textwidth}
        \centering
        \vspace{-.5em}
        \includegraphics[width=0.34\textheight]{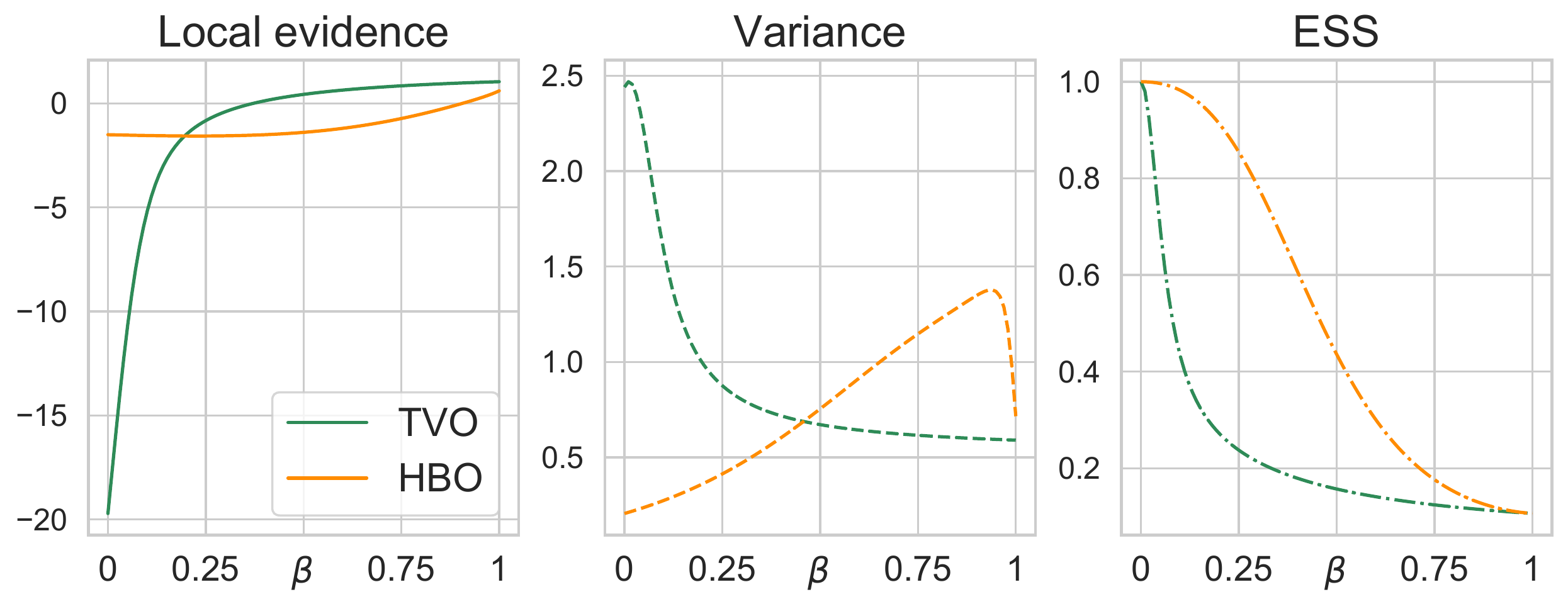}
        \vspace{-1.5em}
        \caption{Evidence curve estimation variance and effective-sample size.}
        \label{fig:var_ess}
    \end{minipage}%
    \hspace{8pt}
        \begin{minipage}{.25\textwidth}
        \centering
	\includegraphics[width=1.\textwidth]{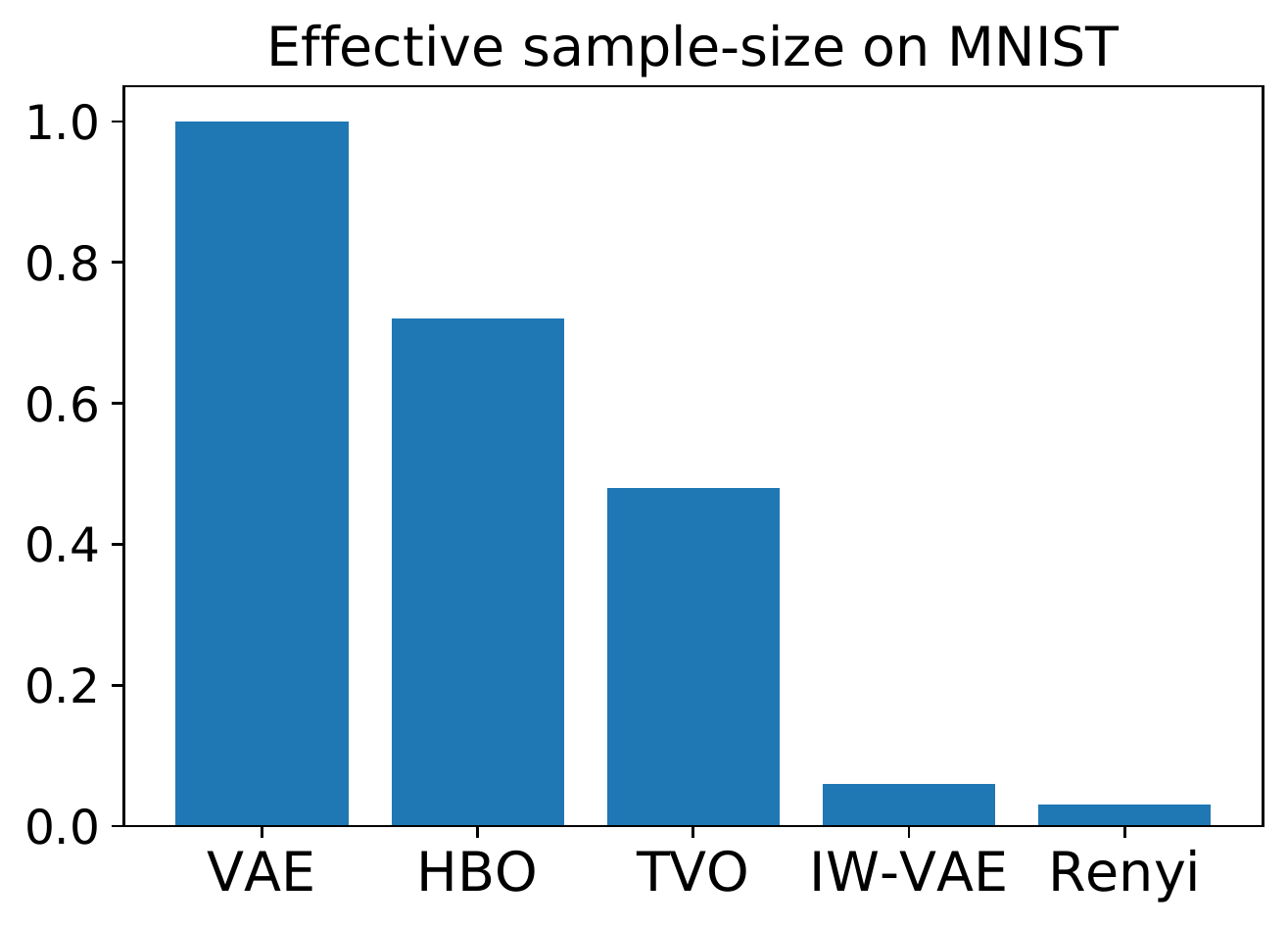}
        \vspace{-2em}
        \caption{Effective sample-size of different VI objectives on MNIST. \label{fig:ess_cmp}}
    \end{minipage}%
    % \begin{minipage}{0.5\textwidth}
    %     \centering
    %     \includegraphics[width=0.193\textheight]{figures/mnist/mnist_elbo}
    %     \includegraphics[width=0.178\textheight, trim=.5in 0in 0in 0in, clip]{figures/mnist/mnist_iwelbo}
    %     \vspace{-1.6em}
    %     \caption{MNIST test sample evaluation.}
    %     \label{fig:mnist_training}
    % \end{minipage}
\vspace{-1.5em}
\end{figure*}

% \begin{figure*}[!ht]
%     \centering
%     \begin{minipage}{.5\textwidth}
%         \centering
%         \includegraphics[width=0.34\textheight]{figures/hbo/compare}
%         \vspace{-.6em}
%         \caption{Variance and effective-sample size.}
%         \label{fig:prob1_6_2}
%     \end{minipage}%
%     \begin{minipage}{0.5\textwidth}
%         \centering
%         \includegraphics[width=0.193\textheight]{figures/mnist/mnist_elbo}
%         \includegraphics[width=0.178\textheight, trim=.5in 0in 0in 0in, clip]{figures/mnist/mnist_iwelbo}
%         \vspace{-1.6em}
%         \caption{MNIST test sample evaluation.}
%         \label{fig:mnist_training}
%     \end{minipage}
% \vspace{-1.5em}
% \end{figure*}

%Our work is in line with the large literature on tightening variational bounds. 
{\bf Trade-offs for tightening the variational bounds} have drawn extensive discussions in recent literature. While  \citet{rainforth2017tighter} has argued that tighter bounds may inadvertently hurt inference, we show experimentally that the proposed $\HBO$ transcends this trade-off by providing arguments based on both reduced variance and improved effective sample-size, see our experiments for details. In a similar spirit to \citet{tao2018variational}, we have leveraged geometric perspectives to improve VI bounds. Closely related is the work of \citet{brekelmans2020all}, where the author(s) analyzed the variational gap of $\TVO$ in generalized divergence metrics and proposed a smart partition scheme for reduce the estimation gap. \citet{habeck2017model} also discussed issues related to the construction of $\TVO$ from physics perspectives. It is possible to further extend the scope to intractable approximate posteriors using adversarial schemes \citep{mescheder2017adversarial, dai2018coupled}. 
%Limited by space, we leave more discussions to the SM. 

\vspace{-2pt}
{\bf Energy-based modeling} has close ties to $\TVO$: ($i$) the marginal likelihood is the partition function of an energy model defined by the joint density; and ($ii$) $\TVO$-type procedures require sampling from intermediate distributions on the thermodynamic path defined by the energy models. Classical energy-based approaches have focused on estimating the parameters of an unnormalized energy model: MCMC-MLE \citep{geyer1991markov, geyer1994convergence}, {\it contrastive divergence} \citep{hinton2002training, du2019implicit, nijkamp2019learning}, {\it score matching} \citep{hyvarinen2005estimation}, {\it noise contrastive estimation}  \citet{gutmann2010noise, gutmann2012noise, gutmann2012bregman}. These schemes often lack a generative perspective and scale less competitively to data complexity. 
%For historical context, early variational schemes have been explicitly associated with the name free-energy minimization \citet{friston2007variational}. Classical developments for energy-based modeling focused on the estimation of parameters of an unnormalized energy model, rather than a generative procedure. Such efforts include Markov chain Monte Carlo MLE (MCMC-MLE) \citet{geyer1991markov, geyer1994convergence}, {\it contrastive divergence} (CD) \citep{hinton2002training, du2019implicit, nijkamp2019learning}, {\it score matching} (SM) \citet{hyvarinen2005estimation}, {\it noise contrastive estimation} (NCE)  \citet{gutmann2010noise, gutmann2012noise, gutmann2012bregman}. An exception is the Restricted Boltzmann machine (RBM) \citet{hinton2006reducing}, which bridges the estimation and sampling perspectives using energy based model. 
Recently, energy-based perspectives have gain renewed interests in probabilistic learning, Of particular interests are the Stein variational inference \citep{liu2016stein, pu2017vae} and Fenchel mini-max learning procedures \citep{dai2018coupled, tao2019fenchel}. See also  \citet{brekelmans2020all} for a theoretical interpretation of $\TVO$ variational gaps using the energy-based exponential family distributions. 
% \vspace{-2pt}

{\bf Simulated annealing and thermodynamic integration} have deep connections \citep{neal2001annealed, frenkel2001understanding}, and both have been used to approximate the intractable partition function \citep{ogata1989monte}. Such computational strategies are deeply rooted in non-equilibrium statistical physics \citep{habeck2012evaluation, crooks1999entropy, habeck2012evaluation, habeck2017model}, and have been used for evaluating the data likelihoods for complex models \citep{wu2017quantitative}. While the geometric mean path is dominantly popular,  alternative thermodynamic integration paths have also been considered in prior literature:  \citet{gelman1998simulating} derived a few optimal integration paths for special cases, and \citet{grosse2013annealing} explored the moment averaged path in the exponential family. \citet{brekelmans2020all} used the moment averaged path to motivate a novel non-uniform partition strategy for $\TVO$ that adapts to the shape of the thermodynamic curve. Our $\HBO$ instead finds a good thermodynamic path that properly ``bends'' the curve. A similar perspective was adopted by concurrent work \citep{masrani2021q} which has different focuses (see SM for clarifications). 
% \textcolor{olive}{q-path \citep{masrani2021q, zimmermann2021nested}}

%While we have used tractable approximate posteriors,

\vspace{-7pt}
\section{Experiments}
\vspace{-3pt}

To validate our HBO framework, we consider a wide range of experiments, with synthetic \& real-world datasets and state-of-the-art variational schemes. All experiments are implemented with PyTorch \footnote{Our code: \url{https://github.com/author_name/HBO}} and executed on a single NVIDIA TITAN Xp GPU. Details of the experimental setup and extended results are provided in the SM Sec. D. Note our experiments focus on validating theoretical aspects of the HBO framework, the establishment of new state-of-the-art results is beyond the scope of this study.

% To validate the proposed HBO framework and benchmark it against modern variational schemes, we consider a wide range of experiments, with synthetic and real-world datasets. All experiments are implemented with PyTorch and executed on a single NVIDIA TITAN Xp GPU. Details of the experimental setup and extended results are provided in the SM Sec. D, due to space limits. Our code is available from \url{https://www.github.com/author_name/HBO}. Note our goal is to probe the practical aspects of the HBO framework, to explore the synergies with its counterparts and to verify HBO works favorably or similarly compared with competing solutions under the same setup. Establishing new state-of-the-art results is beyond the scope of this study and is left for future work.

\vspace{-6pt}
%\subsection{From $\TVO$ to $\HBO$}
\subsection{Synthetic examples: from $\TVO$ to $\HBO$}
\vspace{-3pt}

%\begin{wrapfigure}[11]{R}{0.3\columnwidth}
%%\hspace{-4pt}
%\begin{minipage}{0.3\columnwidth}
%\begin{figure}
%%\vspace{-2em}
%\begin{center}
%\includegraphics[width=1.\columnwidth]{figures/hbo/bounds_cmp}
%\end{center}
%\vspace{-1.5em}
%\caption{Bounds comparison.}
%\label{fig:bounds}
%\end{figure}
%\vskip -.2in
%\end{minipage}
%\end{wrapfigure}

To expose the limitations of existing solutions and demonstrate the appealing features of $\HBO$ with a simple example, we synthesize our toy data from $x \sim \CN(\sin(z), 10^{-2}), z \sim \CN(0, 1)$, and fixed our approximate posterior to be $q(z|x) = \CN(0, 1.5^2)$. Unless otherwise specified, we employ the uniform partition strategy for both $\TVO$ and $\HBO$. Observations from real data are characteristically similar to this toy. 
% We design this example to expose the limitations of existing solutions and to highlight the practical advantages enjoyed by $\HBO$, and we note similar observations are made in real data. 

% To demonstrate the appealing features of $\HBO$, and to provide practical guidance for researchers and practitioners, we extensively evaluate $\HBO$'s performance from various aspects. 
% We synthesize our toy data from $x \sim \CN(\sin(z), 10^{-2}), z \sim \CN(0, 1)$, and have fixed our approximate posterior to be $q(z|x) = \CN(0, 1.5^2)$. Unless otherwise specified, we employ the uniform partition strategy for both $\TVO$ and $\HBO$. We design the toy examples to expose the limitations of existing solutions and to highlight the practical advantages enjoyed by $\HBO$, and we note similar observations are made in real data. 

\begin{figure*}[t!]
\begin{minipage}{1.05\columnwidth}
\begin{sc}
\small
\vspace{-.5em}
\captionof{table}{ MNIST \& Omniglot
 results {\small ($\uparrow$ is better)} \label{tab:mnist}}
\vspace{-5pt}
 \resizebox{1.\columnwidth}{!}{
\begin{tabular}{ccc@{\hskip 1em}cc}
\toprule
& \multicolumn{2}{c}{MNIST} & \multicolumn{2}{c}{Omniglot} \\
Test & $\ELBO \uparrow$ & $\IWELBO \uparrow$ & $\ELBO \uparrow$ & $\IWELBO \uparrow$ \\
\midrule
$\ELBO$ & -94.00 & -89.34 & -117.81 & -108.12 \\
$\IWELBO$ & -96.65 & -88.31 & -118.78 & -108.13 \\
R\'enyi &  -94.82 & -88.55 & -117.68 & -107.89 \\
$\TVO$ & -96.30 & -88.27 & -117.60 & -107.88\\
[2pt]
$\HBO$ & {\bf -93.12} & {\bf -87.82} & {\bf -116.86} & {\bf -107.70} \\
\bottomrule
\end{tabular}
}
\end{sc}
\end{minipage}
\begin{minipage}{1.0\columnwidth}
\includegraphics[width=1.\textwidth]{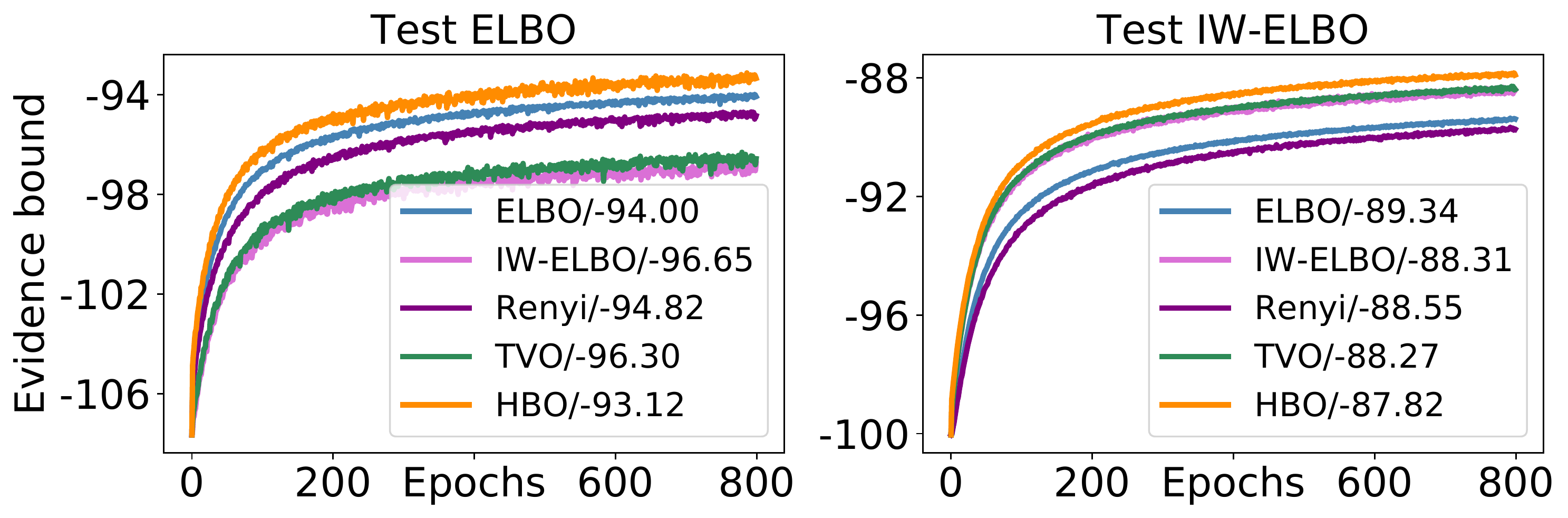}
        \vspace{-2.3em}
        \caption{MNIST results on test data.}
        \label{fig:mnist_training}
\end{minipage}
% \vspace{-.5em}
\end{figure*}

\vspace{-2pt}
{\bf Bound sharpness} is the primary concern of this study. We compare the tightness of $\TVO$, $\HBO$ and other competing variational bounds under various parameter configurations (partition number $K$, batch-size $B$, curvature $\alpha$, etc.). In Figure \ref{fig:bounds}, we visualize the different popular bounds with matched configurations ($K=100, B=10, \alpha=0.8$), along with the true likelihood $p(x)$. We see $\HBO$ consistently provides a tighter approximation relative to its counterparts. 
% \hl{Note that while R\'enyi bounds are theoretically equivalent to $\TVO$, its empirical estimator acts quite differently from the theoretical counterpart, often resulting in unpredictable behavior. In practice we have observed empirically $\TVO$ variants of R\'enyi are both more numerically stable and consistently provides tighter approximation.}  
{In Figure \ref{fig:hbo_k_err} we plot the approximation error $e(\hat{p}) = \int p(x) | p(x) - \hat{p}(x) | \ud x$ under different practical computational budgets, showing $\HBO$ significantly improves over $\TVO$ given the same budget.} Note however, a sharper bound alone does not necessarily imply a better objective for variational learning \citep{rainforth2017tighter}. In the next experiment we examine other aspects that matter for a good variational objective. 
% More experimental results and analyses are summarized in the SM.

\vspace{-2pt}
{\bf Estimation variance and effective sample size} are two factors that greatly affect the learning efficiency of VI. Low-variance estimators generally directly promote fast convergence  \citep{mnih2016variational, jang2016categorical, roeder2017sticking, tucker2018doubly}; and the {\it effective sample-size} (ESS), or sample efficiency, describes on average how much does an individual sample contribute \citep{liu2001monte}. For a weighted representation, the normalized ESS is defined as $\ESS \triangleq \frac{\sum_i \tilde{w}_i}{m \sum_i \tilde{w}_i^2} \in [\frac{1}{m},1]$, where $\{\tilde{w}_i\}$ denote the unnormalized weights and $m$ is the original sample size. For a small $\ESS$, only a small fraction of the samples are contributing; while for  $\ESS=1$, each sample contributes equally.

% the weighted distribution is dominated by a small fraction of the samples, with the majority of samples not effectively contributing information. For $\ESS=1$, each sample contributes equally.

Figure \ref{fig:var_ess} compares the estimation variance and ESS along the thermodynamic path for $\TVO$ and its much flatter counterpart $\HBO$. An interesting observation is that the variance shoots up when thermodynamic curves encounter a higher curvature. This partly explains why the original $\TVO$ is less competitive, as the partition points are mostly placed in the high-variance regions (small $\beta$).  While the sample efficiency achieves its maximum at the $\ELBO$ ($\beta=0$), it undergoes a sharp drop as we move along the $\TVO$ curve. $\HBO$ shows more robustness against the deterioration in sample efficiency and a much better variance profile in the low-temperature regime. In Figure \ref{fig:ess_cmp}, we compare averaged ESS of difference VI bounds on the real-world \texttt{MNIST} data. Our finding suggests smarter partition strategies should account for both variance and $\ESS$ to strike a better deal for the bias-variance trade-off in the design of empirical estimators.

\begin{figure*}[!ht]
    \centering
        \begin{minipage}{.68\textwidth}
        \centering
        \includegraphics[width=\textwidth]{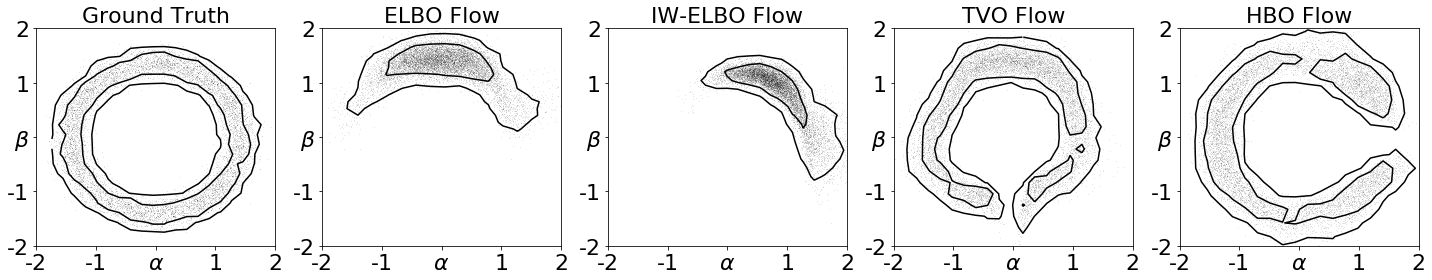}
    \label{fig:br_loss}
    \vspace{-1.8em}
    \caption{Posterior approximation with masked auto-regressive flows. \label{fig:vi_approx}}
    \end{minipage}%
    \hspace{5pt}
    \begin{minipage}{.3\textwidth}
        \centering
        \vspace{-.5em}
        \includegraphics[width=.9\textwidth]{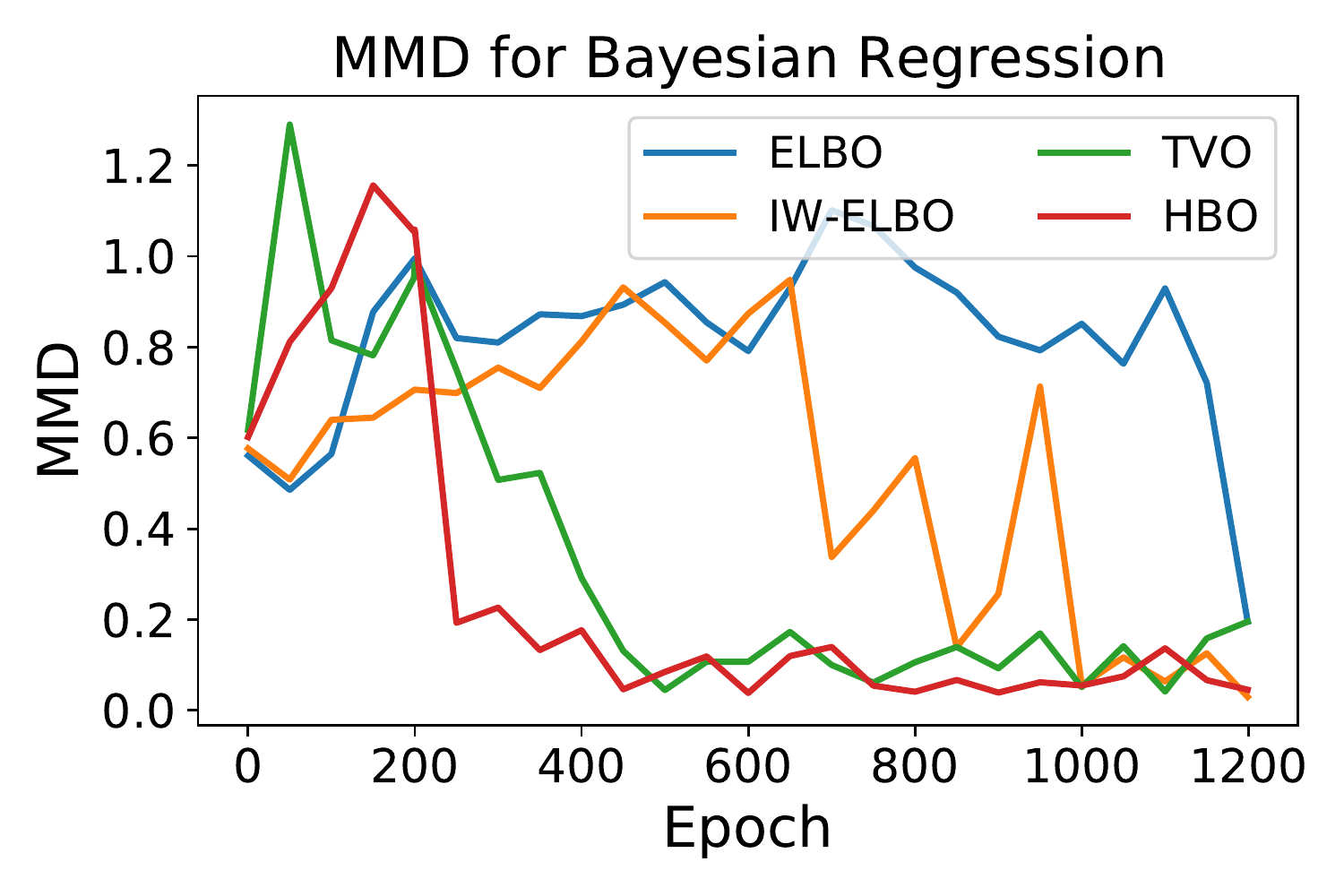}
    \label{fig:br_loss}
    \vspace{-1.5em}
    \caption{Posterior convergence for different bounds. \label{fig:mmd}}
    \end{minipage}%
\vspace{-1.em}
\end{figure*}

% \begin{figure*}[t!]
%     \centering
%     \includegraphics[width=\textwidth]{figures/inference/bayesian_five.png}
%     \label{fig:br_loss}
%     \vspace{-1.8em}
%     \caption{Complex posterior approximation with masked auto-regressive flows. \label{fig:vi_approx}}
% \vspace{-1.em}
% \end{figure*}

{\bf Matching complex posterior distributions.} We hypothesized that thermodynamic schemes did not outperform regular objectives because they only test simple posterior approximations. As such, we consider the more challenging case: $y = \sqrt{z_1^2+z_2^2}+\xi, z_1,z_2 \sim \CN(0,1), \xi \sim \CN(0,0.1^2)$,
and the goal is to infer pair $(z_1, z_2)$ given an observation $y=1$, whose posterior will be highly nonlinear (see Figure \ref{fig:vi_approx}). To model such complex posterior distribution, we use the {\it masked auto-regression flow} (MAF) \citep{papamakarios2017masked} to model the approximate posterior. We observe that neither $\ELBO$ nor $\IWELBO$ converged (Figure \ref{fig:vi_approx}). Both thermodynamic schemes give reasonable approximations, but $\HBO$ converges slightly faster and to a better solution. 

{\bf Bayesian regression.} We also benchmark $\HBO$'s performance against other bounds using the Bayesian regression model: $Y = \alpha + \beta \tilde{X} + \epsilon, \quad X = \tilde{X} +\zeta, p(\alpha,\beta) \propto (1+\beta^2)^{-3/2}, p(\sigma) \propto \frac{1}{\sigma}$ (details in SM Sec. E). We examine the convergence of different VI criteria quantitatively using the {\it maximal mean discrepancy} (MMD) metric \citep{gretton2012kernel} to evaluate the similarity between the ground-truth posterior (computed from the  \texttt{emcee} package) to the variational approximations. In Figure \ref{fig:mmd}, we see $\HBO$ delivers the fastest convergence to ground-truth, followed by $\TVO$ and then $\IWELBO$. Vanilla $\ELBO$ struggles the most. This is consistent with the theoretical predictions of the tightness of the bounds.

\vspace{-8pt}
\subsection{Real-world data}
\vspace{-6pt}

To demonstrate HBO's ability to model real-world complex distributions, we consider analyzing the \texttt{MNIST} data in the main text, with details on implementation along with analyses for other popular benchmark datasets deferred to the SM Sec. D.

% a number of popular benchmark datasets. We focus on analyzing the \texttt{MNIST} data in the main text, with details on implementation along with analyses for other datasets deferred to the SM Sec. D.

\vspace{-2pt}
{\bf Approximation quality.} Two aspects of VI approximation are of particular interest: ($i$) inference: the quality of the approximate posterior $q_{\phi}(\zv|\xv)$; and ($ii$) model: the quality of the likelihood approximation $\log p_{\theta}(x)$. ($i$) is particularly important, as many applications leverage VI as a principled probabilistic approach for representation extraction, which seeks high-quality encoders $q_{\phi}(\zv|\xv)$. 

These two aspects can be respectively assessed through the $\ELBO$ and $\IWELBO$ bounds, with the latter using large importance samples to accurately approximate the marginal $\log$-likelihood. In Figure \ref{fig:mnist_training}, we plot the evolution trajectories of respective bounds on the test set for the models trained with different variational objectives, with the numbers summarized in Table \ref{tab:mnist}. Consistent with the analysis from \citet{rainforth2017tighter}, $\ELBO$ trains better encoder compared with $\IWELBO$ and R\'enyi, but does worse in likelihood learning. Our new thermodynamic $\HBO$ scheme is immune to this tension, and performs better than competing solutions in both metrics. 

\vspace{-2pt}
{\bf $\alpha$-tuning, ablation and additional analyses.} Choosing the right $\alpha$ is crucial for $\HBO$. We applied the simple heuristic search procedure described in Section \ref{sec:imp}, with $\alpha$ updated each epoch. It brings neglectable extra computations, while substantially improves performance (see ablation in the SM Sec. F). This distincts from R\'enyi and other parameterized variational objectives, where there is no good heuristics for parameter tuning. Limited by space, we provide extended ablations, analyses and results in the SM Sec. E-F, further assessing $\HBO$'s inference and modeling properties, along with more real-world applications for image generation, language modeling and Bayesian regressions.

\vspace{-8pt}
\section{Conclusions}
\vspace{-6pt}
This study provides a comprehensive discussion of the recently introduced thermodynamic variational schemes. In particular, we elucidate the role of thermodynamic variational objectives in the larger picture of the modern VI literature. Motivated from a geometric argument, we present $\HBO$, a novel generalization of thermodynamic variational bound that further improves $\TVO$. Via seeking an alternative integration path, $\HBO$ flattens the pathological curve of local evidence, yielding a tighter bound. We also cover important topics such as automated parameter tuning. Empirical evidence verified that $\HBO$ enjoys low variance and high effective-sample size, leading to improved performance.

% {\it Remark.} 
\clearpage

\bibliography{hbo}

\newpage

\appendix

\renewcommand{\thetable}{S\arabic{table}}
\renewcommand{\thefigure}{S\arabic{figure}}
\renewcommand{\thealgorithm}{S\arabic{algorithm}}
\renewcommand{\thethm}{S\arabic{thm}}

\setcounter{table}{0}
\setcounter{figure}{0}
\setcounter{algorithm}{0}
\setcounter{thm}{0}

%%%%% To create table of contents only for Appendix %%%%%%
\onecolumn
\addcontentsline{toc}{section}{Appendix} % Add the appendix text to the document TOC
\part{Appendix} % Start the appendix part
\parttoc % Insert the appendix TOC
\newpage
%%%%%%%%%%%%%%%%%%%%%%%%%%%%%%%%%%

\section{Extended Results and Discussions}
\subsection{Statement}
The author(s) would like to note that this work was originally prepared back in late 2019 titled {\it ``Flattening The Curve: Variational Inference with H\"older Bound''}. When it was first submitted for review in early 2020, the global pandemic hits. The reviewers had concerns the phrase ``flattening the curve'' may be considered implicitly referencing the on-going raging pandemic and would inadvertently jog sorrow memories. This was not the case since this study was complete work before COVID (8-page main text with 11-page Appendix, excluding refs). That said, the author(s) understood this concern and decided to change the title, withdraw submission and wait until COVID pandemic recedes to resubmit this work. The earliest appearance of other work with a similar idea appear only in late 2020, and only as a 4-pager workshop submission with only toy experiments. 

\subsection{When do TVO/HBO bounds work better?}

Despite their elegant formulation, the family of thermodynamic variational objectives, such as the original $\TVO$ and the proposed $\HBO$, has been critically challenged for its practical utility. As the original $\TVO$ paper has observed \citep{masrani2019the}, $\TVO$ offers similar performance relative to $\IWELBO$. We devote this section to clarify why people fail to benefit from the supposedly-exact $\TVO$-type bounds, and provide concrete examples where $\TVO$s learns efficiently while other more standard variational bounds struggle.

In the main text, we have provided arguments based on estimation variance and effective sample size to support $\HBO$ and $\TVO$. However, these two alone can not guarantee $\TVO$s work better than the other alternatives. Here we want to offer some additional insights. 

We argue the key lies in the gap between: ($a$) the complexity of posterior distribution $p_{\theta}(z|x)$; ($b$) the expressive power of the posterior approximator $q_{\phi}(z|x)$. 
Our observation is that, $\TVO$-type variational objectives are mostly handicapped by the expressive power of $q_{\phi}(z|x)$, which typically takes the convenient choice of mean-field Gaussian distributions in practice. When the ground-truth posterior distribution shows apparent deviations from Gaussianity, the Gaussian posterior approximations learned by alternative variational bounds are not characteristically different.  

To make this point clear, let us look at the Bayesian inference example described in Sec \ref{sec:bayes_inf}. In this example, the true posterior is approximately a unit circle (Figure \ref{fig:vi_approx}). With Gaussian approximation, the estimated parameter distribution only covers a corner of the circle. Of course one can conveniently blame this on a bad choice of $q_{\phi}(z|x)$, and thinks the true-posterior can be recovered by using a more flexible posterior, say {\it normalizing flows} \citep{rezende2015variational}. Unfortunately, that is not the case. In Figure \ref{fig:flow_elbo} and Figure \ref{fig:flow_iwelbo}, we show that even one adopts a highly-expressive non-parametric normalizing flow (see Sec \ref{sec:nf}), using $\ELBO$ and $\IWELBO$ objectives do not adequately approximate the true posterior. Encouragingly, we see good convergence to the ground-truth for $\TVO$-type objectives (Figure \ref{fig:flow_tvo} and Figure \ref{fig:flow_hbo}). Between the two, $\HBO$ converges a little bit faster. In Sec \ref{sec:more_exp} we will further confirm these observations quantitatively using other distributional metric. Based on experiment evidence, this better convergence might result from the mode-covering behavior of $\HBO$ (see below). 

\begin{figure}[t!]
	\centering
	\includegraphics[width=.85\textwidth]{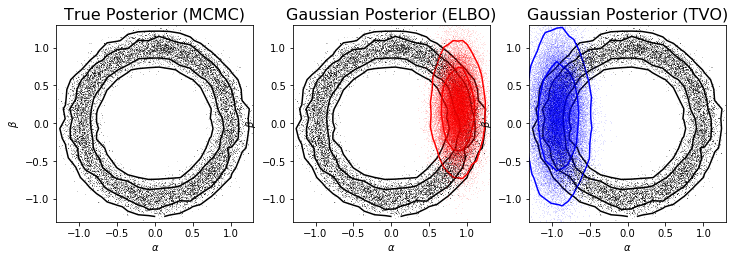}
	\vspace{-.5em}
	\caption{True posterior (left) and Gaussian variational approximations (through $\ELBO$ and $\TVO$).}
	\label{fig:vi_approx}
\end{figure}

\begin{figure}[t!]
	\centering
	\includegraphics[width=.7\textwidth]{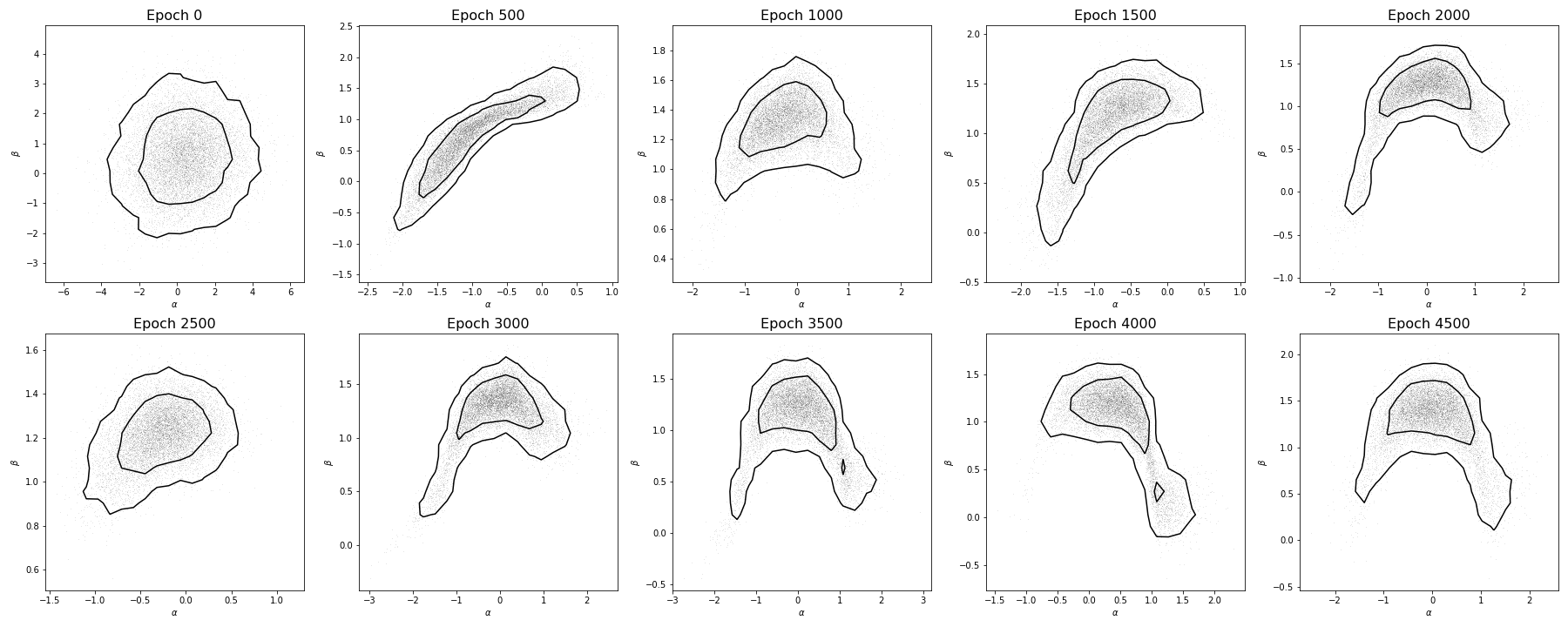}
	\vspace{-.5em}
	\caption{Flow-posterior trained with $\ELBO$,  {\color{red}\bf unable to converge}.}
	\label{fig:flow_elbo}
	\vspace{.5em}
	\includegraphics[width=.7\textwidth]{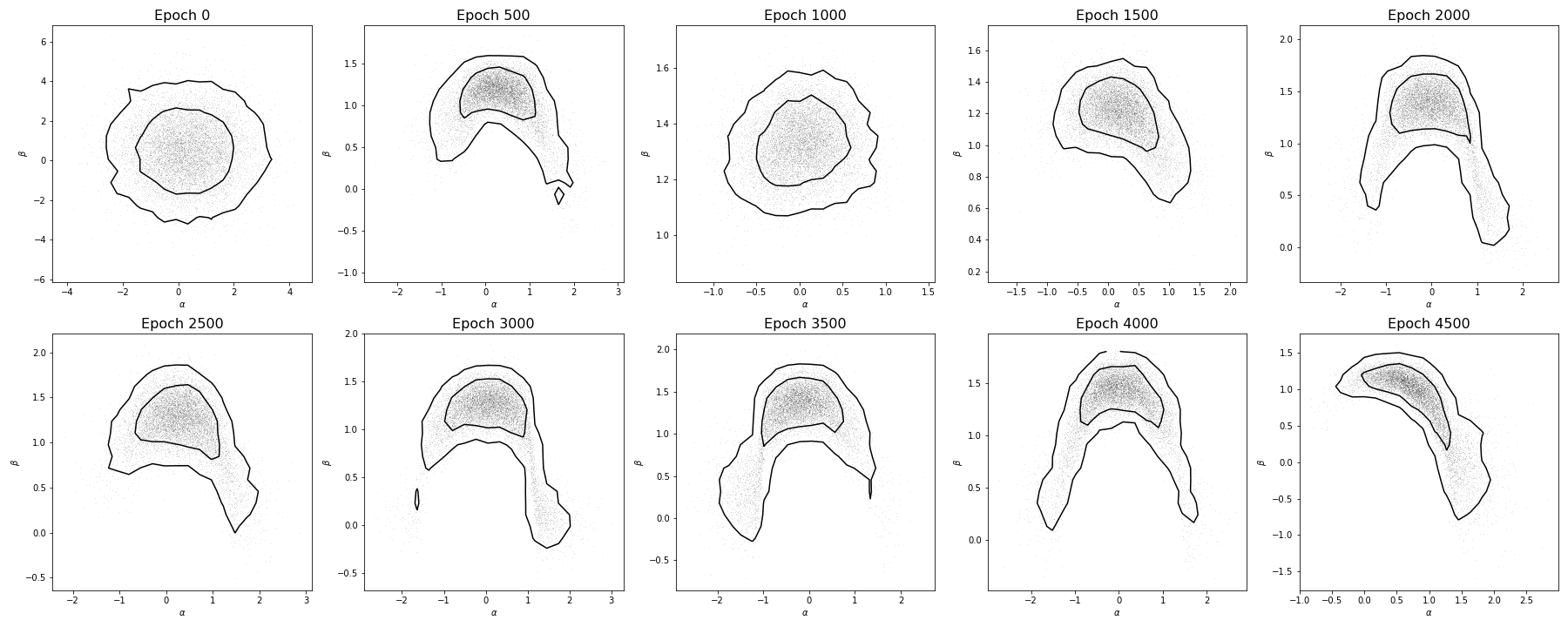}
	\vspace{-.5em}
	\caption{Flow-posterior trained with $\IWELBO(S=5)$, {\color{red}\bf unable to converge}.}
	\label{fig:flow_iwelbo}
	\vspace{.5em}
	\includegraphics[width=.7\textwidth]{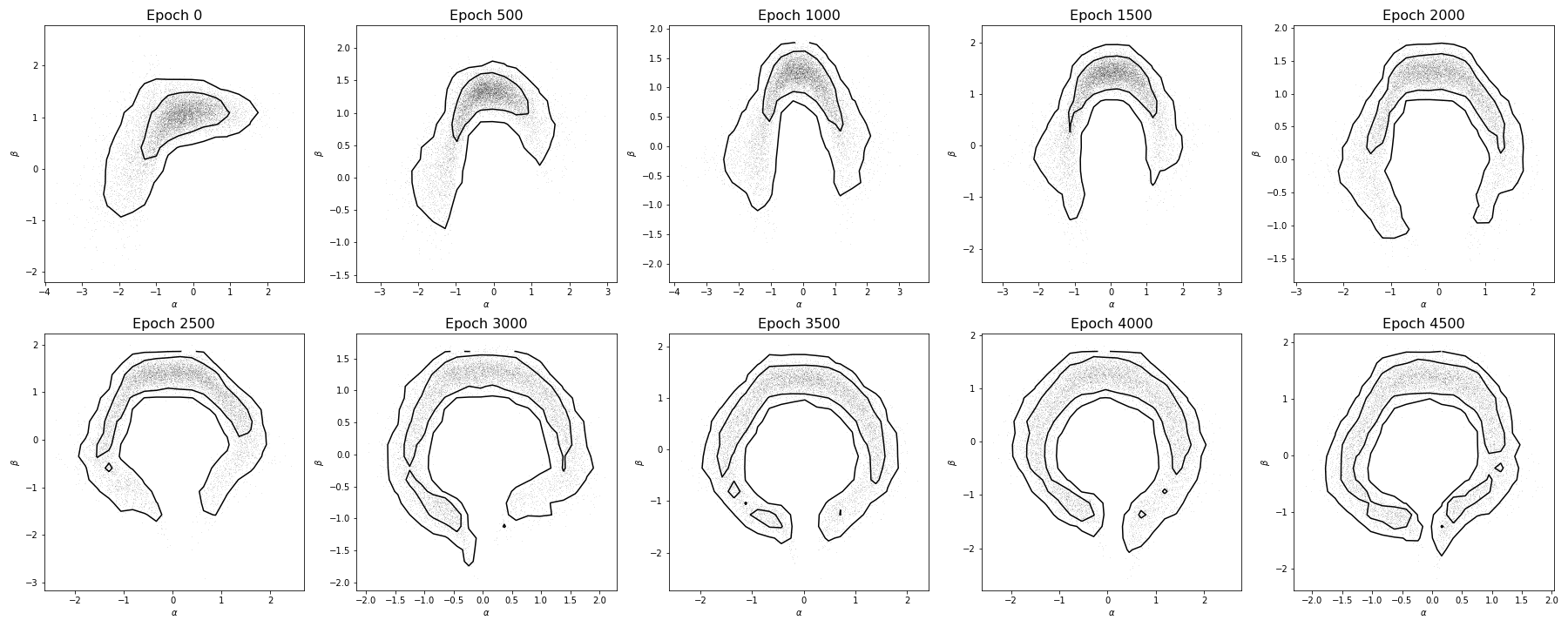}
	\vspace{-.5em}
	\caption{Flow-posterior trained with $\TVO(S=5,K=5)$, {\color{olive}\bf converge}.}
	\label{fig:flow_tvo}
	\vspace{.5em}
	\includegraphics[width=.7\textwidth]{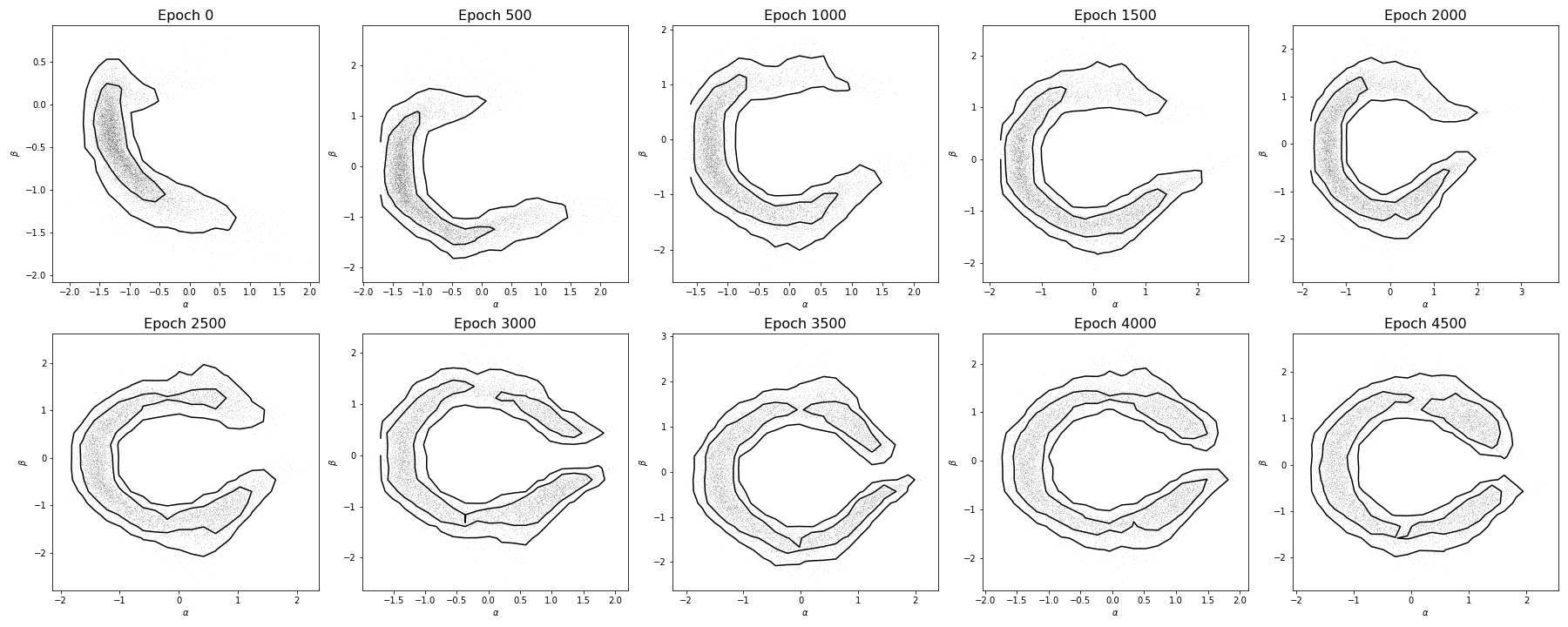}
	\vspace{-.5em}
	\caption{Flow-posterior trained with $\HBO(\alpha=10^{-3},S=5,K=5)$, {\color{olive}\bf converge}.}
	\label{fig:flow_hbo}
	% \vspace{-2em}
\end{figure}

% why people fail to benefit frose sharper bounds

\subsection{Mode covering behavior}

For Bayesian inference, one might be interested in the mode covering behavior of an estimator: whether the estimated posterior is condensed in the high-density region of the true posterior ({\it i.e.}, mode-seeking), or it covers up the support of the true posterior ({\it i.e.}, mode-covering). Different applications might desire different mode covering behaviors \citep{li2016renyi, dieng2017chi}. Our experiments seem to suggest overall $\HBO$ favors covering the true posterior distribution (see for instance, Figure \ref{fig:mode-covering}). 

\clearpage

\subsection{More related work}

\paragraph{Likelihood estimation} has been a long-standing challenge for both statistical and machine learning communities.  Classical solutions, such as kernel density estimation \cite{parzen1962estimation} and mixture models \cite{lindsay1995mixture}, does not scale-well with the dimensionality and complexity of modern datasets, such as images and natural languages. For machine learning applications, usually there are a few additional expectations for a ``good'' likelihood estimation procedure: it is {\it generative}, in the sense that samples can be easily drawn from the learned model \cite{mohamed2016learning, nowozin2016f}; it is {\it recapitulative}, such that it summarizes the important features of the data that can be repurposed \cite{dumoulin2016adversarially}. 
%These two desiderata respectively constructs two important research fields known as generative modeling \cite{} and unsupervised learning \cite{}.

Recent advances focused on both modeling flexibility and computational scalability of likelihood-based learning procedures, and can be broadly categorized into {\it exact-likelihood} models and {\it approximate inference} models. The exact-likelihood models progressively build up complex distributions via stacking simple transformations with easy-to-compute Jacobians \cite{dinh2016density}, and thus commonly being referenced to as {\it flow}-based models \cite{rezende2015variational, kingma2016improving, kingma2018glow}. While proven quite powerful, they are excessively resource demanding and often requires dedicated engineering efforts for large-scale problems \cite{oord2017parallel}, and consequently limiting their direct application in practical terms. Note that flow-based procedures have been successfully applied to improve the expressiveness of the approximate posterior in VI, and marry it to our $\HBO$ may yield additional gains.

\paragraph{Approximate inference} models, like their name suggest, seek to efficiently optimize cheap approximations to the likelihood, which includes procedures such as {\it expectation maximization} (EM) \cite{dempster1977maximum}, {\it expectation propagation} (EP) \cite{minka2013expectation} and {\it variational inference} (VI) (see \cite{hernandez2016black} for a unifying perspective). In recent years, VI, more commonly known under the name {\it variational auto-encoder} (VAE) in the machine learning context \cite{kingma2013auto}, gained its popularity due to its easy implementation, intuitive interpretation, strong performance and versatility in its presentations. As discussed in the main text, ($i$) tightening the variational gap and $(ii)$ reducing variance construct two main research topics in VI. As we have shown, via properly tuning the geometry parameter $\alpha$, the proposed $\HBO$ improves both aspects: it allows sharp approximation with low computational budget \footnote{The savings would be tremendous if one need to sample exactly from the intermediate distribution rather than appealing to importance reweighting.}, and it empirically shows low variance and high sample efficiency. 

%The proposed $\HBO$ is motivated by . 

Note that variants of VI which does not necessarily conform to the likelihood perspective, such as $\beta$-VAE and adversarial auto-encoder (AAE), have been proven most useful in practice (assessed using domain-specific metrics in addition to the likelihood). These attempts can be sometimes interpreted as empirical adjustments to correct for the modeling bias due to misspecification. How, and will the integration of $\HBO$ perspectives lend additional benefits provide an interesting topic for future investigation.  

Close to $\HBO$ is the work of GLBO \cite{tao2018variational}, which shared a number of motivations. GLBO also sets off its discussion from a geometrical argument. Unlike $\HBO$, GLBO decomposes variational objectives into a composition of convex map and its inverse, interlaced with the expectation of the approximate posterior. A keen observation of GLBO is that, via altering the geometry of the convex transform, making it ``flatter'', results in tighter bounds, at the price of a higher variance. Compared to GLBO, our $\HBO$ considers flattening in the thermodynamic domain, and does not pay the price for the exacerbated variance. The GLBO paper also discussed the perturbed estimator and shared insights on model selection for VI, which inspires similar efforts in our development.

\paragraph{Alternative distribution discrepancy metrics} with appealing characteristics compared to the KL-divergence minimization criteria implicitly assumed by the likelihood-based learning, have been extensively studied in machine learning. Such efforts include direct generalizations to the {\it information theoretic divergence}, also known as the $f$-divergence. An interesting intersection with the current study is the R\'enyi VI, where the variational gap being optimized is the $\alpha$-divergence. \cite{li2016renyi} discussed different profiles of the approximate posterior learned with R\'enyi VI objectives,  and our analysis shows its connection to $\TVO$. Note that a major criticism against the information theoretic measures is that they are not metric-aware \cite{arjovsky2017towards, ozair2019wasserstein}. In other words, small perturbations to the distribution can result in drastic changes in the $f$-divergence. Such undefined behavior can be overcome by the {\it integral probability metrics} (IPM) \cite{muller1997integral, sriperumbudur2009integral}, including prominent examples such as {\it maximal mean discrepancy} (MMD) \cite{gretton2012kernel} and {\it Wasserstein distance} \cite{arjovsky2017wasserstein}. IPM is defined as the maximal contrast between two distributions wrt a fixed function space. And via a Kantorovich-Rubinstein duality argument, it can be shown that with an appropriate choice of function space, IPM can be expressed in terms of the distance given by the ground metric. This motivates the development of both primal and dual solver for IPM distance.

\paragraph{Adversarial schemes} have recently demonstrated their effectiveness in handling complex, intractable distributions \cite{goodfellow2014generative}. Their applications have been previously considered for the improvement of VI, such as {\it adversarial variational Bayes} (AVB) \cite{mescheder2017adversarial}, {\it coupled variational inference} (CVB) \cite{dai2018coupled} and Fenchel mini-max VAE \cite{tao2019fenchel}. While a direct generalization of such schemes to $\HBO$ is straightforward, we caution extra care must be taken, as the thermodynamic integration may be sensitive to the approximation error of the likelihood contrast. 

\paragraph{The utility and tightness trade-off} for the variational posterior is first formally investigated in \cite{rainforth2017tighter}. Intuitively, it says with a sharper variational bound to the marginal likelihood, the optimization procedure is less incentivized to improve the variational posterior $q_{\phi}(\zv|\xv)$. Our analysis suggests that thermodynamic objectives, such as $\HBO$, might be immune to this dilemma. And a plausible explanation is offered: in $\HBO$, posterior samples contribute more ``equally''. 

%\paragraph{Sampling}

\paragraph{Automated tuning} is an important yet less studied topic in the VI context. It promises to alleviate the burden of model fitting for empirical investigators and avoid introducing subjective bias. While the specific tuning target is tailored for $\HBO$, some designing principles and experience can be shared. For one, the proposed $\HBO$ tuning explicitly seeks to minimize the variance along the thermodynamic curve. Also, our analysis the effective sample-size indicate alternative target for future improvements. Close to our developments is the adaptive $\TVO$ partition scheme, which replaces Riemannian integration with Lebesgue integration. Also in \cite{tao2018variational,lu2020reconsidering} the authors explored maximal entropy principle for model selection in VI.

\paragraph{More open questions} are in order following our $\HBO$ developments, which is beyond the scope of current study. A key question is under what conditions, if possible, can we ensure the monotonicity of the thermodynamic curve? This is significant as the monotonicity flipping point directly implies one-step approximation of the exact likelihood. Although in the absence of such analysis, one still hopes to considerably improve the performance of VI using the $\HBO$ perspectives. Practically, while our study characterize the qualitative behavior of $\HBO$, there is much room for quantitative improvements. Like all predecessors of thermodynamic inference procedures, our results do not see a drastic boost relative to the existing counterparts, possibly limited by the restrictive modeling choices we have adopted ({\it e.g.}, MLP architecture and Gaussian posterior). Potential improvements are expected using more expressive model architectures. Also of particular interest is whether there is a better alternative proposal sampler between the extremes of an MCMC sampler and the approximate posterior sampler.

\paragraph{MCMC Variational Inference \citep{salimans2015markov}} belongs to a more general family of VI schemes known as the {\it Auxiliary VI} (Aux-VI) \citep{maaloe2016auxiliary}, which seeks to improve bound sharpness via introducing auxiliary latent variables $\tilde{z}$. Standard $\ELBO$ employs $q(z|x)$ with simple forms that enjoy analytical expressions ({\it e.g.}, mean field Gaussian), which limits the expressive power of the posterior and subsequently compromises sharpness. To overcome this limitation, Aux-VI assumes an augmented latent variable model $p_{\theta}^{\dagger}(x,z,\tilde{z}) = p_{\theta}(x,z) r_{\theta}(\tilde{z}|x,z)$ and variational posterior $q_{\phi}^{\dagger}(z,\tilde{z}|x) = q_{\phi}^{\dagger}(z|x,\tilde{z})q_{\phi}^{\dagger}(\tilde{z}|x)$, 
denoting 
\beq
\begin{array}{c}
	\AVI \triangleq \EE_{Z,\tilde{Z} \sim q_{\phi}^{\dagger}}\left[\log \frac{p_{\theta}^{\dagger}(x,Z,\tilde{Z})}{q_{\phi}^{\dagger}(Z,\tilde{Z}|x)}\right], \\
	[10pt]
	\ELBO^{\dagger} \triangleq \EE_{Z\sim q_{\phi}^{\dagger}}\left[ \log \frac{p_{\theta}^{\dagger}(x,Z)}{q_{\phi}^{\dagger}(Z|x)} \right], 
\end{array}
\eeq
where we have used $\ELBO^{\dagger}$ to differentiate the auxiliary latent-augmented ELBO from the vanilla ELBO. One may readily verify that 
\beq
\AVI = \ELBO^{\dagger} - \bar{\BD}_{\KL}  \leq \ELBO^{\dagger} \leq \log p_{\theta}(x),
\eeq
where $\bar{\BD}_{\KL} \triangleq \EE_{Z\sim q_{\phi}^{\dagger}(z|x)}\left[ \KL(q_{\phi}^{\dagger}(\tilde{z}|x,Z)\parallel p_{\theta}(\tilde{z}|x,Z)) \right]$. 
Note $\tilde{z}$ is typically defined by a sequence of simple transition kernels $q(\tilde{z}_{t+1}|\tilde{z}_t, x)$ for tractable computations.
Here $\AVI$ is to be compared with standard $\ELBO$ defined by $q(z|x)$. 
The motivation behind Aux-VI is that, since the marginalized latent distribution $q_{\phi}^{\dagger}(z|x) \triangleq \EE_{\tilde{Z}\sim q(\tilde{z}|x)}[q(z,\tilde{Z}|x)]$ is very expressive, it is anticipated to bring more increases in the $ \ELBO^{\dagger}$ term compared to the reductions from the $\bar{\BD}_{\KL}$ term due to the auxiliary variables $\tilde{z}$, which means $\ELBO \leq \AVI$ is very likely to hold.  
%For computational tractability, $q(z|x)$ often admits simple forms that enjoy analytical expressions, which limits the expressive power of the posterior and subsequently compromises bound sharpness. 
%The motivation behind Aux-VI is that, relative to a directly specified $q_{\phi}(z|x)$ with closed-form likelihood expression, the marginalized latent distribution $q_{\phi}^{\dagger}(z|x) \triangleq \EE_{\tilde{Z}\sim q(\tilde{z}|x)}[q(z,\tilde{Z}|x)]$ is very expressive, which is anticipated to bring more increases in the $ \ELBO^{\dagger}$ term compared to the reductions from the $\bar{\BD}_{\KL}$ term due to the auxiliary variables $\tilde{z}$. 

A concrete implementation of Aux-VI is the MCMC-VI formulated in \citet{salimans2015markov}, where the $(z, \tilde{z})$ is specified in a hierarchical fashion using Markov chains, {\it i.e.}, 
\beq
%\tilde{z} = [\tilde{z}_0, \cdots, \tilde{z}_T], \tilde{z}_t \sim q_t(\tilde{z}_t|x,\tilde{z}_{t-1}), \tilde{z}_t \sim \tilde{r}_t(\tilde{z}_t|x, \tilde{z}_{t+1})
\begin{array}{l}
	\tilde{z} = [\tilde{z}_0, \cdots, \tilde{z}_T], \,\,\, z \triangleq \tilde{z}_T, \\
	[5pt]
	q_{\phi}^{\dagger}(\tilde{z}|x): \tilde{z}_t \sim q_t(\tilde{z}_t|x,\tilde{z}_{t-1}), \\
	[5pt]
	r_{\theta}(\tilde{z}|x,z): \tilde{z}_t \sim \tilde{r}_t(\tilde{z}_t|x, \tilde{z}_{t+1}), 
\end{array}
\eeq 
and here $\{q_t, \tilde{r}_t\}_{t}$ denote the transition kernels. 
%where $\{ q_t(\tilde{z}_t|x,\tilde{z}_{t-1}), \tilde{r}_t(\tilde{z}_t|x,\tilde{z}_{t+1}) \}_t$ denote the respective transition operators defining $(q_{\phi}^{\dagger}(\tilde{z}|x), r_{\theta}(\tilde{z}|x,z))$, and we identify $z \triangleq \tilde{z}_T$. 
Hereafter we omit the dependence of approximate posterior $q$ on data $x$ for clarity. When each $q_t(\tilde{z}_t|\tilde{z}_{t-1})$ observe the detailed balance for $\pi_{\beta_t}$, {\it i.e.}, the intermediates on the geometric path, and if the reverse model $r(\tilde{z}_{t-1}|\tilde{z}_t) = q_{t}(\tilde{z}_t|\tilde{z}_{t-1})\pi_{t}(\tilde{z}_{t-1})/\pi_{t}(\tilde{z}_t)$, one arrives at
\beq
%\begin{array}{c}
\MCMCVI = \underbrace{\EE_{\tilde{Z}_{1:T}\sim q(\tilde{z}|x)}\left[\sum_t \Delta \beta_t \log \frac{p_{\theta}(x,\tilde{Z}_t)}{q_0(\tilde{Z}_t)}\right] }_{\TVO} , % \leq \log p_{\theta}(x)
%\end{array}
\label{eq:mcmc}
\eeq
where $\Delta \beta_t \triangleq (\beta_t - \beta_{t-1})$. One can readily recognize that the $\MCMCVI$ defined in (\ref{eq:mcmc}) exactly recovers the $\TVO$ estimator. Note that the original MCMC-VI paper was unable to establish the tightness of its bound as $\max \{ \Delta \beta_t \} \rightarrow 0$, a result that straightforward under the $\TVO$ framework.

\section{Figures Cited in the Main Text}

Figure \ref{fig:holder_path} visualizes representative examples of H\"older paths, including special examples of geometric mean path ($\alpha=0$) and Wasserstein mean path ($\alpha=1$). 

Figure \ref{fig:hbo_tvo_k} compares $\TVO$ bounds and $\HBO$ bounds with different partition budget with uniform partitioning. We set $\alpha=0.2$ in this case. 

\begin{figure}[h]
	\begin{center}
		\includegraphics[width=.3\textwidth]{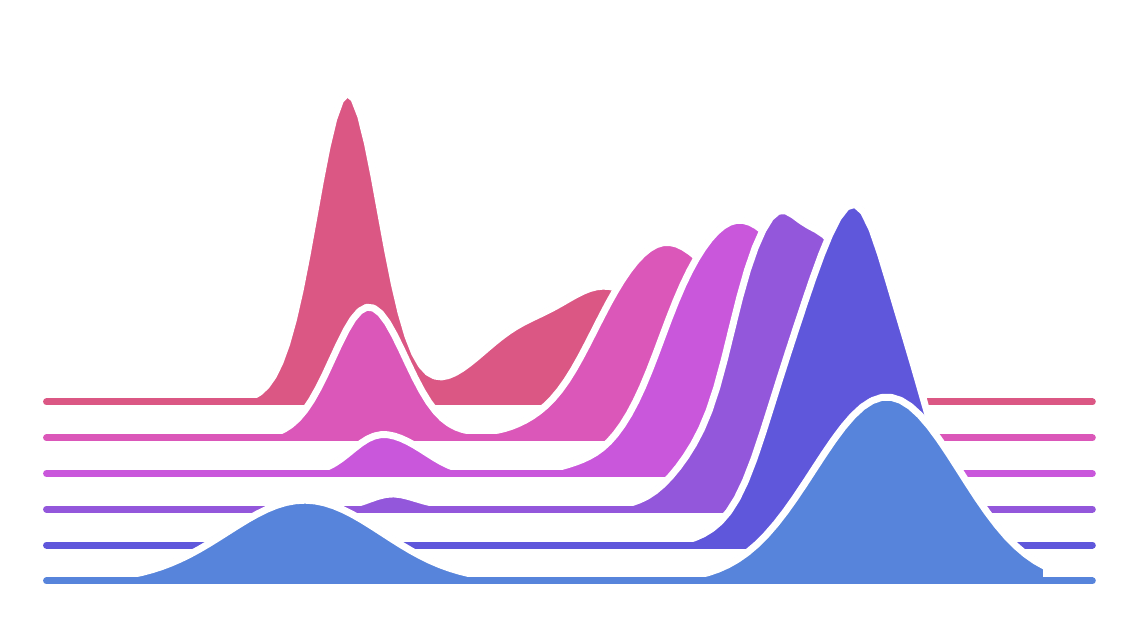}
		\includegraphics[width=.3\textwidth]{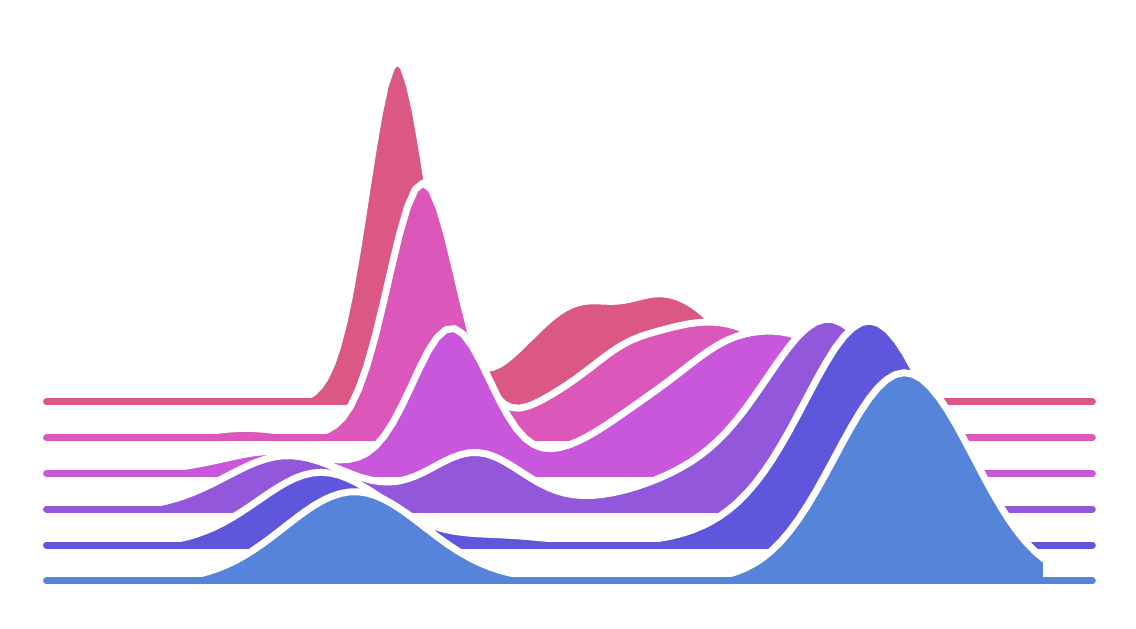}
		\includegraphics[width=.3\textwidth]{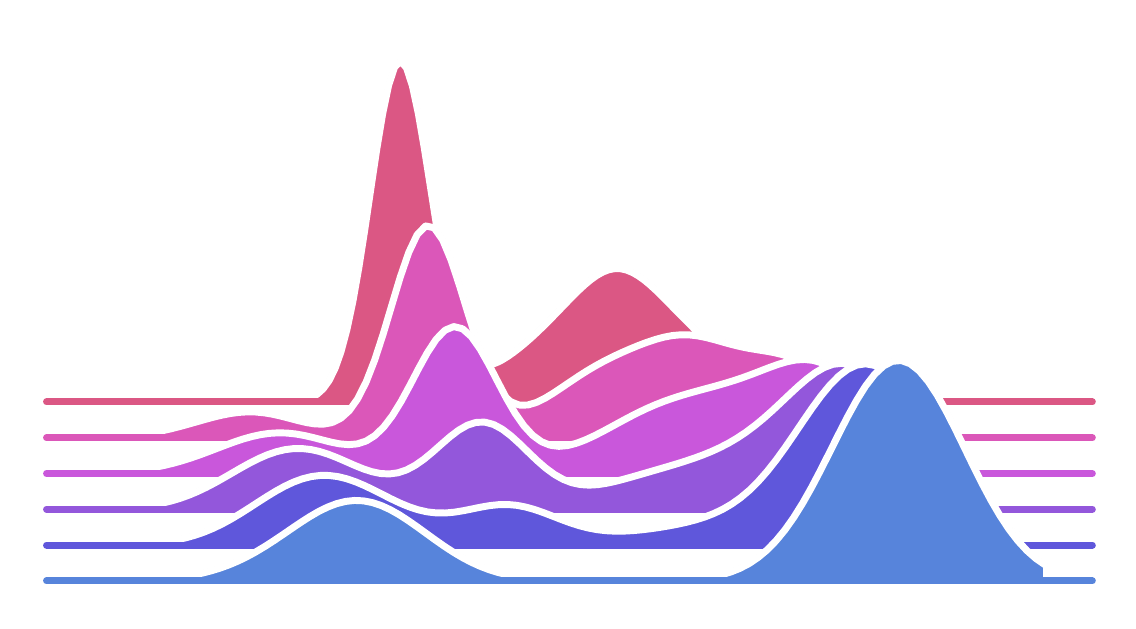}
	\end{center}
	\vspace{-1em}
	\caption{Comparison of different H\"older paths. From left to right: $\alpha=0, 0.5, 1$, respectively for geometric mean path, H\"older mean path and Wasserstein mean path. \label{fig:holder_path}}
	\vspace{-1.em}
\end{figure}

\begin{figure}[h]
	\begin{center}
		\includegraphics[width=.48\textwidth]{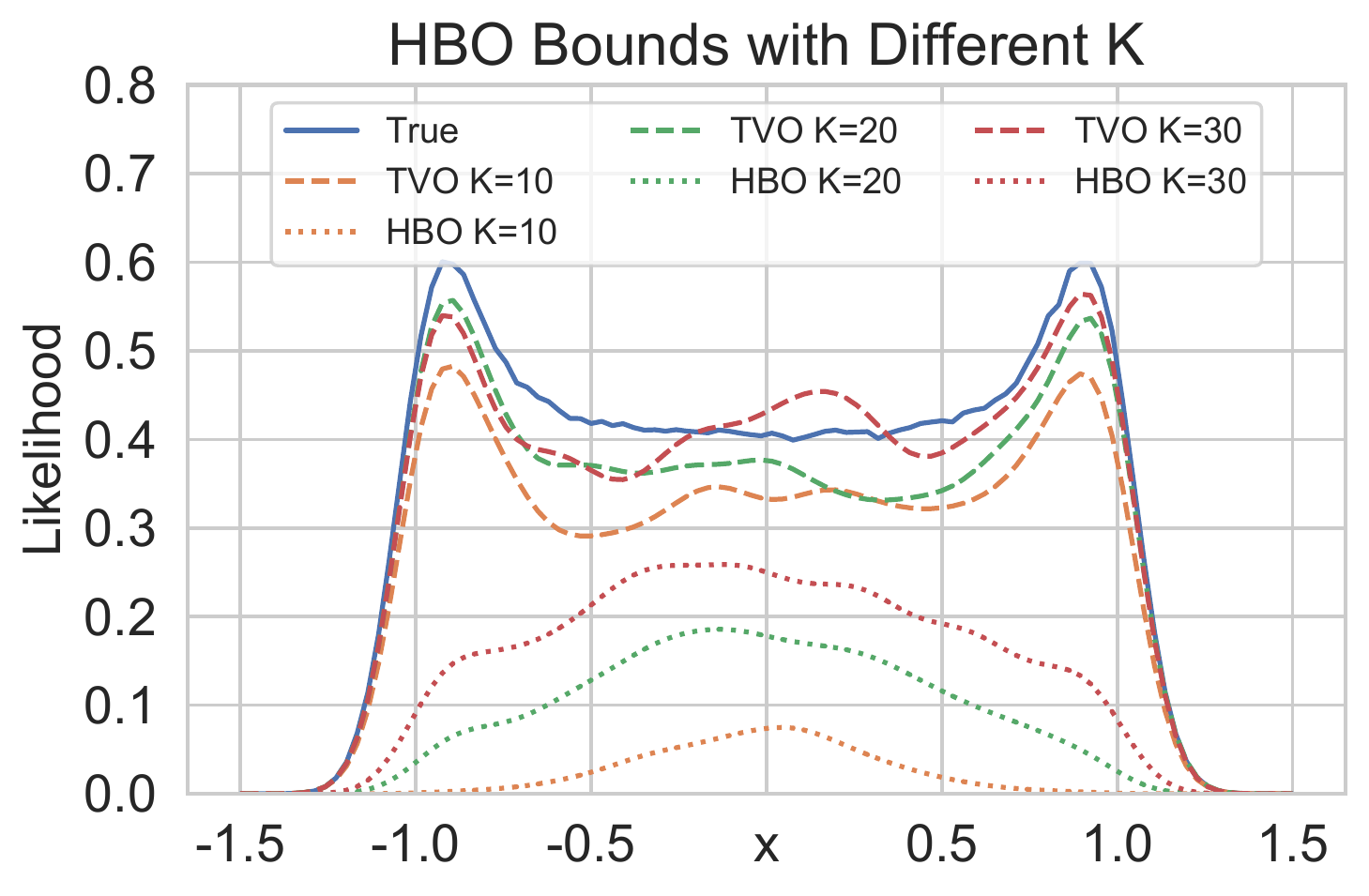}
	\end{center}
	\vspace{-1em}
	\caption{Comparison of $\TVO$ and $\HBO$ bounds with different partition budget $K$.  \label{fig:hbo_tvo_k}}
	\vspace{-1.em}
\end{figure}

\section{Technical Proofs and Derivations}

In this section, we give details of the technical proofs and derivations for the results reported in the main text. 

%\subsection{}

%\section{Variational R\'enyi bound}

\subsection{Proof of Proposition 2.1 (Equivalence to the R\'enyi Bound)}

%\begin{figure}[h]
%%\includegraphics[width=\textwidth]{fig12.eps}
%\end{figure}

\begin{defn}
	Variational R\'enyi bound
	\begin{equation}
		\mathcal{L}_{\alpha}(\tilde{\pi}_1, \tilde{\pi}_0)= \frac{1}{\alpha} \log \mathbb{E}_{\pi_0}\left[\left(\frac{\tilde{\pi}_1}{\tilde{\pi}_0}\right)^{\alpha}\right]
	\end{equation}
\end{defn}

\paragraph{Proposition 2.1} $\mathcal{L}_{\alpha}(\tilde{\pi}_1, \tilde{\pi}_0) =\frac{1}{\alpha}\int_0^\alpha \mathbb{E}_{\pi_{0,\beta}}[U^\prime_{0, \beta}] \ud \beta$. 

%\begin{theorem}
%\begin{equation}
%	\mathcal{L}_{\alpha}(\tilde{\pi}_1, \tilde{\pi}_0) =\frac{1}{\alpha}\int_0^\alpha \mathbb{E}_{\pi_{0,\beta}}[U^\prime_{0, \beta}]d\beta
%\end{equation}
%\end{theorem}

\begin{proof}
	\begin{equation}
		\mathcal{L}_{\alpha}(\tilde{\pi}_1, \tilde{\pi}_0) = \frac{1}{\alpha}\log\int \tilde{\pi}_1^{\alpha}(z)\tilde{\pi}_0^{1-\alpha}(z) \ud z
		=\frac{1}{\alpha}\log Z_{0,\alpha} = \frac{1}{\alpha}\int_0^\alpha \mathbb{E}_{\pi_{0,\beta}}[U^\prime_{0, \beta}] \ud \beta
	\end{equation}
\end{proof}

\subsection{Proof of Proposition 3.1 (Wasserstein Thermodynamic Bounds)}

%We form a family between $\pi_0(z)$ and $\pi_1(z)$ via a scalar parameter $\beta\in [0,1]$.
Recall Wasserstein thermodynamic path is given by the following family of intermediate distributions:

\begin{equation}
	\tilde{\pi}^W_{\beta} = \beta \tilde{\pi}_1+(1-\beta)\tilde{\pi}_0,~~ 
	Z_{\beta}^W = \int_\mathcal{Z} \tilde{\pi}^W_{\beta}(z) \ud z,~~
	\pi^W_\beta = \frac{\tilde{\pi}^W_{\beta}}{Z_{\beta}^W}, ~~\beta\in[0,1]
\end{equation}

%Assuming the legitimacy of exchanging integration with differentiation,
For sufficiently smooth $\tilde{\pi_\beta}^W$ ({\it e.g.}, $\tilde{\pi}^W$ and $\log \tilde{\pi}^W$ are continuous wrt $z$ and $\beta$ both), we can exchange the integration with differentiation
\begin{equation}\label{f1}
	\begin{split}
		\frac{\partial \log Z_{\beta}^W}{\partial \beta} &= \frac{1}{Z_{\beta}^W}\frac{\partial}{\partial \beta}{Z_{\beta}^W}=\frac{1}{Z_{\beta}^W} \frac{\partial}{\partial \beta}\int \tilde{\pi}^W_{\beta} \ud z\\
		&=\int \frac{\tilde{\pi}^W_{\beta}}{Z_{\beta}^W}\frac{\partial}{\partial \beta}\log \tilde{\pi}^W_{\beta} \ud z = \mathbb{E}_{\pi_{\beta}^W}\left[\frac{\partial}{\partial \beta}\log \tilde{\pi}^W_{ \beta}\right]
	\end{split}
\end{equation}
Denote $U_{\beta}^W \triangleq \log \tilde{\pi}^W_{\beta}$, 
we integrate out $\beta$ on both sides of (\ref{f1}),
\begin{equation}
	\int_{0}^{1}\frac{\partial \log Z_{\beta}^W}{\partial \beta} \ud \beta = \int_{0}^{1}\mathbb{E}_{Z\sim \pi_{\beta}^W}\left[\partial_{\beta} {U_{\beta}^W}(Z)\right]\ud \beta,
\end{equation}
which gives
\begin{equation}
	\log Z_1 - \log Z_0 = \int_{0}^{1}\mathbb{E}_{Z\sim \pi_{\beta}^W}\left[U^\prime_{ \beta}(Z)\right] \ud \beta.
\end{equation}
Denote the local evidence as 
\begin{equation}
	\begin{split}
		h(\beta) \triangleq \mathbb{E}_{\pi_\beta^W}\left[\frac{\tilde{\pi}^W_1 - \tilde{\pi}^W_0}{\tilde{\pi}^W_\beta}\right]
	\end{split}
\end{equation} 
We want to show that $h(\beta)$ is non-increasing.
\begin{proof}
	For notational clarity, here we suppress the superscript $W$ on the partition function $Z_{\beta}$. 
	\begin{equation}
		\begin{split}
			\frac{\partial}{\partial \beta}h(\beta) &= \frac{\partial}{\partial \beta}\left[\int \pi_\beta^W(z)\frac{\tilde{\pi}^W_1(z) - \tilde{\pi}^W_0(z)}{\tilde{\pi}^W_\beta(z)}\ud z\right]\\
			&= \frac{\partial Z_\beta^{-1}} {\partial \beta}\int\left[\tilde{\pi}_1(z)-\tilde{\pi}_0(z)\right] \ud z\\
			&=-\frac{1}{Z_\beta^2}\frac{\partial Z_\beta}{\partial \beta} \int\left[\tilde{\pi}_1(z)-\tilde{\pi}_0(z)\right] \ud z\\
			&= -\frac{1}{Z_\beta^2}\left(\int\left[\tilde{\pi}_1(z)-\tilde{\pi}_0(z)\right] \ud z\right)^2\leq 0
		\end{split}
	\end{equation}
	This proves Wasserstein thermodynamic local evidence curve is non-increasing. 
	Since
	\begin{equation}
		\begin{split}
			&\tilde{\pi}_0(z) = q_{\phi}(z|x), \quad Z_0 = \int q_\pi(z|x) \ud z = 1,\\
			&\tilde{\pi}_1(z) = p_\theta(x,z), \quad Z_1 = \int p_\theta(x,z)\ud z = p_\theta(x).
		\end{split}
	\end{equation}
	Since $h(\beta)$ is non-increasing,
	\begin{equation}
		\begin{split}
			h(0) &= \mathbb{E}_{\pi_0} \left[\frac{\tilde{\pi}_1 - \tilde{\pi}_0}{\tilde{\pi}^W_0}\right]\\
			&=\int \left[\tilde{\pi}_1(z) - \tilde{\pi}_0(z) \right] \ud z\\
			&\leq \int \tilde{\pi}_1(z)\log\frac{\tilde{\pi}_1(z)}{\tilde{\pi}_0(z)}\ud z = Z_1\cdot \mbox{EUBO}.
		\end{split}
	\end{equation}
	where the last inequality follows from Lemma (\ref{lem}) below.
	\begin{equation}
		\begin{split}
			h(1) &= \mathbb{E}_{\pi_1}\left[\frac{\tilde{\pi}_1 - \tilde{\pi}_0}{\tilde{\pi}_1}\right]\\
			&= \int\frac{\pi_1}{\tilde{\pi}_1}\left[\tilde{\pi}_1 - \tilde{\pi}_0\right] \ud z\\
			&=\frac{\int\left[\tilde{\pi}_1 - \tilde{\pi}_0 \right] \ud z}{Z_1}
			\geq \frac{1}{Z_1}\mbox{ELBO}
		\end{split}	
	\end{equation}
	where we have applied Lemma (\ref{lem}) again for the inequality. 
\end{proof}
\begin{lem}\label{lem} When  $f, g>0$, the following inequalities hold.
	\begin{equation}
		f-g\leq f\log \frac{f}{g},~~ f-g\geq g\log \frac{f}{g}	. 
	\end{equation}
\end{lem}
\begin{proof}
	Denote $x = \frac{f}{g}>0.$ 
	\begin{equation}
		\begin{split}
			\frac{f-g}{f} &= 1-\frac{1}{x}\leq \log x =\log \frac{f}{g}.\\
			\frac{f-g}{g} &= x-1 \geq \log x = \log \frac{f}{g}.
		\end{split}
	\end{equation}
\end{proof}

\subsection{Proof of Proposition 3.3 (Monotonicity of $\HBO$)}

\begin{proof}
	Recall $\log p_\theta(x) = \int_0^1 g_{\alpha}(\beta) \ud \beta$, where $g_{\alpha}(\beta) = \mathbb{E}_{\pi_{\alpha,\beta}}\left[U^\prime_{\alpha, \beta}(z)\right]$. 
	For $\alpha\neq 0$,
	
	\beqs
	& U_{\alpha, \beta} = \log \tilde{\pi}_{\alpha, \beta}, \quad 
	U^\prime_{\alpha,\beta} = \frac{\partial U_{\alpha, \beta}}{\partial \beta}= \frac{1}{\tilde{\pi}_{\alpha, \beta}}\frac{\partial \tilde{\pi}_{\alpha, \beta}}{\partial \beta} = \frac{1}{\alpha}\frac{\tilde{\pi}_1^\alpha - \tilde{\pi}_0^\alpha}{\tilde{\pi}_{\alpha, \beta}^\alpha}, \\
	& U^{\prime\prime}_{\alpha,\beta} = \frac{\partial^2 U_{\alpha,\beta}}{\partial \beta^2}
	= -\frac{1}{\alpha}\frac{\left(\tilde{\pi}_1^\alpha - \tilde{\pi}_0^\alpha\right)^2}{\tilde{\pi}_{\alpha,\beta}^{2\alpha}} = -\alpha \left(U^\prime_{\alpha,\beta}\right)^2. 
	%&\log p_\theta(x) = \int_0^1 g_{\alpha}(\beta)d\beta\\
	%&g_{\alpha}(\beta) = \mathbb{E}_{\pi_{\alpha,\beta}}\left[U^\prime_{\alpha, \beta}(z)\right]\\
	%\end{split}
	%\end{equation}
	\eeqs
	Via direct computation, we have 
	\beqs
	%	\begin{equation}
	%	\begin{split}
	\frac{\partial}{\partial \beta}g_{\alpha}(\beta) 
	& = & \frac{\partial}{\partial \beta}\mathbb{E}_{\pi_{\alpha,\beta}}\left[U^\prime_{\alpha, \beta}(z)\right]
	=\int \frac{\partial}{\partial \beta}\left[\pi_{\alpha,\beta}(z)U^\prime_{\alpha, \beta}(z)\right] \ud z \\
	&=&\int\frac{\partial}{\partial \beta}\left[\frac{\tilde{\pi}_{\alpha,\beta}(z)}{Z_{\alpha,\beta}}U^\prime_{\alpha, \beta}(z)\right] \ud z\\
	&=& \int \frac{1}{Z_{\alpha,\beta}}\frac{\partial\tilde{\pi}_{\alpha, \beta}(z)}{\partial \beta}U^\prime_{\alpha, \beta}(z)dz - \int\frac{\tilde{\pi}_{\alpha,\beta}(z)}{Z_{\alpha, \beta}^2}\frac{\partial Z_{\alpha, \beta}}{\partial\beta}U^\prime_{\alpha, \beta}(z) \ud z \\
	& & \hspace{4em} +\int\frac{\tilde{\pi}_{\alpha,\beta}(z)}{Z_{\alpha,\beta}}U^{\prime\prime}_{\alpha,\beta}(z) \ud z\\
	&=& \int \frac{\tilde{\pi}_{\alpha, \beta}(z)}{Z_{\alpha, \beta}}\frac{\partial \log \tilde{\pi}_{\alpha, \beta}(z)}{\partial \beta}U^\prime_{\alpha, \beta}(z)dz 
	+\int\frac{\tilde{\pi}_{\alpha,\beta}(z)}{Z_{\alpha,\beta}}U^{\prime\prime}_{\alpha,\beta}(z)dz \\
	& & \hspace{4em}- \int \frac{\tilde{\pi}_{\alpha,\beta}(z)}{Z_{\alpha,\beta}}U^\prime_{\alpha,\beta}(z)dz\int\frac{1}{Z_{\alpha,\beta}}\frac{\partial \tilde{\pi}_{\alpha,\beta}(z)}{\partial \beta} \ud z\\
	&=&\int\pi_{\alpha, \beta}(z)\left(U^\prime_{\alpha, \beta}\right)^2 \ud z
	-\int \pi_{\alpha, \beta}(z) U^\prime_{\alpha, \beta}(z) \ud z
	\int \frac{\tilde{\pi}_{\alpha,\beta}(z)}{Z_{\alpha,\beta}}U^\prime_{\alpha, \beta}(z)\ud z \\
	&&  \hspace{4em} +\int\frac{\tilde{\pi}_{\alpha,\beta}(z)}{Z_{\alpha,\beta}}U^{\prime\prime}_{\alpha,\beta}(z) \ud z\\
	&=&\mathbb{E}_{\pi_{\alpha, \beta}(z)}\left[\left(U^\prime_{\alpha,\beta}(z)\right)^2	\right]-\left(\mathbb{E}_{\pi_{\alpha, \beta}(z)}\left[U^\prime_{\alpha, \beta}(z)\right]\right)^2+ \mathbb{E}_{\pi_{\alpha, \beta}(z)}\left[U^{\prime\prime}_{\alpha,\beta}(z)\right]\\
	&=&\mathbb{E}_{\pi_{\alpha, \beta}(z)}\left[\left(U^\prime_{\alpha,\beta}(z)\right)^2	\right]
	-\left(\mathbb{E}_{\pi_{\alpha, \beta}(z)}\left[U^\prime_{\alpha, \beta}(z)\right]\right)^2 
	-\alpha \mathbb{E}_{\pi_{\alpha, \beta}(z)}\left[(U^\prime_{\alpha, \beta}(z))^2 \right]\\
	&=& \mbox{Var}_{\pi_{\alpha,\beta}(z)}\left[U^\prime_{\alpha,\beta}(z)\right]
	-\alpha \mathbb{E}_{\pi_{\alpha, \beta}(z)}\left[(U^\prime_{\alpha, \beta}(z))^2 \right]\\
	&=&-\left(\mathbb{E}_{\pi_{\alpha, \beta}(z)}\left[U^\prime_{\alpha, \beta}(z)\right]\right)^2 +(1-\alpha) \mathbb{E}_{\pi_{\alpha, \beta}(z)}\left[(U^\prime_{\alpha, \beta}(z))^2 \right]
	%		\end{split}
	%	\end{equation}
	\eeqs
	Therefore, when $\alpha\leq 0$ then $\partial_\beta g_\alpha(\beta) \geq 0$, so $g_\alpha(\beta)$ is non-decreasing; 
	and when $\alpha\geq 1$ then $\partial_\beta g_\alpha(\beta)\leq 0$, and $g_\alpha(\beta)$ is non-increasing.
\end{proof}

\subsection{Derivation for the Importance Weighted $\HBO$}

\begin{proof}
	
	For simplicity, let us denote
	
	%\section{IW-HBO}
	\beq
	s=\left(\frac{p(\xv, \zv)}{q(\zv)}\right)^\alpha, \text{ and } t=s-1
	\eeq
	
	The the (unnormalized) importance weight $\tilde{\pi}/q$ can be expressed ass
	
	\beq
	\frac{\tilde{\pi}_{\alpha, \beta}(\zv)}{{q(z)}}=\left[\beta \left(\frac{p}{q}\right)^\alpha+(1-\beta)\right]^{1/\alpha} = (\beta \cdot t +1)^{1/\alpha}
	\eeq
	
	Given a set of empirical samples $\{\zv_i\}_{i=1}^m$ from $q_{\phi}(\zv|\xv)$, the self-normalized importance weights are given by 
	\beq
	\bar{w}_{\alpha, i}^\beta = w_{\alpha, i}^\beta/\sum_{i}w_{\alpha, i}^\beta, 
	\eeq
	where $w_{\alpha, i}^\beta = (\beta \cdot t_i +1)^{1/\alpha}$. And on the other hand, we have the integrand for the local evidence term given by 
	\beq
	U_{\alpha, \beta}^\prime =\frac{p^{\alpha} -q^{\alpha}}{\alpha \pi_{\alpha, \beta}^\alpha} = \frac{t}{\alpha (\beta t+1)}
	\eeq
	
	Plugging into the empirical self-normalized importance weighted estimator, 
	\beq
	\mathbb{E}_{\pi_{\alpha,\beta}}[f(\Zv)] \approx \sum_i \left[\bar{w}_{\alpha, i}^\beta \cdot f(\mathbf{z}_i)\right]
	\eeq
	this gives us
	\beqs
	\CE_{\alpha, \beta} & =  & \mathbb{E}_{\pi_{\alpha, \beta}}[U_{\alpha, \beta}^\prime] \\
	& \approx & \sum_{i}\frac{t_i}{\alpha(\beta t_i+1)} \bar{w}_{\alpha,i}^\beta \\
	& = & \sum_{i}\frac{t_i}{\alpha(\beta t_i+1)}\frac{(\beta \cdot t_i +1)^{1/\alpha}}{\sum_{i'} (\beta t_{i'} +1)^{1/\alpha}} \\
	& = & \frac{1}{\alpha \sum_{i'} (\beta t_{i'} +1)^{1/\alpha}} \sum_i \frac{t_i}{(\beta t_i +1)^{1-1/\alpha}}
	\eeqs
	
	%Sample $z$ from $q$
	
	%$$\rightarrow \mathbb{E}_{\pi_{\alpha,\beta}}[f(z)] = \sum_{i}\left[\bar{w}_{\alpha, i}^\beta \cdot f(z)\right]$$
	
	%$$ U_{\alpha, \beta}^\prime =\frac{p^{\alpha} -q^{\alpha}}{\alpha \pi_{\alpha, \beta}^\alpha} = \frac{t}{\alpha (\beta t+1)}$$

\end{proof}

\subsection{Derivation for the Perturbed $\HBO$ Estimator}

%\section{Taylor expansion}
%\subsection{Useful facts}

\begin{proof}
	
	To construct a perturbed estimator for $\HBO = \int \EE_{\pi_{\alpha,\beta}}[U_{\alpha,\beta}^\prime] \ud \beta$ wrt parameter $\alpha$, we will need Taylor expansions for both the integrand $U_{\alpha,\beta}^\prime$ and path density $\pi_{\alpha,\beta}$. The following facts are handy for our derivations
	\beqs
	&\tilde{\pi}_{\alpha,\beta}^\alpha = \beta\tilde{\pi}_1^\alpha +(1-\beta)\tilde{\pi}_0^\alpha\\
	&\lim_{\alpha\rightarrow 0}\tilde{\pi}_{\alpha,\beta}^\alpha = 1\\
	&\partial_\alpha \tilde{\pi}_{\alpha, \beta}^\alpha  
	= \beta\tilde{\pi}_1^\alpha \log \tilde{\pi}_1 +(1-\beta)\tilde{\pi}_0^\alpha\log\tilde{\pi}_0\\
	&\lim_{\alpha\rightarrow 0}\partial_\alpha \tilde{\pi}_{\alpha, \beta}^\alpha 
	= \log \left[\tilde{\pi}_1^\beta\tilde{\pi}_0^{1-\beta}\right]. 
	\eeqs
	
	%\begin{equation}
	%	\begin{split}
	%	&\tilde{\pi}_{\alpha,\beta}^\alpha = \beta\tilde{\pi}_1^\alpha +(1-\beta)\tilde{\pi}_0^\alpha\\
	%	&\lim_{\alpha\rightarrow 0}\tilde{\pi}_{\alpha,\beta}^\alpha = 1\\
	%	&\partial_\alpha \tilde{\pi}_{\alpha, \beta}^\alpha  
	%	= \beta\tilde{\pi}_1^\alpha \log \tilde{\pi}_1 +(1-\beta)\tilde{\pi}_0^\alpha\log\tilde{\pi}_0\\
	%	&\lim_{\alpha\rightarrow 0}\partial_\alpha \tilde{\pi}_{\alpha, \beta}^\alpha 
	%	= \log \left[\tilde{\pi}_1^\beta\tilde{\pi}_0^{1-\beta}\right]\\
	%	%&\partial^2_{\alpha\alpha}\tilde{\pi}_{\alpha,\beta}^\alpha = 
	%	\end{split}
	%\end{equation}
	
	%\paragraph{Taylor Expansion of Integrand $U^\prime_{\alpha,\beta}$ at $\alpha= 0$}
	
	We first compute the Taylor expansion of integrand $U^\prime_{\alpha,\beta}$ at $\alpha= 0$
	%\begin{equation}
	%\begin{split}
	\beqs
	U^{\prime}_{\alpha,\beta}|_{\alpha=0}
	&=& \lim_{\alpha\rightarrow 0} \frac{\tilde{\pi}_1^\alpha -\tilde{\pi}_0^\alpha}{\alpha \tilde{\pi}_{\alpha,\beta}^\alpha}
	= \lim_{\alpha\rightarrow 0} \frac {\partial_ \alpha\left[\tilde{\pi}_1^\alpha -\tilde{\pi}_0^\alpha\right]}{\partial_\alpha\left[\alpha\beta\tilde{\pi}_1^\alpha +\alpha (1-\beta)\tilde{\pi}_0^\alpha\right]}\\
	&=&\lim_{\alpha\rightarrow 0}\frac{\tilde{\pi}_1^\alpha \log \tilde{\pi}_1 - \tilde{\pi}_0^\alpha \log \tilde{\pi}_0}{\beta\tilde{\pi}_1^\alpha(1+\alpha \log\tilde{\pi}_1) + (1-\beta)\tilde{\pi}_0^\alpha(1+\alpha\log\tilde{\pi}_0)}\\
	&=& \log\tilde{\pi}_1 -\log\tilde{\pi}_0 \rightarrow U^\prime_{0,\beta}\\
	\partial_\alpha U^{\prime}_{\alpha,\beta}|_{\alpha=0}
	&=& \lim_{\alpha\rightarrow 0} \frac{\partial}{\partial_\alpha}\left[\frac{\tilde{\pi}_1^\alpha -\tilde{\pi}_0^\alpha}{\alpha \tilde{\pi}_{\alpha,\beta}^\alpha}\right]\\
	&=& \lim_{\alpha\rightarrow 0} \bigg(\underbrace{-\frac{1}{\alpha^2}\frac{\tilde{\pi}_1^\alpha - \tilde{\pi}_0^\alpha}{\tilde{\pi}_{\alpha,\beta}^\alpha} }_{T_1}
	+ \underbrace{\frac{1}{\alpha}\frac{\tilde{\pi}_1^\alpha \log\tilde{\pi}_1 - \tilde{\pi}_0^\alpha \log \tilde{\pi}_0}{\tilde{\pi}_{\alpha,\beta}^\alpha}}_{T_2}
	-\underbrace{\frac{\tilde{\pi}_1^\alpha - \tilde{\pi}_0^\alpha}{\alpha \tilde{\pi}_{\alpha, \beta}^{2\alpha}}\partial_\alpha \tilde{\pi}_{\alpha,\beta}^\alpha}_{T_3}\bigg)\\
	\lim_{\alpha\rightarrow 0}T_1 &=&-
	\lim_{\alpha\rightarrow 0}\frac{\tilde{\pi}_1^\alpha\log\tilde{\pi}_1 - \tilde{\pi}_0^\alpha \log\tilde{\pi}_0}{2\alpha \tilde{\pi}_{\alpha,\beta}^\alpha +\alpha^2\partial_\alpha \tilde{\pi}_{\alpha,\beta}^\alpha}\\
	&=&- \lim_{\alpha\rightarrow 0}\frac{\tilde{\pi}_1^\alpha(\log \tilde{\pi}_1)^2 - \tilde{\pi}_0^\alpha(\log \tilde{\pi}_0)^2 }{2\tilde{\pi}_{\alpha,\beta}^\alpha 
		+4\alpha \partial_\alpha \tilde{\pi}_{\alpha,\beta}^\alpha + \alpha^2 \partial^2_{\alpha\alpha}\tilde{\pi}_{\alpha,\beta}^\alpha} =- \frac{\left(\log{\tilde{\pi}_1}\right)^2 - \left(\log{\tilde{\pi}_0}\right)^2}{2}\\
	\lim_{\alpha\rightarrow 0}T_2 &=& \lim_{\alpha\rightarrow 0}\frac{\tilde{\pi}_1^\alpha\left(\log\tilde{\pi}_1\right)^2 -\tilde{\pi}_0^\alpha\left(\log\tilde{\pi}_0\right)^2}{\tilde{\pi}_{\alpha,\beta}^\alpha + \alpha\partial_\alpha \tilde{\pi}_{\alpha,\beta}^\alpha} 
	=\left(\log{\tilde{\pi}_1}\right)^2- \left(\log{\tilde{\pi}_0}\right)^2\\
	\lim_{\alpha\rightarrow 0} T_3 &=& \lim_{\alpha\rightarrow 0}\frac{\left(\tilde{\pi}_1^\alpha\log\tilde{\pi}_1 - \tilde{\pi}_0^\alpha\log\tilde{\pi}_0\right)\partial_\alpha\tilde{\pi}_{\alpha,\beta}^\alpha - \left(\tilde{\pi}_1^\alpha -\tilde{\pi}_0^\alpha\right)\partial^2_{\alpha\alpha}\tilde{\pi}_{\alpha, \beta}^\alpha}{\tilde{\pi}_{\alpha, \beta}^{2\alpha} + 2\tilde{\pi}_{\alpha,\beta}^\alpha\partial_\alpha \tilde{\pi}_{\alpha,\beta}^\alpha}\\
	&=& \log\left[\tilde{\pi}_1\tilde{\pi}_0^{-1}\right]\log \left[\tilde{\pi}_1^\beta\tilde{\pi}_0^{1-\beta}\right]
	\eeqs
	Putting everything together, we have
	\beqs
	\partial_\alpha U^\prime_{\alpha, \beta}|_{\alpha=0}  &=& \left(\log \tilde{\pi}_1 - \log\tilde{\pi}_0\right)^2\left(\frac{1}{2}-\beta\right) \rightarrow \partial_\alpha U^\prime_{0, \beta}\\
	U^\prime_{\delta,\beta} &\approx& \left(\log\tilde{\pi}_1 -\log\tilde{\pi}_0\right)+ \left(\log \tilde{\pi}_1 - \log\tilde{\pi}_0\right)^2\left(\frac{1}{2}-\beta\right)\delta + \mathcal{O}(\delta)^2
	\eeqs
	%\end{split}
	%\end{equation}
	
	%\subsection{Taylor expansion of $\tilde{\pi}_{\alpha,\beta}$ around $\alpha = 0$}
	
	%\paragraph{Taylor Expansion of Path Density $\tilde{\pi}_{\alpha,\beta}$ at $\alpha = 0$}
	For the path density, we consider its unnormalized form
	
	\beqs
	%\begin{equation}
	%\begin{split}
	\partial_\alpha \log \tilde{\pi}_{\alpha, \beta} 
	&=& -\underbrace{\frac{1}{\alpha^2}\log\left(\beta\tilde{\pi}_1^\alpha +(1-\beta)\tilde{\pi}_0^\alpha\right)}_{I_1} + \underbrace{\frac{\beta\tilde{\pi}_1^\alpha\log\tilde{\pi}_1+(1-\beta)\tilde{\pi}_0^\alpha\log\tilde{\pi}_0}{\alpha\left(\beta \tilde{\pi}_1^\alpha +(1-\beta)\tilde{\pi}_0^\alpha\right)}}_{I_2}\\
	\lim_{\alpha\rightarrow 0}I_1 
	&=& \lim_{\alpha\rightarrow 0}\frac{\beta\tilde{\pi}_1^\alpha \log\tilde{\pi}_1+(1-\beta)\tilde{\pi}_0^\alpha \log\tilde{\pi}_1^\alpha}{2\alpha(\beta\tilde{\pi}_1^\alpha+(1-\beta)\tilde{\pi}_0^\alpha)} \\
	&=& \lim_{\alpha\rightarrow 0}\frac{\beta\tilde{\pi}_1^\alpha (\log\tilde{\pi}_1)^2 +(1-\beta)\tilde{\pi}_0^\alpha (\log\tilde{\pi}_0)^2}{2(\beta\tilde{\pi}_1^\alpha +(1-\beta)\tilde{\pi}_0^\alpha)} \\
	& = & \frac{1}{2}\left[\beta(\log\tilde{\pi}_1)^2 +(1-\beta)(\log\tilde{\pi}_0)^2\right]\\
	\lim_{\alpha\rightarrow 0}I_2 
	&=& \lim_{\alpha\rightarrow 0}\frac{\beta\tilde{\pi}_1^\alpha(\log\tilde{\pi}_1)^2+(1-\beta)\tilde{\pi}_0^\alpha(\log \tilde{\pi}_0)^2}{\beta\tilde{\pi}_1^\alpha +(1-\beta)\tilde{\pi}_0^\alpha} \\
	&=&\beta(\log\tilde{\pi}_1)^2 +(1-\beta)(\log\tilde{\pi}_0)^2\\
	\lim_{\alpha\rightarrow 0 }\partial_\alpha \log \tilde{\pi}_{\alpha,\beta} 	& = & \frac{1}{2}\left[\beta(\log\tilde{\pi}_1)^2 +(1-\beta)(\log\tilde{\pi}_0)^2\right]\\
	\log\tilde{\pi}_{\delta, \beta} &\approx& (\beta\log\tilde{\pi}_1 +(1-\beta)\log\tilde{\pi}_0) + \frac{1}{2}\left[\beta(\log\tilde{\pi}_1)^2 +(1-\beta)(\log\tilde{\pi}_0)^2\right]\delta
	%\end{split}
	%\end{equation}
	\eeqs
	This concludes our proof.
\end{proof}

\section{Experimental Setups}

\subsection{Datasets}

The following datasets are considered in the current study.

\begin{itemize}
	
	\item {\bf MNIST} a handwritten digit database with $70$k binarized $28\times 28$ images. Following standard split, we use $60$k for development (5/1 split for training and validation) and the rest $10$k for test.

	\item {\bf Omniglot} \cite{lake2015human} is a dataset of $1623$ handwritten characters across $50$ alphabets. Each data point is a binarized $28\times 28 $image. We split the dataset into $24,345$ for training and $8,070$ for testing.
	
	\item {\bf Cifar10} \cite{krizhevsky2009learning} consists of $60$k size $32\times 32$ colour images from $10$ classes. We split the dataset into $50$k training and $10$k for testing.
	
	\item {\bf CelebA} \cite{liu2018large} consists more than $200$k celebrity images. We split the dataset into $162,770$ training and $19,962$ for testing.
	
	\item {\bf Yelp} \cite{shen2017style, yang2017improved} contains more than $100,000$ long sentences with average length equals $96.7$. We split the dataset into $100,000$ training sentences, $10,000$ for validation and $10,000$ for test.
\end{itemize}

% \paragraph{Integrating local evidence}
\subsection{Numerical integration}

To compute an approximation to the $\log$-likelihood, we need to numerically integrate the local evidence along the thermodynamic curve.  
We consider the following partition schemes for varying budget $K \in [2, 5, 10, 30, 50]$:
\begin{itemize}
	\item Log partition: fix $\beta_0 = 0$, split the interval $[\beta_1,1]$ evenly on a log scale. We follow the original $\TVO$ paper \cite{masrani2019the} choose $\beta_1 = 10^{-1.09}$ for a large number of partition.
	% and $\beta_1=1$ for a few partitions.
	\item Uniform partition: split the interval $[0,1]$ evenly on a linear scale.
\end{itemize}
% \paragraph{Integration} 
% To compute $\int_{0}^{1}f(x)dx$, we splited internal $[0,1]$ into $K$ partitions: $0=\beta_0\leq\beta_1\leq\cdots\leq \beta_K=1$. 
We have compared the following three integration strategies, 
\begin{itemize}
	\item Trapz:  $\sum_{i=0}^{K-1}\left(\beta_{i+1}-\beta_{i}\right)\cdot \frac{f(\beta_{i})+f(\beta_{i+1})}{2}$,
	
	\item Left: $\sum_{i=0}^{K-1}\left(\beta_{i+1}-\beta_{i}\right)\cdot f(\beta_{i})$,
	
	\item Right: $\sum_{i=0}^{K-1}\left(\beta_{i+1}-\beta_{i}\right)\cdot f(\beta_{i+1})$.
\end{itemize}
%We have found that Trapz consistently yields the best results, and report all our results using the Trapz scheme. 

\subsection{Model architectures (standard)}

Following \cite{kingma2013auto, burda2015importance}, we use the standard Gaussian for the prior and conditional Gaussian for the approximate posterior, unless otherwise specified, with the mean and diagonal covariance parameterized by deep neural net. The condional likelihood model is set to Bernoulli for binary responses and Gaussian for continuous responses (for simplicty we fix the variance to $1$). 
% The model is of the form $p(\zv)p_\theta(\xv|\zv) = Normal(\zv|0)$
\begin{equation}
	\begin{split}
		p_{\theta}(\xv,\zv) &= p_{\theta}(\xv|\zv)p(\zv),\\
		p(\zv) &= Normal(\zv|0,I), \\
		p_{\theta}(\xv|\zv) &= Bernoulli(\xv|decoder_{\theta}(\zv))
		~or~ p_{\theta}(\xv|\zv) =  Normal(\xv|decoder_\theta(\zv),\sigma^2)\\
		q_{\phi}(\zv|\xv) &=Normal(\zv|\mu_\phi(\xv),\sigma_\phi(\xv))
	\end{split}
\end{equation}
For MNIST and Omniglot dataset, we use two-layer MLPs with tanh activation as encoder. The output of the encoder is duplicated and passed through an additional linear layer to parameterize the mean and log-standard deviation of a 200 hidden dimensions conditionally independent Normal distribution. The decoder is a three-layer MLP with tanh activations and sigmoid 
output which parameterizes the probabilities of the Bernoulli distribution.

For Cifar10 dataset, we use three-layer Conv2d (filters: $64\times 4\times 4-128\times 4\times 4-512\times 4\times 4$) with ReLU activation as encoder \citep{chen2021finite}. The latent representation first passes through a Dense layer, then a three-layer DeConv2d decoder (filters: $256\times 4\times 4-64\times 4\times 4-3\times 4\times 4$) with ReLU activations. We set hidden dimensions to $1024$.

For CelebA dataset, we use the five-layer Conv2d $(32\times 4\times 4-64\times 4\times 4 - 128\times 4\times 4 - 256\times 4\times 4 - 512\times 4\times 4)$.
The decoder is a five-layer Upsampling ConV2d network where the filter size and kernel size are the reverse of the encoder. Each Conv2d layer in the decoder follows a upsampling layer. 
For both encoder and decoder, we choose LeakyReLU as activation and add a BN layer after each Conv2d layer.  

For Yelp dataset, we implement both encoder and decoder as one-layer LSTMs with 1024 hidden units and 512-dimensional word embeddings. The vocabulary size is 20K. The last hidden state of the encoder concatenated with a 32-dimensional Gaussian noise is used to sample 32-dimensional latent codes, which is then adopted to predict the initial hidden state of the decoder LSTM and additionally fed as input to each step at the LSTM decoder. A KL-cost annealing strategy is commonly used \citep{tao2019variational}, where the scalar weight on the KL term is increased linearly from 0.1 to 1.0 each batch over 10 epochs. There are dropout layers with probability 0.5 between the input-to-hidden layer and the hidden-to-output layer on the decoder LSTM only.

All the parameters for the image model are initialized at random and optimized using ADAM.

\subsection{Inference with normalizing flows}
\label{sec:nf}

To enable more flexible posterior approximation, we consider empowering $q_{\phi}(z|x)$ with normalizing flows \citep{rezende2015variational}. We modified the Karpathy's \texttt{Torch} implementation of flows \footnote{\url{https://github.com/karpathy/pytorch-normalizing-flows/}} ({\it e.g.}, adding conditional flow support), and replace standard encoder with a flow-based encoder. To get maximal sampling efficiency, we choose the MAF-based IAF flow as our default choice \citep{kingma2016improving, papamakarios2017masked,chen2020supercharging}. We notice such shift-scale flows sometimes suffer stability issues, which can be remedied by the slower alternatives \citep{papamakarios2019normalizing}, such as spline flows \citep{durkan2019neural}. 

\subsection{Baselines and specifications} 

We consider the following representative or popular baselines to benchmark our $\HBO$. 
\begin{itemize}
	\item $\ELBO$/VAE: \cite{kingma2013auto} Vanilla $\ELBO$  
	\item $\IWELBO$/IW-VAE: \cite{burda2015importance} Importance weighted $\ELBO$. We set importance samples to $10$. 
	\item R\'enyi: \cite{li2016renyi} We use the R\'enyi-Max variant which reported best performance in the original R\'enyi paper
	\item $\TVO$: \cite{masrani2019the} Thermodynamics variational objective
\end{itemize}
For fair comparison, other aspects are all matched. 

\section{Additional Results and Analyses for Synthetic Examples}
\label{sec:more_exp}

\subsection{Bayesian regression}

To benchmark $\HBO$'s performance against other bounds, we consider Bayesian parameter estimation for regression problems. Specifically, we consider the  classic toy model given below, given assigns a non-informative prior:
\beqs
Y = \alpha + \beta \tilde{X} + \epsilon, \quad X = \tilde{X} +\zeta \label{eq:br_mdl} \\
p(\alpha,\beta) \propto (1+\beta^2)^{-3/2}, \quad p(\sigma) \propto \frac{1}{\sigma}. \label{eq:br_prior}
\eeqs
where $\epsilon,\zeta \sim \CN(0,\sigma^2)$, $\tilde{X} = \CU[0,100]$, and observations are given in pairs of $\mathcal{D} = \{(x_i, y_i)\}_{i=1}^n$. We set ground-truth parameters to $\alpha = 25, \beta=0.5, \sigma^2 = 10$, and sample $n=20$ points. The goal is to inference the posterior distribution $p(\alpha,\beta,\sigma|\mathcal{D})$. We use the \texttt{emcee} \footnote{\url{https://emcee.readthedocs.io/en/stable/}} package to draw MCMC samples as ground-truth reference (See Figure \ref{fig:br_mcmc_fit}). 
% * $y = \alpha + \beta x + \epsilon$
%   * $\epsilon \sim \CN(0,\sigma^2)$
%   * $p(\alpha,\beta) \propto (1+\beta^2)^{-3/2}$
%   * $p(\sigma) \propto \frac{1}{\sigma}$

In Figure \ref{fig:br_loss} we examine the convergence of different VI criteria quantitatively. Specifically, we appeal to the {\it maximal mean discrepancy} (MMD) metric \citep{gretton2012kernel} to evaluate the similarity between the ground-truth posterior to the variational approximations. Since the original scale of parameters differ, we use the mean and standard deviation of the true posterior to normalize all sample estimates, and set kernel bandwidth parameter to $0.5$. We use $5k$ MCMC samples as the reference distribution and find the results sufficiently stable. As we can see, $\HBO$ delivers the fastest convergence to ground-truth, followed by $\TVO$ and then $\IWELBO$. Vanilla $\ELBO$ struggles the most. This is consistent with the theoretical predictions of the tightness of the bounds. We further carried out the ablation study to examine the convergence for different $\HBO(\alpha)$. 

\begin{figure}[!t]
	\centering
	\includegraphics[width=0.5\textheight]{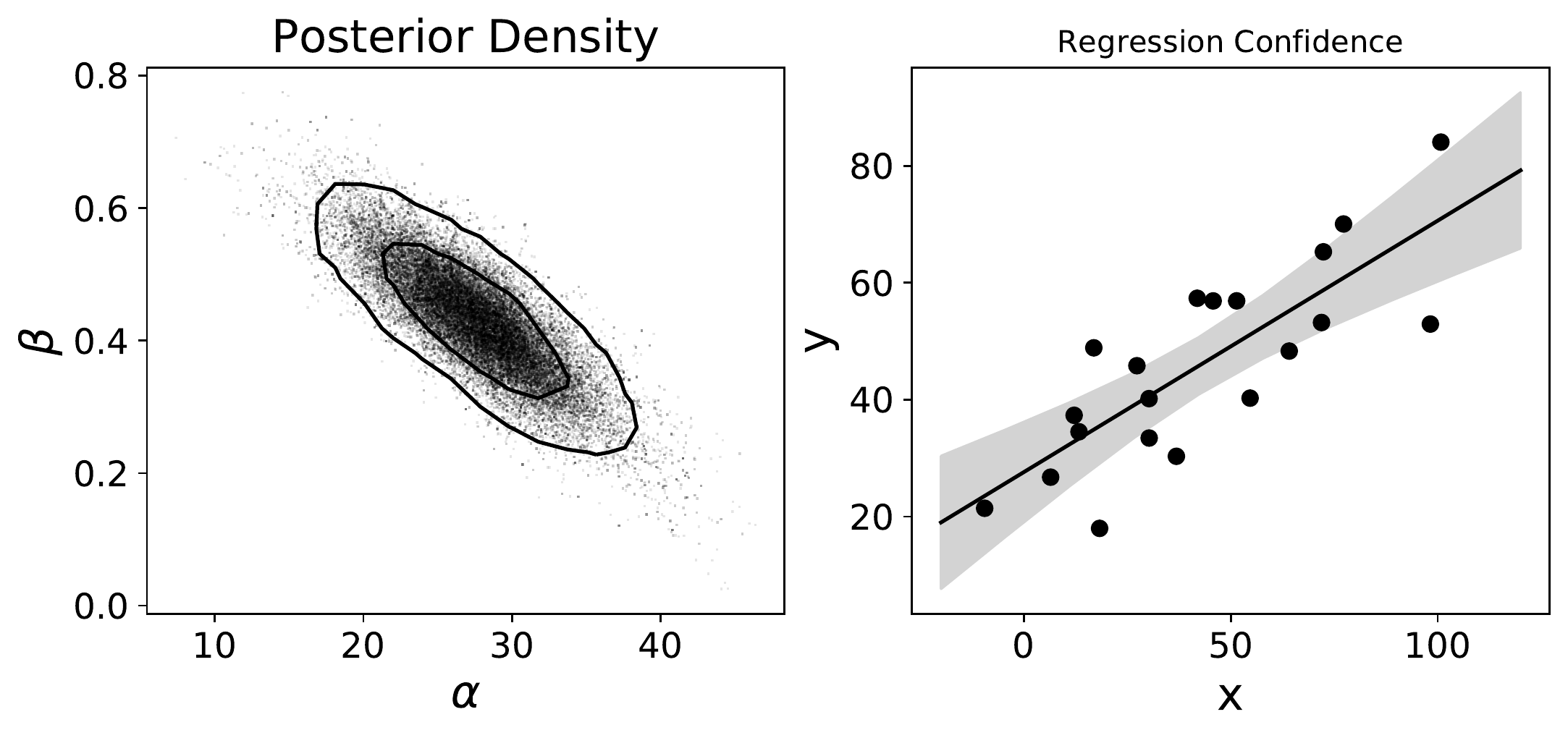}
	\caption{MCMC posterior (left) and model fit (right) for Eqn (\ref{eq:br_mdl}-\ref{eq:br_prior}). Dots are data points, solid line fitted curve, and shaded region confidence intervals.}\label{fig:br_mcmc_fit}
\end{figure}

\begin{figure}
	\centering
	\includegraphics[width=\textwidth]{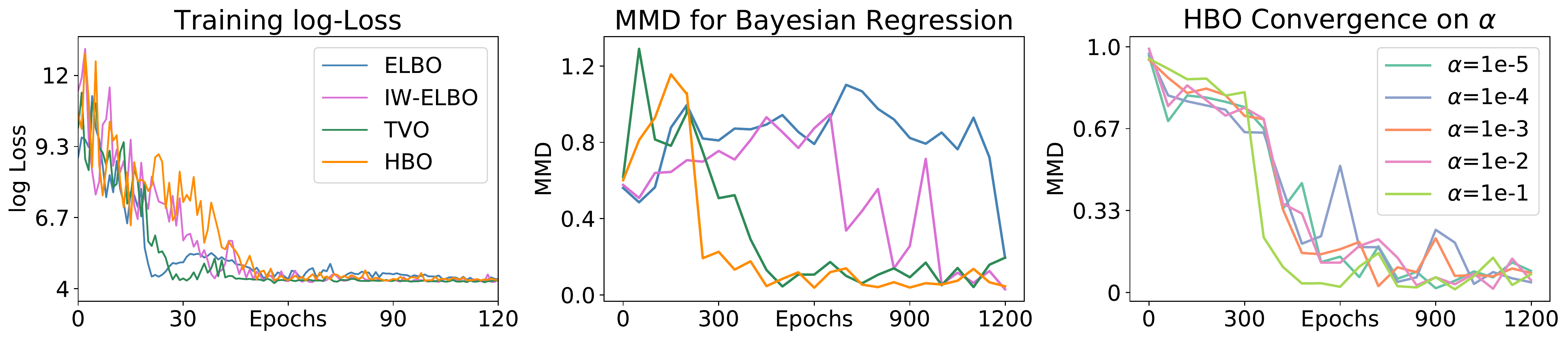}
	\caption{Convergence plot for different objectives. (left) training loss; (middle) MMD loss for different bounds; (right) MMD loss for different $\HBO(\alpha)$. Lower is better.}
\end{figure}

\subsection{Bayesian inference for complex posterior}\label{sec:bayes_inf}

To make the inference more challenging, we consider the following model:
\begin{equation}
y = \sqrt{z_1^2+z_2^2}+\xi, z_1,z_2 \sim \CN(0,1), \xi \sim \CN(0,0.1^2),
\end{equation}
and the goal is to infer pair $(z_1, z_2)$ given an observation $y$. The posterior will be circle-like for a sufficiently large $y$ (see Figure \ref{fig:vi_approx}), and here we focus on the case where $y=1$.

% $y = \sqrt{z_1^2+z_2^2}+\xi$, $z_1,z_2 \sim \CN(0,1)$
%   * $\xi \sim \CN(0,0.1^2)$

\begin{figure}
	\centering
	\includegraphics[width=.8\textwidth]{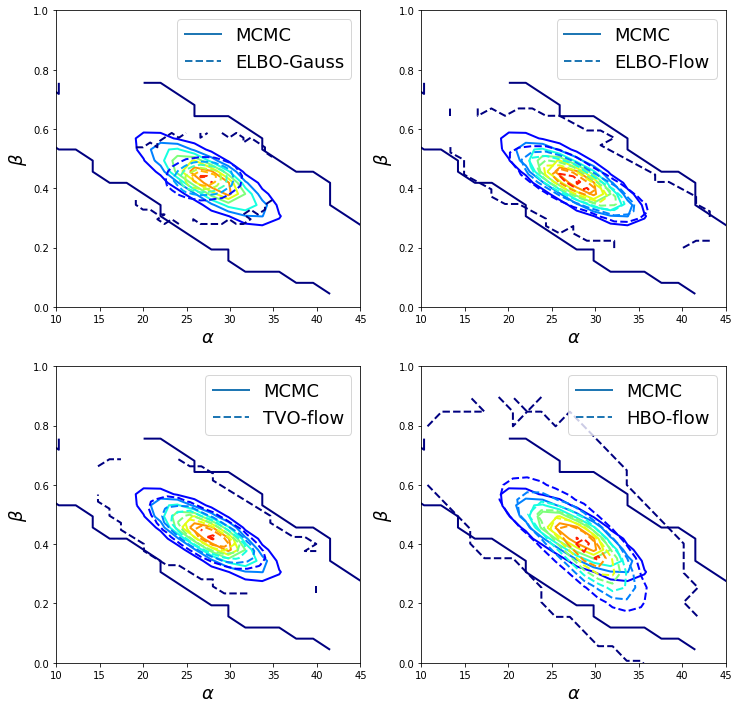}
	\caption{Comparison of posterior estimates from different variational objectives, overlaid on ground-truth contours.}
	\label{fig:mode-covering}
\end{figure}

\section{Additional Results and Analyses for Real-World Datasets}

% In Table \ref{tab:omniglot} we summarize the likelihood bound evaluations for the Omniglot data. 

In Figure \ref{fig:ess_cmp} we compare the effective sample-size (ESS) \citep{chen2021simpler} of different schemes on MNIST. To come up with a single summarizing statistics, ESS is averaged along the discretized thermodynamics path. Consistent with our analysis from the toy model, $\HBO$ shows better latent sample efficiency compared with other alternatives with a sharper bound. 

In Figure \ref{fig:utility_elbo} we visualize the spread of ELBO evaluation comparing the utility of learned posterior respectively learned from $\TVO$ and $\HBO$. It shows posterior samples from $\HBO$ are more likely relative to those from $\TVO$. We also examined R\'enyi, but it shows and more generated profile, and therefore we removed it from our presentation for visual clarity. 

Figure \ref{fig:alpha_tuning} plots a few thermodynamic local evidence curve corresponding to different $\alpha$. We see the phase transition ({\it i.e.}, flipped monotonicity) happening. Note that the best curve does not seem strictly flat, which can be potentially attributed to the approximation error originated from our perturbed estimator and the importance weighting scheme.

\begin{figure}[!htb]
	\centering
	%         \begin{minipage}{0.48\textwidth}
	%       \vspace{-3em}
	% \renewcommand{\figurename}{Table}
	% \setcounter{figure}{0}
	% \caption{Omniglot Results \label{tab:omniglot}}
	% \renewcommand{\figurename}{Figure}
	% \setcounter{figure}{2}
	% \vspace{.5em}
	% \end{minipage}
	\hspace{1em}
	\begin{minipage}{.48\textwidth}
		\centering
		\includegraphics[width=0.3\textheight]{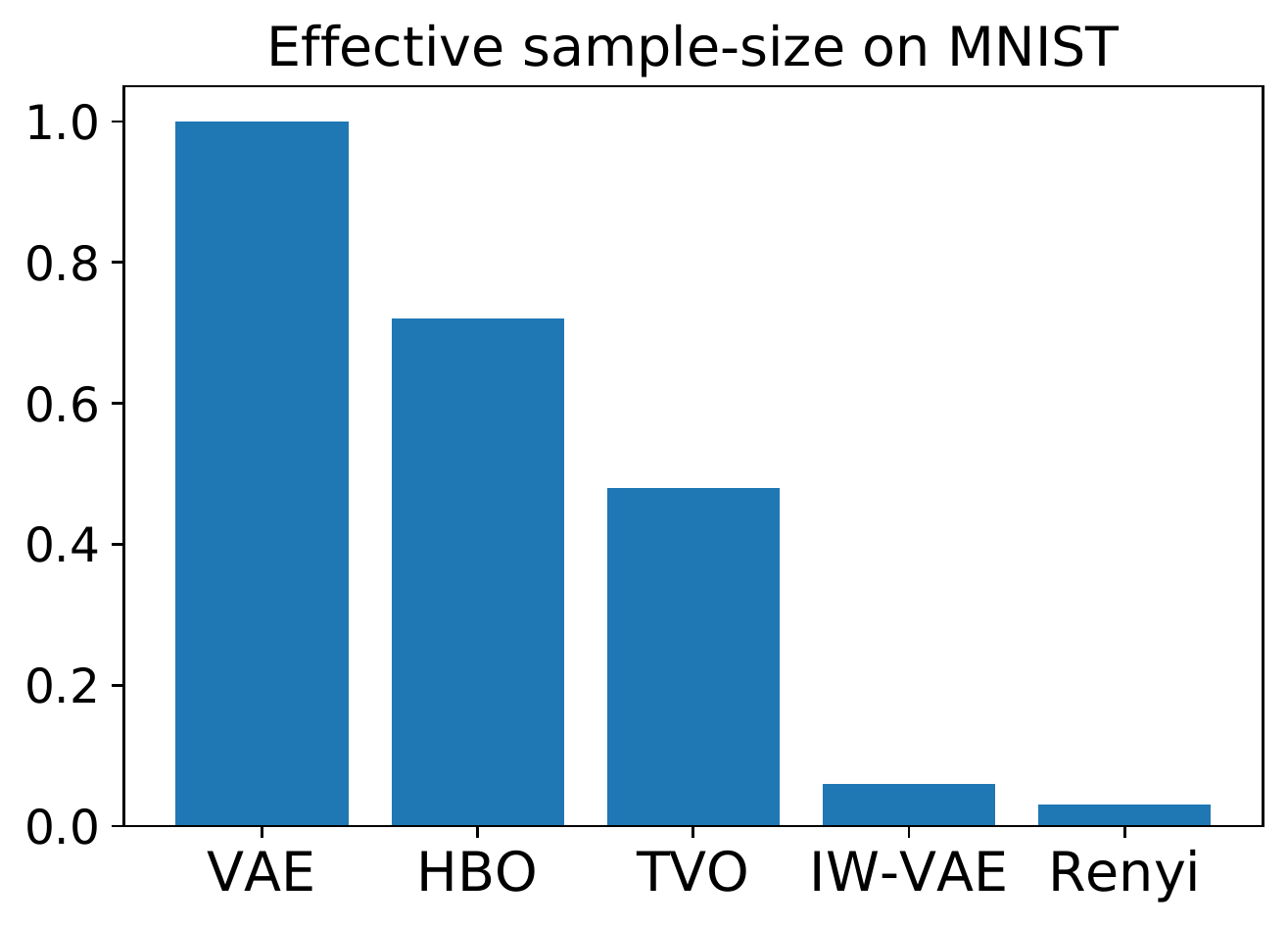}
		\caption{Comparison of effective sample-size with different VI objectives on MNIST. \label{fig:ess_cmp}}
	\end{minipage}%
	\vspace{-1.5em}
\end{figure}

\begin{figure}[!htb]
	\centering
	\begin{minipage}{0.48\textwidth}
		\centering
		\includegraphics[width=.3\textheight]{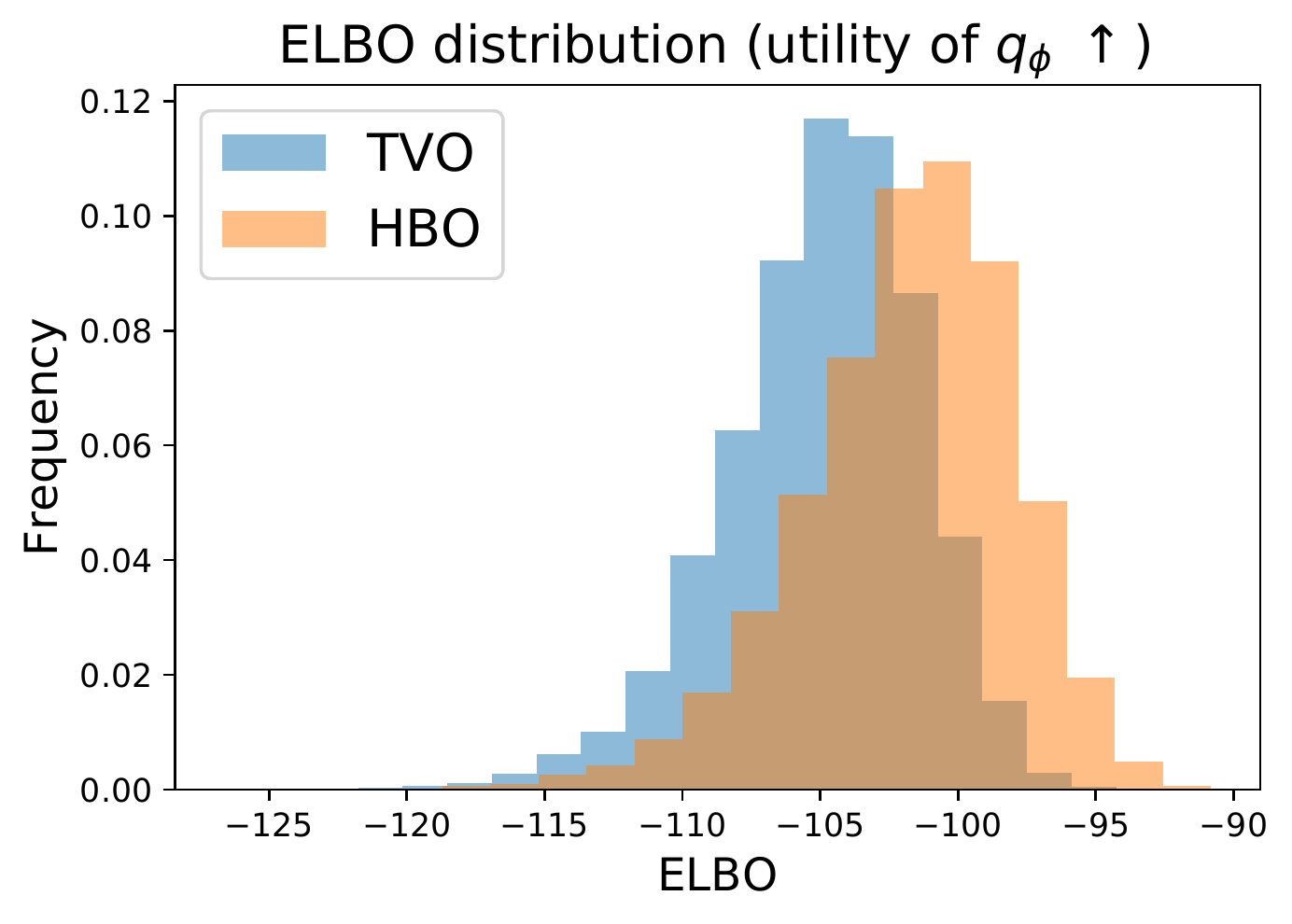}
		\caption{Distribution of local evidence ($\ELBO$) using models respectively trained with $\TVO$ and $\HBO$.}
		\label{fig:utility_elbo}
	\end{minipage}
	\begin{minipage}{.48\textwidth}
		\centering
		\includegraphics[width=0.3\textheight]{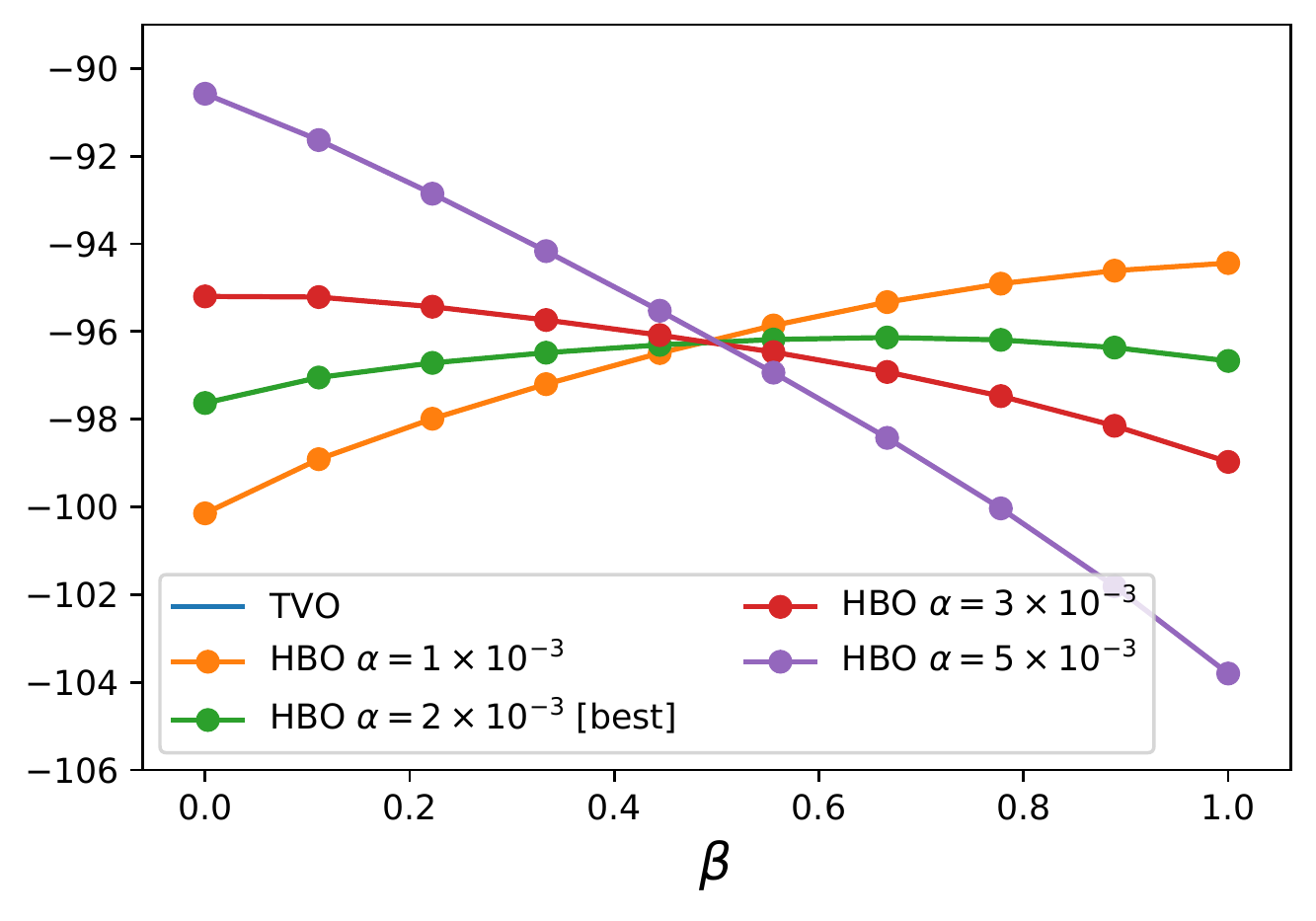}
		\caption{$\HBO$ curves for different $\alpha$ on MNIST.}
		\label{fig:alpha_tuning}
	\end{minipage}%
\end{figure}

\begin{figure}[!htb]
	\centering
	
	\begin{minipage}{0.48\textwidth}
		\centering
		\includegraphics[width=0.3\textheight]{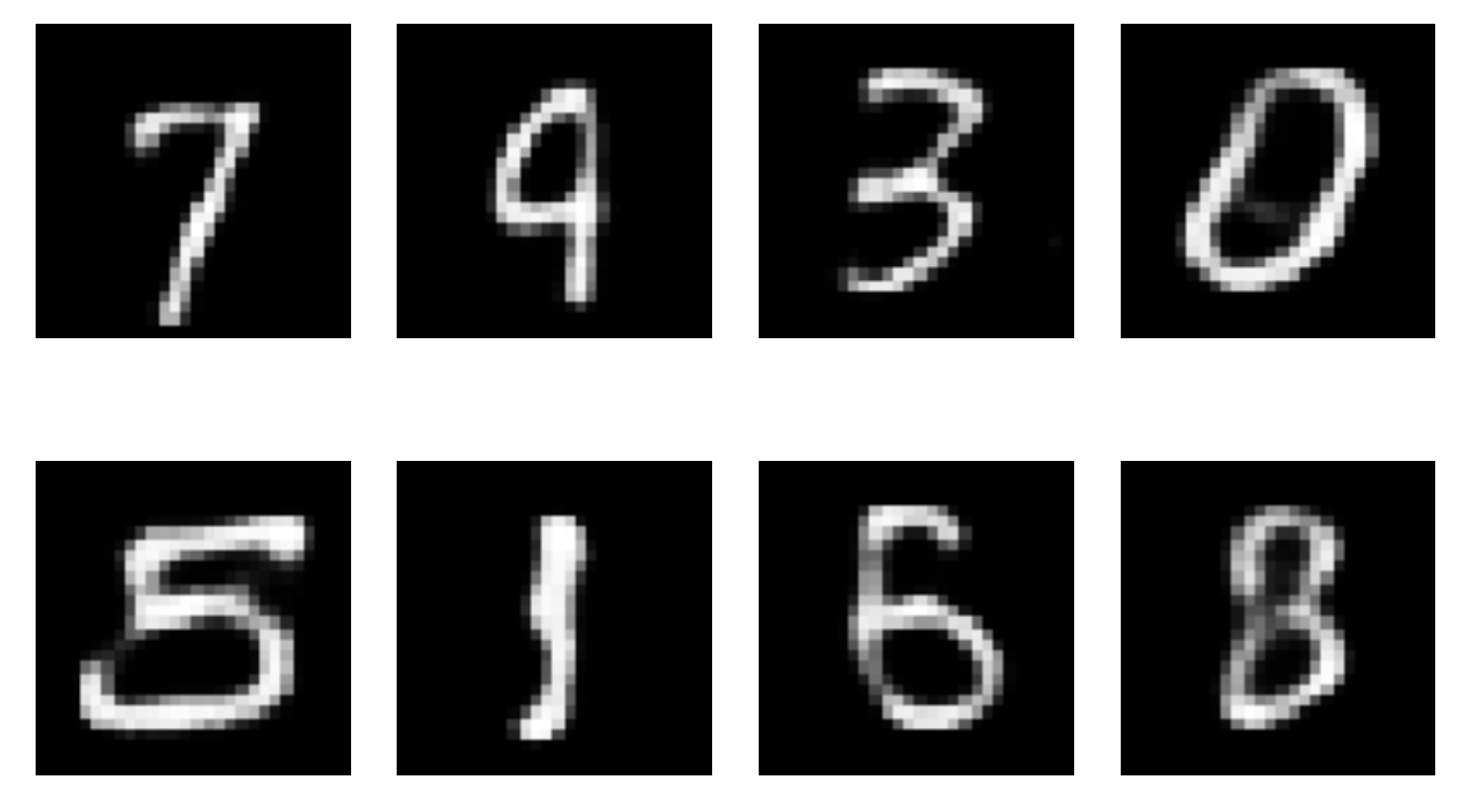}
		\caption{Samples generated from the $\HBO$ model on MNIST.}
		\label{fig:hbo_mnist}
	\end{minipage}
	\hspace{1em}
	\begin{minipage}{.48\textwidth}
		\centering
		\includegraphics[width=0.3\textheight]{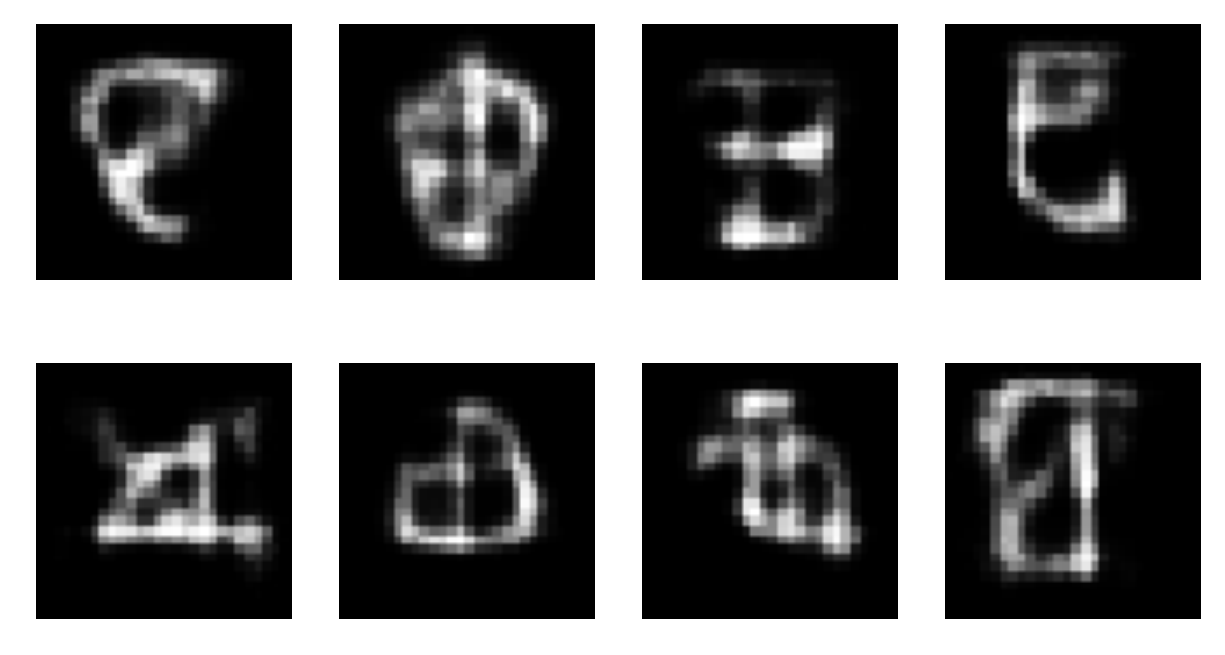}
		\caption{Samples generated from the $\HBO$ model on Omniglot.}
		\label{fig:hbo_omniglot}
	\end{minipage}
\end{figure}

\begin{figure}
	\includegraphics[width=.7\textheight]{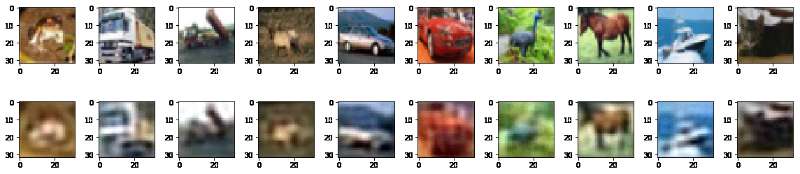}
	\caption{Samples reconstructed from Cifar.}
	\label{fig:hbo_cifar}
\end{figure}

\begin{figure}
	\begin{minipage}{.48\textwidth}
		\includegraphics[width=.3\textheight]{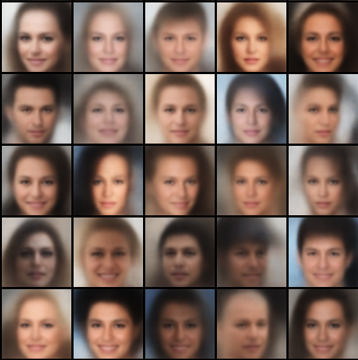}
		\caption{Samples generated from CelebA.}
		\label{fig:hbo_celeba}
	\end{minipage}
	\hspace{3em}
	\begin{minipage}{.48\textwidth}
		Original:
		"<s> slow service , rude hostess and cold food . what 's the point when you have so many options ? </s>"\\
		
		VAE:
		"food service . food rude , rude beer . the a not point of you 're to many bad for </s>"\\
		
		HBO:
		"the service . slow staff , the food . </s> a the point of you are a much people for </s>"
		
		\begin{center}
			Samples reconstructed from Yelp.
		\end{center}
	\end{minipage}
\end{figure}

\end{document}